\documentclass[twoside]{article}
\usepackage[round]{natbib}
\usepackage{amssymb}
\usepackage{hyperref}
\usepackage{url}
\usepackage{amsthm}
\usepackage{algorithm}
\usepackage{algorithmic}
\usepackage{comment}
\usepackage{xr}
\usepackage{adjustbox}
\usepackage{caption}
\usepackage{multirow}
\usepackage{tabularx}
\usepackage{subcaption}
\usepackage{paralist}
\usepackage[table]{xcolor}
\usepackage{booktabs}
\usepackage{tikz}
\usetikzlibrary{tikzmark}

\newcommand{\z}{\textbf{z}}
\newcommand{\x}{\textbf{x}}

\usepackage{enumitem}
\setlist[enumerate]{itemsep=0pt, topsep=0pt} 
\setlist[itemize]{itemsep=0pt, topsep=0pt}
\newtheorem{theorem}{Theorem}
\newtheorem{lemma}{Lemma}
\newtheorem{corollary}{Corollary}
\newtheorem{proposition}{Proposition}
\newtheorem{definition}{Definition}
\newcommand{\ar}[1]{\overrightarrow{#1}}
\newcommand{\ttt}[1]{\texttt{#1}}

\makeatletter

\newcommand*{\addFileDependency}[1]{
\typeout{(#1)}
\@addtofilelist{#1}
\IfFileExists{#1}{}{\typeout{No file #1.}}
}\makeatother

\newcommand*{\myexternaldocument}[1]{%
\externaldocument{#1}%
\addFileDependency{#1.tex}%
\addFileDependency{#1.aux}%
}

\myexternaldocument{supplement}

\newcommand\myeq{\mathrel{\stackrel{\makebox[0pt]{\mbox{\normalfont\tiny def}}}{=}}}

\usepackage[accepted]{aistats2025}
\usepackage{tikz}
\usetikzlibrary{automata, positioning}
\usepackage{comment}
\usepackage{amssymb}
\usepackage{xcolor}
\usepackage{hyperref}
\usepackage{url}
\usepackage{xr}
\usepackage{amsthm}
\usepackage{subfiles}
 \usepackage{natbib}
\usepackage{algorithm}
\usepackage{algorithmic}

\newtheorem{claim}{Claim}
\newtheorem{property}{Property}

\myexternaldocument{sample_paper}

\newtheorem*{unumberedtheorem}{Theorem}
\newtheorem*{unumberedproposition}{Proposition}
\newtheorem*{unumberedlemma}{Lemma}
\newtheorem*{unumberedclaim}{Claim}

\usepackage{xr} 
\begin{document}

%

%

\twocolumn[

\aistatstitle{On the Computational Tractability of the (Many) Shapley Values}


\aistatsauthor{Reda Marzouk*\textsuperscript{1} \And Shahaf Bassan*\textsuperscript{2} \And Guy Katz$^\dagger$\textsuperscript{2} \And Colin de la Higuera$^\dagger$\textsuperscript{1}}\begingroup
\renewcommand\thefootnote{}
\endgroup

\aistatsaddress{\textsuperscript{1} LS2N, Université de Nantes, France \And \textsuperscript{2} The Hebrew University of Jerusalem, Israel} ]

\begin{abstract}
Recent studies have examined the computational complexity of computing Shapley additive explanations (also known as SHAP) across various models and distributions, revealing their tractability or intractability in different settings. However, these studies primarily focused on a specific variant called Conditional SHAP, though many other variants exist and address different limitations. In this work, we analyze the complexity of computing a much broader range of such variants, including Conditional, Interventional, and Baseline SHAP, while exploring both local and global computations. We show that both local and global Interventional and Baseline SHAP can be computed in polynomial time for various ML models under Hidden Markov Model distributions, extending popular algorithms such as TreeSHAP beyond empirical distributions. On the downside, we prove intractability results for these variants over a wide range of neural networks and tree ensembles. We believe that our results emphasize the intricate diversity of computing Shapley values, demonstrating how their complexity is substantially shaped by both the specific SHAP variant, the model type, and the distribution.
\end{abstract}


\begin{table*}[ht]
\caption{Summary of complexity results for Baseline, Interventional, and Conditional SHAP (Base, Interv, Cond) applied to decision trees (\texttt{DT}), tree ensembles for regression and classification ($\texttt{ENS-DT}_{\texttt{R}}$, $\texttt{ENS-DT}_{\texttt{C}}$), linear regression ($\texttt{LIN}_{\texttt{R}}$), weighted automata (\texttt{WA}), and neural networks (\texttt{NN-SIGMOID}, \texttt{RNN-ReLU}). Results are analyzed under independent (\texttt{IND}), empirical (\texttt{EMP}), or hidden Markov model (\texttt{HMM}) distributions, with novel findings of this work highlighted in blue. L and G denote local and global SHAP, respectively (L* or G* if only one is covered). While most settings are either tractable (PTIME) or intractable (NP-H, coNP-H, NTIME-H, \#P-H) for all SHAP variants, we indicate cases where a \emph{strict complexity gap} between SHAP variants exists ($\downarrow$ notation).}
	\setlength{\tabcolsep}{0.8em} 
\centering
\begin{tabular}
{|c|c|c|c|c|}
\hline
\centering
& 
& \texttt{IND} & \texttt{EMP} & \texttt{HMM} \\ \hline

 & \textbf{Base}  & \cellcolor{blue!10}  PTIME (L,G) & 
 \hspace{3.5mm}PTIME (L,G)\hspace{3.5mm}\tikzmark{foo3}
 &  \cellcolor{blue!10} \hspace{3mm}PTIME (L,G)\hspace{3mm}\tikzmark{foo}
 \\ 
 {\texttt{DT}, {$\texttt{ENS-DT}_{\texttt{R}}$}, {$\texttt{LIN}_{\texttt{R}}$}} & \textbf{Interv} & \cellcolor{blue!10} PTIME (L,G) & 
 \hspace{3.5mm}PTIME (L,G)\hspace{3.5mm}  
 & \cellcolor{blue!10} \hspace{3mm}PTIME (L,G)\hspace{3mm} 
 \\
  & \textbf{Cond} & \cellcolor{blue!10} PTIME (L,G*) 
  &  
  \hspace{6.85mm}NP-H (L)\hspace{6.85mm}\tikzmark{bar3}
  & \cellcolor{blue!10} \hspace{6.25mm}\#P-H (L)\hspace{6.25mm}\tikzmark{bar} \\ \hline

 & \textbf{Base}  & \cellcolor{blue!10} PTIME (L,G) 
 & \cellcolor{blue!10} \hspace{3mm}PTIME (L,G)\hspace{3mm}\tikzmark{foo4} 
 &   \cellcolor{blue!10} \hspace{3mm}PTIME (L,G)\hspace{3mm}\tikzmark{foo2} 
 \\ 
{\texttt{WA}} & \textbf{Interv} & 
\cellcolor{blue!10} PTIME (L,G) 
& \cellcolor{blue!10} PTIME (L,G)  
& \cellcolor{blue!10} \hspace{3mm}PTIME (L,G)\hspace{3mm} 
\\ 
  & \textbf{Cond} & \cellcolor{blue!10} PTIME (L,G*) 
  & \cellcolor{blue!10} \hspace{6.35mm}NP-H (L)\hspace{6.35mm}\tikzmark{bar4} 
  & \cellcolor{blue!10} \hspace{6.25mm}\#P-H (L)\hspace{6.25mm}\tikzmark{bar2} 
  \\ 
 \hline

 & \textbf{Base}  & \textbf{---} & \cellcolor{blue!10} NP-H (L,G)  
 & \cellcolor{blue!10} NP-H (L,G) 
 \\ 
 {$\texttt{ENS-DT}_{\texttt{C}}$} & \textbf{Interv} & \#P-H   & \cellcolor{blue!10} NP-H (L) 
 & \cellcolor{blue!10} \#P-H (L) 
 \\
  & \textbf{Cond} &  \#P-H    & \cellcolor{blue!10} NP-H (L)   & \cellcolor{blue!10} \#P-H (L)  \\ \hline

 & \textbf{Base}  & \textbf{---} & \cellcolor{blue!10} NTIME-H (L,G) 
 & \cellcolor{blue!10} NTIME-H (L,G) 
 \\ 
 {\texttt{NN-SIGMOID}} & \textbf{Interv} & NP- H (L)   & \cellcolor{blue!10} NTIME-H (L) 
 & \cellcolor{blue!10} NP-H (L) 
 \\
  & \textbf{Cond} & NP-H (L)   & \textbf{---} 
  & \cellcolor{blue!10} NP-H (L) 
  \\ \hline

 & \textbf{Base}  & \textbf{---} & \cellcolor{blue!10} coNP-H (L,G) 
 & \cellcolor{blue!10} coNP-H (L,G) 
 \\ 
{\texttt{RNN-ReLU}} & \textbf{Interv} & \cellcolor{blue!10} NP-H (L) 
& \cellcolor{blue!10} coNP-H (L) 
& \cellcolor{blue!10} NP-H (L) 
\\
  & \textbf{Cond} & \cellcolor{blue!10} NP-H (L) 
  & \textbf{---} 
  & \cellcolor{blue!10} NP-H (L) 
  \\ \hline
\end{tabular}
 \label{fig:summaryresults}
  \begin{tikzpicture}[overlay,remember picture]
\draw[->,black,thick] (pic cs:foo) -- (pic cs:bar);
\end{tikzpicture}

  \begin{tikzpicture}[overlay,remember picture]
\draw[->,black,thick] (pic cs:foo2) -- (pic cs:bar2);
\end{tikzpicture}

  \begin{tikzpicture}[overlay,remember picture]
\draw[->,black,thick] (pic cs:foo3) -- (pic cs:bar3);
\end{tikzpicture}

  \begin{tikzpicture}[overlay,remember picture]
\draw[->,black,thick] (pic cs:foo4) -- (pic cs:bar4);
\end{tikzpicture}
\vspace{-10.3mm}  
\end{table*}

A prominent method for providing post-hoc explanations for ML models is via Shapley additive explanations (SHAP)~\citep{lundberg2017}. However, a major limitation of SHAP is the significantly high computational complexity of computing these explanations~\citep{bertossi2020causality}. Practical methods --- like those in the popular SHAP library~\citep{lundberg2017} --- typically address this computational burden in one of two ways. The first is through approximation techniques, such as KernelSHAP~\citep{lundberg2017}, which offer greater scalability but lack formal guarantees for the resulting explanations. The second approach involves designing algorithms tailored to specific, simpler model types (e.g., tree-based models or linear models), which are more computationally feasible. However, these methods also typically rely on underlying assumptions. For example, the popular TreeSHAP algorithm~\citep{lundbergnature} assumes explanations are based on empirical distributions, while LinearSHAP~\citep{lundberg2017} assumes feature independence.

These model-specific algorithms have sparked interest in developing a deeper theoretical understanding of the computational complexity involved in calculating Shapley values for various types of models and under different distributions. One of the early works by~\cite{vander21} presented tractable results for a range of models while also establishing NP-hardness for relatively simple settings, such as computing SHAP for decision trees with naive Bayes modeled distributions. Additionally,~\cite{arenas23} demonstrated that computing SHAP is tractable for Decomposable Deterministic Boolean Circuits under independent distributions, while \cite{marzouk24a} reported similar positive complexity results for Weighted Automata under Markovian distributions.

{\renewcommand{\thefootnote}{}%
\footnotetext{*Equal contribution (first authors), $^{\dagger}$Equal contribution (last authors).}}


However, a key limitation of these previous computational complexity works is that they have primarily focused on a specific SHAP variant known as Conditional SHAP~\citep{sundararajan20b}. The explainable AI community has explored a variety of SHAP variants, including Conditional, Interventional~\citep{janzing20a}, and Baseline~\citep{sundararajan20b} SHAP, among many others~\citep{frye20, albini22}. These variants were introduced mainly to address the axiomatic limitations of the original Conditional SHAP formulation \citep{sundararajan20b, Huang2023TheIO}. Analyzing these variants is vital as many popular SHAP algorithms incorporate them. For example, KernelSHAP~\citep{lundberg2017} and TreeSHAP~\citep{lundbergnature} typically compute \emph{interventional} SHAP values, and not conditional ones.

 The aim of this paper is to provide a comprehensive, multi-dimensional perspective on the problem of SHAP computation, analyzed through the lens of formal computational theory~\citep{arora2009computational}. This analysis spans four key dimensions: \begin{inparaenum}[(i)] \item the variants of the Shapley value, including Baseline, Interventional, and Conditional SHAP; \item the class of models to be interpreted; \item the distributional assumptions regarding the input data generation process; and \item the scope of the explanatory analysis, which can be either \emph{local} or \emph{global}, with the global scope measured as an aggregate of \emph{local} SHAP values relative to the input data-generating distribution~\citep{frye20}.\end{inparaenum}


\paragraph{Our contributions.} The complexity results of this work, summarized in Table \ref{fig:summaryresults}, are:


\begin{itemize}
    \item \emph{On the positive side,} in Section \ref{sec:tractable} we prove that both local and global Interventional and Baseline SHAP can be computed in polynomial time for the family of weighted automata, decision trees, tree ensembles in regression tasks, and linear regression models when input instances are assumed to be generated from a distribution modeled by a Hidden Markov model.
    \item \emph{On the negative side,} in Section \ref{sec:intractable} we establish intractability results, including NP-hardness, coNP-hardness, etc., for computing not only Conditional SHAP but also Interventional and Baseline SHAP across a range of neural networks (e.g., ReLU, Sigmoid, RNN-ReLU) and tree ensemble classifiers (e.g., Random Forests, XGBoost). These results hold even under strict conditions, such as uniform distributions or feature independence.
    \item Finally, in Section~\ref{sec:generalized}, we present generalized computational complexity relationships between SHAP variants, demonstrating that some are more or less tractable than others in certain distributional scenarios. Using these findings, we establish new complexity results for computing various SHAP variants across different models.
\end{itemize}

\subsection*{Key Takeaways}

The obtained complexity results provide both theoretical and practical  insights regarding the problem of SHAP computation:

\begin{enumerate}
\item \textbf{There are substantial complexity differences between SHAP variants.} Particularly, we show \emph{stark complexity gaps} between obtaining Conditional SHAP, which is NP-Hard, and both Interventional SHAP and Baseline SHAP, which can be solved in polynomial time --- demonstrating substantial deviations among the SHAP variants. Our findings additionaly pinpoint the specific settings where these gaps occur and where they do not.
\item \textbf{The distributional assumptions made by TreeSHAP and LinearSHAP can be extended to cover substantially more expressive classes of distributions.} Specifically, we demonstrate that local and global Interventional and Baseline SHAP can be solved efficiently in polynomial time for certain model families, including XGBoost trees, while surpassing the distributional scope of the widely used TreeSHAP algorithm, limited to empirical distributions. In addition, our results also relax the feature independence requirement in LinearSHAP. These results could significantly improve SHAP value distribution modeling, which is essential for computing faithful explanations~\citep{aas2021explaining}.
\item \textbf{Obtaining SHAP is strictly easier in soft-voting tree ensembles compared to hard-voting tree ensembles.} Many of our findings highlight significant differences in the complexity of obtaining SHAP for tree ensembles used in \emph{regression} versus \emph{classification}. These findings extend some of the conclusions made in~\citep{huangupdates} to a more general setting, and underline the feasibility of obtaining SHAP explanations for soft-voting ensembles as compared to hard-voting ensembles, where these prove to be intractable.
\item \textbf{Obtaining SHAP for neural networks is hard, even in highly simplified settings.} We prove various intractbility results (e.g., NP, $\#$P-Hardness, etc.) for different neural networks, demonstrating that this hardness persists in different settings, and even for \emph{baseline} SHAP, the simplest SHAP variant.
\item \textbf{Obtaining Global SHAP is often tractable, despite the additional expectation factor.} Interestingly, we prove that in many scenarios, the tractability of local SHAP extends to the global SHAP variant explored in~\citep{frye20}, despite the added complexity of computing an expectation over the entire data distribution.\end{enumerate}

Due to space limitations, we provide only a brief summary of the proofs for our claims in the paper, with the full detailed proofs included in the appendix.

\section{Background} \label{sec:background}

\subsection{Model Types}

We examine the following types of models: \begin{inparaenum}[(i)] \item linear regression models ($\texttt{LIN}_{\texttt{R}}$); \item decision trees (\texttt{DT}); \item tree ensembles (including both majority voting ensembles such as Random Forests and weighted voting ensembles like XGBoost) used for regression or classification ($\texttt{ENS-DT}_{\texttt{R}}$, $\texttt{ENS-DT}_{\texttt{C}}$); \item (feed-forward/recurrent) neural networks with ReLU/Sigmoid activations (\texttt{NN-SIGMOID}, \texttt{RNN-ReLU}); and \item Weighted Automata (\texttt{WA}). \end{inparaenum} A complete formalization of all models is provided in Appendix~\ref{app:sec:terminology}.

Although WAs may be considered a niche model family within the broader ML community, a significant portion of our paper focuses on establishing tractability results for them. From the perspective of Explainable AI, our interest in WAs stems from two factors:  Firstly, WAs have been proposed as abstractions of neural networks ~\citep{okudono2020weighted, eyraud2024distillation, weiss2019learning, lacroce2021extracting}, offering enhanced transparency. Secondly, and importantly, WAs can be reduced to various other popular ML models like decision trees, linear regression, and tree ensembles, making tractability results for WAs applicable to a wide range of models. We begin by defining N-Alphabet WAs, a generalization of WAs: 


\begin{definition}[N-Alphabet Weighted Automata] \label{def:nletterwa}

For $n,N \in \mathbb{N}$, and $\{\Sigma_{i}\}_{i \in [N]}$, a collection of finite alphabets, an N-Alphabet Weighted Automaton $A$ over the product $\Sigma_{1} \times \ldots \times \Sigma_{N}$ is defined by the tuple $\langle\alpha, {A_{\sigma_{1}, \ldots ,\sigma_{N}}}, \beta\rangle$, where $(\alpha, \beta) \in \mathbb{R}^{n}\times \mathbb{R}^{n}$ are the initial and final state vectors, and $A_{\sigma_{1}, \ldots ,\sigma_{N}} \in \mathbb{R}^{n \times n}$ are transition matrices. The N-Alphabet WA $A$ computes a function over $\Sigma_{1}^{} \times \ldots \times \Sigma_{N}^{}$ as: 
$$f_{A}(w^{(1)}, \ldots , w^{(N)}) \myeq \alpha^{T} \cdot \prod\limits_{i=1}^{L}
   A_{w_{i}^{(1)} \ldots w_{i}^{(N)}} \cdot \beta$$ where $(w^{(1)}, \ldots , w^{(N)}) \in \Sigma_{1}^{} \times \ldots \times \Sigma_{N}^{}$ and $|w^{(1)}| = \ldots = |w^{(N)}| = L$.


\end{definition}

The integer $n$ denotes the size of the WA $A$, denoted as $\texttt{size}(A)$. For $N = 1$, 1-Alphabet WAs match the classical definition of WA~\citep{droste10}, so we use ``WA'' and ``1-letter WA'' interchangeably.


\subsection{Distributions}

Our analysis briefly touches on several types of distributions, primarily for comparison purposes. These include \begin{inparaenum}[(i)]
\item \emph{independent distributions} (\texttt{IND}), where all features are assumed to be independent of one another; \item \emph{empirical distributions} (\texttt{EMP}), i.e., the family of distributions induced from finite datasets; and 
\item \emph{Markovian distributions} (\texttt{MARKOV}), i.e., distributions where the future state depends only on the current state, independent of past states\end{inparaenum}. A full formalization of these distribution families is in Appendix~\ref{app:sec:terminology}.

Previous studies explored the complexity of these three distributions for the conditional SHAP variant~\citep{arenas23, vander21, marzouk24a, huangupdates}. However, the tractability results here apply to a broader class of distributions, specifically those modeled by \emph{Hidden Markov Models} (\texttt{HMMs}). HMMs are more expressive than standard Markovian distributions, as they incorporate hidden states to model sequences influenced by latent variables. In Appendix~\ref{app:reductiontree}, we prove that HMMs include independent, Markovian, and empirical distributions.

\textbf{HMM distributions.} HMMs~\citep{rabiner1986introduction} are a popular class of sequential latent probabilistic models used in various applications~\citep{knill1997hidden, de2007hidden}. For an alphabet $\Sigma$ (also referred to as the observation space), an HMM defines a probabilistic function over $\Sigma^{\infty}$. Formally, an HMM of size $n$ over $\Sigma$ is a tuple $\langle\alpha, T, O\rangle$, where: \begin{inparaenum}[(i)] \item $\alpha \in \mathbb{R}^{n}$, the initial state vector, represents a probability distribution over $[n]$; and \item $T \in \mathbb{R}^{n \times n}$, $O \in \mathbb{R}^{n \times |\Sigma|}$ are stochastic matrices, with each row encoding a probability distribution.\end{inparaenum}

    
\textbf{HMMs and WAs}. The WA formalism in Definition \ref{def:nletterwa}  suffices to cover HMMs, up to reparametrization.  Indeed, it has been proven that the probability that an HMM $M=\langle\alpha, T, O\rangle$ generates a prefix $w \in \Sigma^{*}$ is: $\mathbf{1}^{T} \cdot \prod\limits_{i=1}^{|w|} A_{w_{i}} \cdot \alpha$~\citep{hsu12}, where $\mathbf{1}$ is a row vector with all $1$'s, and for any $\sigma \in \Sigma$, $A_{\sigma} \myeq T \cdot \text{Diag}(O[:,\sigma])$. The matrix $\text{Diag}(O[:, \sigma])$ is the diagonal matrix formed from the column vector in $O$ indexed by $\sigma$. We follow this parameterization of HMMs and assume they are parameterized by the 1-Alphabet WA formalism in Definition~\ref{def:nletterwa}. For \emph{non-sequential models}, we assume the family of HMMs, denoted $\overrightarrow{\text{HMM}}$, represents latent variable models describing probability distributions over random vectors in a finite domain.

\begin{definition} {($\overrightarrow{\emph{HMM}}$)}\label{def:hmmnonsequentialdata}
Let $(n,N) \in \mathbb{N}^{2}$ be two integers, and $\mathbb{D}$ a finite set. An $\overrightarrow{\emph{HMM}}$ over $\mathbb{D}^{n}$ is parameterized by the tuple $\langle\pi, \alpha, \{T_{i}\}_{i \in [n]}, \{O_{i}\}_{i \in [n]}\rangle$, where $\pi$ is a permutation on $[n]$, and for each $i \in [n]$, $T_{i}$ and $O_{i}$ are stochastic matrices in $\mathbb{R}^{N}$ and $\mathbb{R}^{N \times |\mathbb{D}|}$, respectively. A model $M$ in $\overrightarrow{\emph{HMM}}$ computes the following probability distribution over $\mathbb{D}^{n}$:

$$
      P_{M}(x_{1}, \ldots, x_{n}) := \mathbf{1}^{T} \cdot \prod\limits_{i=1}^{n} A_{i,x_{\pi(i)}} \cdot \alpha
      $$
  where: 
  $A_{i,x} \myeq T_{i} \cdot \text{Diag}(O_{i}[:,x])$.
\end{definition}

In essence, models in the family $\overrightarrow{\text{HMM}}$ are non-stationary HMMs where observations are ordered by a permutation $\pi$. They include a stopping probability mechanism, terminating after the $n$-th symbol with probability 1. Like HMMs, $\overrightarrow{\text{HMM}}$ includes independent, empirical, and Markovian distributions (see proof in  Appendix~\ref{app:reductiontree}).


\subsection{The (Many) Shapley Values}


\textbf{Local Shapley values.} 
Let there be a discrete input space $\mathcal{X} = \mathcal{X}_{1} \times \ldots \times \mathcal{X}_{n}$ and a model $f$, which can be either a regression model $f:\mathcal{X}\to \mathbb{R}$ or a classification model $f:\mathcal{X}\to[c]$ for a certain set of classes $[c]$ ($c\in \mathbb{N}$), along with a specific \emph{local} instance $\x \in \mathcal{X}$. Then, the (local) Shapley value attribution for a feature $i \in [n]$ with respect to $\langle f,\x\rangle$ is defined as:

\begin{equation}
\begin{aligned}
\phi(f,\x,i)\myeq\sum_{S \subseteq [n] \setminus \{i\}} \frac{|S|!\cdot (n - |S| - 1)!}{n!} \nonumber \cdot \\
\!\!\!\!\!\! \left[ v(f,\x, S \cup \{i\}) - v(f,\x,S) \right] \label{eq:genericshap}
\end{aligned}
\end{equation}

where $v$ is referred to as the \emph{value function} of $\phi$. The primary versatility of Shapley values lies in the various ways $v$ can be defined. Typically, Shapley values are computed with respect to a distribution $\mathcal{D}_p$ over $\mathcal{X}$, meaning $v$ is determined by this distribution, i.e., $v(f,\x,S,\mathcal{D}_p)$. A common definition of $v$ is through conditional expectation, referred to here as \emph{Conditional SHAP}, also known as Conditional Expectation SHAP (CES)~\citep{sundararajan20b}:


\begin{equation}
    v_c(f,\x,S, \mathcal{D}_p) \myeq \mathbb{E}_{\z \sim \mathcal{D}_p} \left[ f(\z) | \z_S=\x_S \right]
\end{equation}

where $\z_S=\x_S$ indicates that the values of the features $S$ in $\z$ are set to those in $\x$. Another approach for computing the value function is \emph{Interventional SHAP}~\citep{janzing20a}, also known as Random-Baseline SHAP~\citep{sundararajan20b}, used in practical algorithms like KernelSHAP, LinearSHAP, and TreeSHAP~\citep{lundberg2017, lundbergnature}. In interventional SHAP, when a feature $j \in \overline{S}$ is missing, it is replaced with a reference value independently drawn from a predefined distribution, breaking dependencies with other features. Formally:


\begin{equation}
    v_i(f,\x,S, \mathcal{D}_p) \myeq \mathbb{E}_{\z \sim \mathcal{D}_p} \left[ f(\x_{S}; \z_{\Bar{S}}) \right]
\end{equation}

where $(\x_{S}; \z_{\Bar{S}})$ represents a vector in which the features in $S$ are fixed to the values in $\x$, and the features in $\overline{S}$ are fixed to the values in $\z$. When the distribution $\mathcal{D}_p$ assumes feature independence, interventional and conditional SHAP coincide, i.e., $v_i(f,\x,S, \mathcal{D}_p)=v_c(f,\x,S, \mathcal{D}_p)$~\citep{sundararajan20b}. However, this alignment does not typically hold in many real-world distributions.

Finally, instead of defining Shapley values with respect to a distribution $\mathcal{D}_p$, they can be defined using an auxiliary baseline $\z^{\text{ref}}$ that captures the ``missingness'' of features in $\overline{S}$. \emph{Baseline SHAP}~\citep{sundararajan20b} is defined as follows:


\begin{equation}
    v_b(f,\x,S, \z^{\text{ref}}) \myeq f(\x_{S}; \z^{\text{ref}}_{\Bar{S}}) 
\end{equation}


By substituting these value function definitions into the generic Shapley value formula (Equation~\ref{eq:genericshap}), we obtain $\phi_c(f,i,\mathcal{D}_p)$, $\phi_i(f,i, \mathcal{D}_p)$, and $\phi_b(f,i, \z^{\text{ref}})$, corresponding to the (local) conditional, interventional, and baseline Shapley values, respectively.


\textbf{Global Shapley values.} Shapley values $\phi(f,\x,i)$ offer \emph{local} explainability for the model's prediction on a specific data point $\x$. One approach to deriving a global importance indicator for input features from their local Shapley values consists at aggregating these values over the input space, weighted by the target data generating distribution~\citep{frye20}:\footnote{Several other methods exist for computing global feature importance using SHAP~\citep{covert2020understanding, lundberg2017, lundbergnature}, but they fall outside the scope of this work.}

\begin{equation}
  \label{eq:VBshap}
   \Phi(f,i, \mathcal{D}_p) \myeq \mathbb{E}_{\x \sim \mathcal{D}_p}[\phi(f,\x,i)]
\end{equation}



Note that the global Shapley value is always computed with respect to a distribution $\mathcal{D}_p$, as the inputs $\x$ are aggregated over it. This gives rise to $\Phi_c(f,i,\mathcal{D}_p)$, $\Phi_i(f,i,\mathcal{D}_p)$, and $\Phi_b(f,i,\mathcal{D}_p,\z^{\text{ref}})$, representing the (global) conditional, interventional, and baseline Shapley values, respectively. For \emph{sequential models}, $f$ accepts input vectors of arbitrary length $n$, i.e., $|\x|=n$. While local Shapley values are computed for a specific input $\x$, global Shapley values pose a challenge as they are input-independent. One approach is to fix the feature space length $n$ and compute the global Shapley value for an input of that size. In this case, the global Shapley value $\Phi$ incorporates $n$ as part of its input: $\Phi(f,i,\mathcal{D}_p,n) \myeq \mathbb{E}_{\x \sim \mathcal{D}_p^{(n)}}[\phi(f,\x,i)]$, where $i \in [n]$, $\mathcal{D}_p$ is a probability distribution over an infinite alphabet $\Sigma^{\infty}$, and $\mathcal{D}_p^{(n)}$ is the probability of generating an infinite sequence prefixed by $\x \in \Sigma^{n}$.

\textbf{Shapley values for sequential models.} For complexity results on sequential models (WAs and RNNs), we build on prior work~\citep{marzouk24a}, which uses the \emph{pattern} formalism to analyze the complexity of obtaining Shapley values for these models. Formally, let $\Sigma$ be a finite alphabet, with its elements called symbols. The set of all finite (and infinite) sequences from $\Sigma$ is denoted by $\Sigma^{}$ (and $\Sigma^{\infty}$). For any integer $n > 0$, $\Sigma^{n}$ represents sequences of length $n$. For a sequence $w \in \Sigma^{}$, $|w|$ is its length, $w_{i:j}$ is the subsequence from the $i$-th to the $j$-th symbol, and $w_{i}$ is its $i$-th symbol. A pattern $p$ over $\Sigma$ is a regular expression in the form $\#^{i_{1}}w_{1} \ldots \#^{i_{n}}w_{n}\#^{i_{n+1}}$, where $\#$ is a placeholder symbol (i.e., $\# = \Sigma$), $\{i_{k}\}_{k \in [n+1]}$ are integers, and $\{w_{k}\}_{k \in [n+1]}$ are sequences over $\Sigma^{*}$. The extended alphabet $\Sigma \cup \#$ is $\Sigma{\#}$. The language of a pattern $p$ is $L{p}$, with $|p|$ as its length and $|p|_{\#}$ indicating the number of $\#$ symbols. 

We define two operators on patterns: \begin{inparaenum}[(i)] \item The \emph{swap} operator, which takes a tuple $(p,\sigma,i) \in \Sigma_{\#}^{*} \times \Sigma \times \mathbb{N}$ with $i \leq |p|$, and returns a pattern where the i-th symbol of $p$ is replaced by $\sigma$. For example, with $\Sigma = \{0,1\}$: $\texttt{swap}(0\#0\#, 1,2) = 0\mathbf{1}0\#$; \item The \emph{do} operator, which takes a tuple $(p,w',w) \in \Sigma_{\#}^{*} \times \Sigma^{*} \times \Sigma^{*}$ with $|w|$ = $|w'|$ = $|p|$, and returns a sequence $u$ where $u_{i}$ = $w'_{i}$ if $p_{i}$=$\#$, and $u_{i}$ = $w_{i}$ otherwise. For example, with $\Sigma = \{0,1\}$: $\texttt{do}(0\#0\#,1100,1111) = 1110$ \end{inparaenum}. We represent a coalition $S$ in the SHAP formulation using patterns. For instance, with the alphabet $\Sigma = \{0,1\}$ and the sequence $w = 0011$, the coalition of the first and third symbols is represented by the pattern $0\#1\#$.

\textbf{SHAP as a computational problem.} As outlined in the introduction, this work aims to provide a comprehensive computational analysis of the SHAP problem across several dimensions: \begin{inparaenum}[(i)] \item the class of models being interpreted; \item the underlying data-generating distributions; \item the specific SHAP variant; and \item its scope (global or local)\end{inparaenum}. Each combination of these dimensions gives rise to a distinct formal computational problem. To navigate this multi-dimensional landscape of computational problems, we adopt the following notation:
A (variant) of the SHAP computational problem shall be denoted as \texttt{(LOC|GLOB)-(I|B|C)-SHAP}($\mathcal{M}$,$\mathcal{P}$), where \texttt{LOC} and \texttt{GLOB} refer to local and global SHAP, respectively, while \texttt{I}, \texttt{B}, and \texttt{C} correspond to the interventional, baseline, and conditional SHAP variants. The symbols $\mathcal{M}$ and $\mathcal{P}$ represent the class of models and the class of feature distributions, respectively. Under this notation, \texttt{LOC-I-SHAP}(\texttt{WA}, \texttt{HMM}) refers to the problem of computing local interventional SHAP for the family of weighted automata under Hidden Markov Model distributions.

 A variant of the SHAP computational problem takes as input instance a model $M \in \mathcal{M}$, a data-generating distribution 
$\text{P} \in \mathcal{P}$ \footnote{Note that in the formulation of Local Baseline SHAP variants of the SHAP problem, the data-generating distribution $P$ is replaced by a reference input instance.}, an index specifying the input feature of interest, and, in the case of local SHAP variants, the  model's input undergoing explanatory analysis. The computational complexity of the problem is measured with respect to the size of $M$, the size of $P$, and the dimensionality of the input space\footnote{For sequential models, where inputs are sequences, we assume the input space's dimensionality equals the sequence length under analysis.}. A variant is considered tractable if it can be solved in polynomial time with respect to these parameters.

For completeness, a summary of all complexity classes discussed in this article (PTIME, NP, coNP, NTIME, and \#P) is provided in Appendix~\ref{app:sec:terminology}.

\section{Positive Complexity Results} \label{sec:tractable}

This section highlights configurations that allow polynomial-time computation of various Shapley value variants. We first show that both local and global interventional and baseline SHAP values for WAs under HMM-modeled distributions can be computed in polynomial time (Theorem~\ref{thm:shapwa}). By reductions, this result extends to various other ML models (Theorem~\ref{cor:reductions}).

\subsection{Tractability for WAs} \label{subsec:shapwa}
This main result of this subsection is given in the following theorem:
\begin{theorem} \label{thm:shapwa}
    The following computational problems, which include \emph{\texttt{LOC-I-SHAP}}\emph{\texttt{(WA}},\emph{\texttt{HMM)}}, \emph{\texttt{GLO-I-SHAP}}\emph{\texttt{(WA}},\emph{\texttt{HMM)}}, \emph{\texttt{LOC-B-SHAP}}\emph{\texttt{(WA)}}, as well as \emph{\texttt{GLO-B-SHAP}}\emph{\texttt{(WA}},\emph{\texttt{HMM)}} are poly-time computable with respect to the size of the \text{WA}, the size of the \text{HMM}, the sequence length and the size of the alphabet.
\end{theorem}

The remainder of this section provides a proof sketch for Theorem \ref{thm:shapwa}, with the complete proof available in Appendix~\ref{app:shapwa}. The proof is constructive and draws heavily on techniques from the theory of rational languages \citep{berstel88}. We begin by linking the previously defined swap and do operators for patterns to the computation of (interventional) SHAP values. The corresponding relations for baseline SHAP are provided in Appendix~\ref{app:shapwa} due to space constraints.

\begin{lemma}
For a sequential model $f$, a string $w$ (representing an input $\x\in\mathcal{X}$), a pattern $p$ (representing a coalition $S\subseteq [n]$), and a distribution $\mathcal{D}_p$ over $\mathcal{X}$, then the following relations hold:
\begin{equation}
\begin{aligned}
        v_{i}(f,w,p,\mathcal{D}_p) = \mathbb{E}_{\z \sim \mathcal{D}_p^{|w|}} \left[ f(\texttt{do}(p,\z,w)) \right]; \\
        \phi_i(f,w,i,\mathcal{D}_p) = \mathbb{E}_{p \sim \mathcal{P}_i^{\x}} [v_{i}(f,w,\texttt{swap}(p,w_{i},i),\mathcal{D}_p) \nonumber 
 \\ - v_{i}(f,w,p,\mathcal{D}_p)] \quad\quad\quad\quad\quad\quad
\end{aligned}
\end{equation}
where:
\begin{equation*}
\begin{aligned}
\mathcal{P}_{i}^{w}(p)  \myeq \begin{cases}
\frac{(|p|_{\#}-1)! \cdot (|w| - |p|_{\#})!}{|w|!} & \text{if}~~ w \in L_{p} \\ 
0 & \text{otherwise}
\end{cases}
\end{aligned}
\end{equation*}
\end{lemma}

To develop an algorithm for computing $v_i$ and $\phi_i$ in polynomial time for the class of WA under HMM distributions, we utilize two operations on N-Alphabet WA, parameterized by the input instance:

\begin{definition} \label{def:projectionoperation}
  Let $N > 0$ be an integer and $\{\Sigma_{i}\}_{i \in [N]}$ be a collection of $N$ alphabets.
  \begin{enumerate}  
  \item \textbf{The Kronecker product operation} of two $N$-Alphabet WAs $A$ and $B$ over $\Sigma_{1} \times \ldots \times \Sigma_{N}$ at index $i \in [N]$ returns an $N$-Alphabet WA over $\Sigma_{1} \times \ldots \times \Sigma_{N}$, denoted $A \otimes B$ 
  implementing the function:
  \begin{align*} 
  f_{A \otimes B}(w^{(1)},\ldots, w^{(n)}) :=& f_{A}(w^{(1)},\ldots, w^{(n)}) \\
  & \cdot f_{B}(w^{(1)},\ldots, w^{(n)})
  \end{align*}
  \item \textbf{The projection operation} of a 1-Alphabet WA $A$ over an N-Alphabet WA $T$ over $\Sigma_{1} \times \ldots \times \Sigma_{N}$ at index $i \in [N]$, returns an (N-1)-Alphabet WA over $\Sigma_{1} \times \ldots \times \Sigma_{i-1} \times \Sigma_{i+1} \times \ldots \times \Sigma_{N}$, denoted $\Pi_{i}(A,T)$, implementing the following function:
       \begin{align*}
           g(w^{(1)}, \ldots, w^{(i-1)}, w^{(i+1)}, w^{(N)}) := ~~~~~~~~~~~~  \\ 
           \sum\limits_{w \in \Sigma_{i}^{L}} f_{A}(w) \cdot f_{T}(w^{(1)}, \ldots w^{(i)}, w, w^{(i+1)}, \ldots, w^{(N)})
        \end{align*}
       where $(w^{(1)}, \ldots, w^{(i-1)}, w^{(i+1)}, \ldots w^{(N)}) \in \Sigma_{1}^{*} \times \ldots \times \Sigma_{i-1}^{*} \times \Sigma_{i+1}^{*} \ldots \times \Sigma_{N}^{*}$ such that $|w^{(1)}| = \ldots =|w^{(i-1)}| = |w^{(i+1)}| = \ldots = |w^{(N)}| = L$.
       \end{enumerate}
\end{definition}

Note that if $T$ is a 1-Alphabet WA, then the returned model $\Pi_{1}(A,T)$ is a 0-Alphabet WA. For simplicity, we denote: $\Pi_{1}(A,T) \myeq \sum\limits_{w \in \Sigma_{1}^{}}$ $f_{A}(w) \cdot f_{T}(w)$, yielding a scalar. Additionally, we define $\Pi_{0}(A) \myeq \Pi_{1}(A,\mathbf{1})$, where $\mathbf{1}$ is a WA that assigns $1$ to all sequences in $\Sigma^{}$. The next proposition provides a useful intermediary result, with its proof in Appendix~\ref{app:shapwa}:



 
\begin{proposition}\label{prop:efficentoperations}
If $N=O(1)$, the projection and Kronecker product operations are poly-time computable.
\end{proposition}

\textbf{Interventional and Baseline SHAP of WAs in terms of N-Alphabet WA operators.}
As mentioned earlier, the algorithmic construction for $\texttt{LOC-I-SHAP}(\texttt{WA}, \texttt{HMM})$, and $\texttt{GLO-I-SHAP}(\texttt{WA}, \texttt{HMM})$ will take the form of efficiencly computable operations over N-Alphabet WAs parameterized by the input instance of the corresponding problem. The main lemma formalizing this fact is given as follows:
\begin{lemma} \label{lemma:shapasoperations}
Fix a finite alphabet $\Sigma$. Let $f$ be a WA over $\Sigma$, and consider a sequence $(w, w^{\text{reff}}) \in \Sigma^{*} \times \Sigma^{}$ (representing an input and a basline $\x, \x^{\text{reff}} \in \mathcal{X}$) such that $|w| = |w^{\text{reff}}|$. Let $i \in [|w|]$ be an integer, and $\mathcal{D}_P$ be a distribution modeled by an HMM over $\Sigma$. Then:
         {\small 
            \begin{align*}
         \phi_i
         (f,w,i,\mathcal{D}_P) = \quad\quad\quad\quad\quad\quad\quad\quad\quad\quad\quad\quad\quad\quad\quad\quad \\ \Pi_{1} (A_{w,i}, \Pi_{2}(\mathcal{D}_P, \Pi_{3}(f,T_{w,i}) 
          - \Pi_{3}(f,T_{w}) ) ); \quad\quad \\
            \Phi_i(f,i,n,\mathcal{D}_P) = \quad\quad\quad\quad\quad\quad\quad\quad\quad\quad\quad\quad\quad\quad\quad\quad \\ 
            \Pi_{0} ( \Pi_{2}(\mathcal{D}_P, A_{i,n} \otimes \Pi_{2}(\mathcal{D}_P, 
             \Pi_{3}(f,T_{i}) - \Pi_{3}(f,T)))); \quad\\
            \phi_b(f,w,i,w^{\text{reff}}) = \quad\quad\quad\quad\quad\quad\quad\quad\quad\quad\quad\quad\quad\quad\quad\quad \\ \Pi_{1} (A_{w,i}, \Pi_{2}(f_{w^{\text{reff}}}, \Pi_{3}(f,T_{w,i}) 
          - \Pi_{3}(f,T_{w})));\quad \\
            \Phi_b(f,i,n,w^{\text{reff}},\mathcal{D}_P) = \quad\quad\quad\quad\quad\quad\quad\quad\quad\quad\quad\quad\quad\quad \\ \Pi_{0} ( \Pi_{2}(\mathcal{D}_P, A_{i,n} \otimes \Pi_{2}(f_{w^{\text{reff}}} ,
            \Pi_{3}(f,T_{i}) - \Pi_{3}(f,T)) ))
               \end{align*}
             } where:
    \begin{itemize}
        \item $A_{w,i}$ is a 1-Alphabet WA over $\Sigma_{\#}$ implementing the uniform distribution over coalitions excluding the feature $i$ (i.e., $f_{A_{w,i}} = \mathcal{P}_{i}^{w}$);
        \item $T_{w}$ is a 3-Alphabet \emph{WA} over $\Sigma_{\#} \times \Sigma \times \Sigma$ implementing the function: $ g_{w}(p,w',u) := I(\texttt{do}(p,w',w) = u)$.
        \item $T_{w,i}$ is a 3-Alphabet \emph{WA} over $\Sigma_{\#} \times \Sigma \times \Sigma$ implementing the function: $            g_{\x,i}(p,w',u) := I(\texttt{do}(\texttt{swap}(p,w_{i},i),w',w) = u)$.
        \item $T$ is a 4-Alphabet \emph{WA} over $\Sigma_{\#} \times \Sigma \times \Sigma \times \Sigma$ given as: $g(p,w',u,w) := g_{w}(p,w',u)$.
        \item $T_{i}$ is a 4-Alphabet \emph{WA} over $\Sigma_{\#} \times \Sigma \times \Sigma \times \Sigma$ given as: $ g_{i}(p,w',u,w) := g_{w,i}(p,w',u)$.
        \item $A_{i,n}$ is a 2-Alphabet \emph{WA} over $\Sigma_{\#} \times \Sigma$ implementing the function:
         $g_{i,n}(p,w) := I(p \in \mathcal{L}_{i}^{w}) \cdot \mathcal{P}_{i}^{w}(p)$,
         where $|w| = |p| = n$.
       \item $f_{w^{\text{reff}}}$ is an \emph{HMM} such that the probability of generating $w^{\text{reff}}$ as a prefix is equal to $1$.
    \end{itemize}
\end{lemma}

The proof of Lemma~\ref{lemma:shapasoperations} is in Appendix~\ref{app:shapwa}. In essense, it reformulates the computation of both local and global interventional and baseline SHAP using operations on N-Alphabet WAs, depending on the input instance, particularly involving $A_{w,i}$, $T_{w}$, $T_{w,i}$, $T$, $T_{i}$, $A_{i,n}$, and $f_{w^{\text{ref}}}$. The final step to complete the proof of Theorem~\ref{thm:shapwa} is to show that these WAs can be constructed in polynomial time relative to the input size.





\begin{proposition} \label{prop:nletterwaconstruction}
   The N-Alphabet WAs $A_{w,i}$, $T_{w}$, $T_{w,i}$, $T$, $T_{i}$, $A_{i,n}$ (defined in Lemma \ref{lemma:shapasoperations}) and the HMM $f_{w^{\text{reff}}}$ can be constructed in polynomial time with respect to $|w|$ and $|\Sigma|$.
\end{proposition}

The full proof of Proposition \ref{prop:nletterwaconstruction} can be found in Appendix~\ref{app:shapwa}. Theorem \ref{thm:shapwa} is a direct corollary of Lemma \ref{lemma:shapasoperations}, Proposition \ref{prop:efficentoperations}, and Proposition \ref{prop:nletterwaconstruction}.


\subsection{Tractability for other ML models.} \label{subsec:tree2WA}

Beyond the proper interest of the result in Theorem \ref{thm:shapwa} regarding WAs, it also yields interesting results about the computational complexity of obtaining interventional and baseline SHAP variants for other popular ML models which include decision trees, tree ensembles for regression tasks (e.g. Random forests or XGBoost), and linear regression models:


\begin{theorem} \label{cor:reductions}
    Let $\mathbb{S}:= \{ \emph{\texttt{LOC}}, \emph{\texttt{GLO}} \}$, $\mathbb{V} := \{\texttt{\emph{B}}, \texttt{\emph{I}}\}$, $\mathbb{P}:= \{ \texttt{\emph{EMP}}, \overrightarrow{\emph{\texttt{HMM}}} \} $, and $\mathbb{F}:= \{\emph{\texttt{DT}}, \emph{\texttt{ENS-DT}}_{\emph{\texttt{R}}}, \emph{\texttt{Lin}}_{\emph{\texttt{R}}}\}$. Then, for any \emph{\texttt{S}} $\in \mathbb{S}$, $\texttt{\emph{V}} \in \mathbb{V}$, $\texttt{\emph{P}} \in \mathbb{P}$, and $\emph{\texttt{F}}\in \mathbb{F}$ the problem  $\emph{\texttt{S-V-SHAP}}(\emph{\texttt{F}}, \texttt{\emph{P}})$ can be solved in polynomial time.
\end{theorem}

The proof of Theorem~\ref{cor:reductions} is provided in Appendix~\ref{app:reductiontree}, where a poly-time construction of either a decision tree, an ensemble of decision trees for regression, or a linear regression model into a WA is detailed. This result brings forward two interesting outcomes. First, it expands the distributional assumptions of some popular SHAP algorithms such as LinearSHAP and TreeSHAP. Let us denote $\texttt{TREE}$ as the family of all decision trees $\texttt{DT}$ and regression tree ensembles $\texttt{ENS-DT}_{\texttt{R}}$. Then:

\begin{corollary}
\label{treeshap_corollary}
    While the \emph{TreeSHAP}~\citep{lundbergnature} algorithm solves \texttt{\emph{LOC-I-SHAP}}\texttt{\emph{(TREE}},\texttt{\emph{EMP)}} and \texttt{\emph{GLO-I-SHAP}}\texttt{\emph{(TREE}},\texttt{\emph{EMP)}} in poly-time, Theorem~\ref{cor:reductions} establishes that \texttt{\emph{LOC-I-SHAP}}\texttt{\emph{(TREE}},$\overrightarrow{\emph{\texttt{\text{HMM}}}}$\texttt{\emph{)}} and \texttt{\emph{GLO-I-SHAP}}\texttt{\emph{(TREE}},$\overrightarrow{\emph{\texttt{\text{HMM}}}}$\texttt{\emph{)}} can be solved in poly-time.
\end{corollary}

Since empirical distributions are strictly contained within HMMs, i.e., $\texttt{EMP}\subsetneq \texttt{HMM}$ and $\texttt{EMP}\subsetneq \overrightarrow{\texttt{HMM}}$ (see proof in Appendix~\ref{app:reductiontree}), Corollary~\ref{treeshap_corollary} significantly broadens the distributional assumption of the TreeSHAP algorithm beyond just empirical distributions. Similar conclusions can be drawn for LinearSHAP:

\begin{corollary}
    While the \emph{LinearSHAP}~\citep{lundberg2017} algorithm solves \texttt{\emph{LOC-I-SHAP}}$\texttt{\emph{(LIN}}_\texttt{\emph{R}}$,\texttt{\emph{IND}}) in polynomial time, Theorem~\ref{cor:reductions} establishes that \texttt{\emph{LOC-I-SHAP}}$\texttt{\emph{(LIN}}_\texttt{\emph{R}}$,$\overrightarrow{\emph{\texttt{\text{HMM}}}}$\texttt{\emph{)}} and \texttt{\emph{GLO-I-SHAP}}$\texttt{\emph{(LIN}}_\texttt{\emph{R}}$,$\overrightarrow{\emph{\texttt{\text{HMM}}}}$\texttt{\emph{)}} can be solved in polynomial time.
\end{corollary}

Which, again, demonstrates the expansion of the distributional assumption of LinearSHAP, as $\texttt{IND}\subsetneq \overrightarrow{\texttt{HMM}}$ (see Appendix~\ref{app:reductiontree}). 
Lastly, Theorem~\ref{cor:reductions} establishes a \emph{strict computational complexity gap} between computing \emph{conditional} SHAP, which remains intractable even for simple models like decision trees under HMM distributions (and even under empirical distributions~\citep{vander21}), whereas computing both local interventional and baseline SHAP values are shown to be tractable. This suggests that interventional and baseline SHAP are strictly more efficient to compute in these settings.


\begin{corollary}
    If $f\in\{\emph{\texttt{WA}},\emph{\texttt{DT}},\emph{\texttt{ENS-DT}}_{\emph{\texttt{R}}},\emph{\texttt{Lin}}_{\emph{\texttt{R}}}\}$, $\mathcal{D}_p:=\overrightarrow{\emph{\texttt{\text{HMM}}}}$ (or $\mathcal{D}_p:=\emph{\texttt{HMM}}$ for \emph{\texttt{WA}}), and assuming P$\neq$NP, then computing local interventional SHAP or local baseline SHAP for $f$ \emph{is strictly more computationally tractable} than computing local conditional SHAP for $f$.
\end{corollary}



\section{Negative Complexity Results} \label{sec:intractable}

\subsection{When Interventional SHAP is Hard}
\label{subsec:rnnreluhard}
In this subsection, we present intractability results for interventional SHAP.

\begin{theorem} \label{thm:relushap}
    The decision version of the problem \emph{\texttt{LOC-I-SHAP}}\emph{\texttt{(RNN-ReLu}, \texttt{IND)}} is \emph{NP-Hard}.
\end{theorem}

Recall that an RNN-ReLu model $R$ over $\Sigma$ is parametrized as: $\langle h_{init}, W, \{v_{\sigma}\}_{\sigma \in \Sigma}, O\rangle$ and processes a sequence sequentially from left-to-right such that $h_{\epsilon} = h_{init}$, $h_{w'\sigma} = \emph{ReLu}(W \cdot h_{w'} + v_{\sigma} )$, and $f_{R}(w) = I(O^{T} \cdot h_{w} \geq 0)$. The proof of Theorem~\ref{thm:relushap} leverages the efficiency axiom of Shapley values. Specifically, for a sequential binary classifier $f$ over alphabet $\Sigma$, an integer $n$, and $i \in [n]$, the efficiency property can be expressed as~\citep{arenas23}:
\begin{equation} \label{eq:efficiency}
\sum\limits_{i=1}^{n} \phi_i(f,\x,i,P_{unif}) =  f(\x) - 
 P_{unif}^{(n)} (f(\x) = 1) \end{equation}


where $P_{unif}$ denotes the uniform distribution over $\Sigma^{\infty}$. This property will be used to reduce the $\texttt{EMPTY}$ problem to the computation of interventional SHAP in our context. The $\texttt{EMPTY}$ problem is defined as follows: Given a set of models $\mathcal{F}$, $\texttt{EMPTY}$ takes as input some $f \in \mathcal{F}$ and an integer $n > 0$ and asks if $f$ is empty on the support $\Sigma^{n}$ (i.e., is the set $\{ \x \in \Sigma^{n}:~f(\x) = 1 \}$ empty?). The connection between local interventional SHAP and the emptiness problem is outlined in the following proposition:




\begin{proposition} \label{prop:modelcountingshap}
Let $\mathcal{F}$ be a class of sequential binary classifiers. Then, \emph{\texttt{EMPTY($\mathcal{F}$)}} can be reduced in polynomial time to the problem $\emph{\texttt{LOC-I-SHAP}}(\mathcal{F}, \emph{\texttt{IND}})$.
\end{proposition}

Proposition \ref{prop:modelcountingshap} suggests that proving the NP-Hardness of the problem $\texttt{EMPTY}(\texttt{RNN-ReLu})$ is a sufficient condition to yield the result of Theorem \ref{thm:relushap}. This condition is asserted in the following lemma:
\begin{lemma} \label{lemma:emptyrnnrelu}
\emph{\texttt{EMPTY}}\emph{\texttt{(RNN-ReLu)}} is NP-Hard. 
\end{lemma}

The proof of Lemma~\ref{lemma:emptyrnnrelu} is done via a reduction from the \emph{closest string problem} (CSP), which is NP-Hard~\citep{li2000closest}. CSP takes as input a set of strings $S := \{w_{i}\}_{i \in [m]}$ of length $n$ and an integer $k > 0$. The goal is to determine if there exists a string $w' \in \Sigma^{n}$ such that for all $w_{i} \in S$, $d_{H}(w_{i}, w') \leq k$, where $d_{H}(.,.)$ is the Hamming distance: $d_{H}(w,w') := \sum\limits_{i=1}^{ [|w|]} \mathrm{1}_{w_{j}}(w'_{j})$. We conclude our reduction by proving the following proposition, with the proof in Appendix~\ref{app:ISHAPRNN}:


\begin{proposition}
    \emph{\texttt{CSP}} can be reduced in polynomial time to \emph{\texttt{EMPTY}}\emph{\texttt{(RNN-ReLu)}}.
\end{proposition}

\subsection{When Baseline SHAP is Hard} \label{subsec:baseline}
In this subsection, we turn our focus to Baseline SHAP, which can be seen as a special case of Interventional SHAP when confined to empirical distributions induced by a reference input $\x^{\text{reff}}$. Hence, it is expected to be computationally simpler to compute than Interventional SHAP. However, we identify specific scenarios where calculating Baseline SHAP remains computationally challenging:
\begin{theorem} \label{thm:intractable} \begin{inparaenum}[(i)] \item Unless P=co-NP, the problem \texttt{\emph{LOC-B-SHAP}}\texttt{\emph{(NN-SIGMOID)}} can not be solved in polynomial time; \item The decision versions of the problems $\emph{\texttt{LOC-B-SHAP}\texttt{(RNN-ReLu)}}$ and \emph{\texttt{LOC-B-SHAP}$\texttt{(ENS-DT}_{\texttt{C}}\texttt{)}$} 
are co-NP-Hard and NP-Hard respectively.
    \end{inparaenum}
\end{theorem}

We begin with results for \texttt{LOC-B-SHAP}(\texttt{NN-SIGMOID}) and \texttt{LOC-B-SHAP}(\texttt{RNN-ReLu}). Our proofs reduce from the Dummy Player problem of Weighted Majority Games~\citep{freixas2011complexity}:  

 
\begin{definition} \label{def:wmg}
A Weighted Majority Game (WMG) $G$ is a coalitional game defined by the tuple $  \langle N,\{n_{i}\}_{i \in [N]}, q\rangle$, where: \begin{inparaenum}[(i)] \item $N$ is the number of players; \item $n_{i}$ is the voting power of player $i$, for $i \in [N]$; \item $q$ is the winning quota. \end{inparaenum} The value function $v_{G}(S)$ is 1 if $\sum_{i \in S} n_{i} \geq q$, and 0 otherwise.

\end{definition}


A \emph{dummy player} in a WMG $G$ is one whose voting power adds no value to any coalition. Formally, $i$ is a dummy in $G$ if $\forall S \subseteq [N] \setminus \{i\}$, it holds that $v_{G}(S \cup \{i\}) = v_{G}(S)$. Determining if a player $i$ is a dummy in $G$ is co-NP-Complete \citep{freixas2011complexity}.

\textbf{Reducing the dummy problem in WMG to both of the problems \texttt{LOC-B-SHAP}(\texttt{NN-SIGMOID}) and \texttt{LOC-B-SHAP}(\texttt{RNN-ReLu}).} 
Informally, given a WMG $G := \langle N, \{n_{j}\}_{j \in [N]}, q\rangle$, the reduction constructs a model $f_{G}$ over the input set $\{0,1\}^{N}$ from the target model family to simulate $G$'s value function using chosen input instances $\x,\x^{\text{reff}} \in \{0,1\}^{n}$, i.e., $f_{G}(\x_{S}; \x_{\bar{S}}^{\text{reff}}) \approx v(S)$. The properties of this reduction for both model families are formally stated as follows:

\begin{proposition} \label{prop:reductionsigmoid} 

There are poly-time algorithms that: \begin{inparaenum}[(i)] \item Given a WMG $G$ and player $i$, return a sigmoidal neural network $f_{G}$ over $\{0,1\}^{N}$, $\x,\x^{\text{reff}} \in \{0,1\}^{N}$ and $\epsilon \in \mathbb{R}$ such that $i$ is not dummy iff $\phi_{b}(f_{G},i,\x,\x^{\text{reff}}) > \epsilon$; \item Given a WMG $G$ and player $i$, return an RNN-ReLU $f_{G}$ over $\{0,1\}^{N}$, $\x,\x^{\textbf{reff}} \in \{0,1\}^{N}$ such that $i$ is not dummy iff $\phi_{b}(f_{G},i,\x,\x^{\text{reff}}) > 0$. \end{inparaenum}

\end{proposition}

The complexity of computing \texttt{LOC-B-SHAP}(\texttt{SIGMOID}) and \texttt{LOC-B-SHAP}(\texttt{RNN-ReLu}) from Theorem \ref{thm:intractable} are corollaries of Proposition \ref{prop:reductionsigmoid}.



\textbf{The problem \texttt{LOC-B-SHAP}($\texttt{ENS-DT}_{\texttt{C}}$) is NP-Hard.} 
This segment is dedicated to proving the remaining point of Theorem \ref{thm:intractable} stating the NP-Hardness of \texttt{LOC-B-SHAP}($\texttt{ENS-DT}_{\texttt{C}}$). We prove this by reducing from the classical NP-Complete 3-SAT problem. The reduction strategy is illustrated in Algorithm \ref{alg:SAT2b-SHAP}.

\begin{algorithm}
\caption{\texttt{3-SAT} to \texttt{LOC-B-SHAP}($\texttt{ENS-DT}_{\texttt{C}}$)}
\label{alg:SAT2b-SHAP}
\begin{algorithmic}[1]
\REQUIRE A CNF Formula $\Psi$ of $m$ clauses 
\ENSURE An instance of \texttt{LOC-B-SHAP}($\texttt{ENS-DT}_{\texttt{C}}$)
\STATE $\x \leftarrow [1, \ldots , 1]$
\STATE $\x^{\text{reff}} \leftarrow [0, \ldots , 0]$
\STATE $i \leftarrow n+1$
\STATE $\mathcal{T} \leftarrow \emptyset$
\FOR{$j \in [1,m]$}
 \STATE Construct a DT $T_{j}$ that assigns $1$ to variable assignments satisfying the formula: $C_{j} \land \x_{n+1}$.
 \STATE $\mathcal{T} \leftarrow \mathcal{T} \cup \{T\}_{j}$
\ENDFOR
 \STATE Construct a null decision tree $T_{\text{null}}$ that assigns a label $0$ to all variable assignments
 \STATE Add $m-1$ copies of $T_{\text{null}}$ to $\mathcal{T}$
\RETURN $\langle\mathcal{T},i,\x,\x^{\text{reff}}\rangle$
\end{algorithmic}
\end{algorithm}

The next proposition establishes a property of $\texttt{ENS-DT}_{\texttt{C}}$ used in Lemma~\ref{lemma:sat2bshap} to derive our complexity results:

\begin{proposition} \label{prop:dnf2bshaprf}
    Let $\Psi$ be an arbitrary \emph{CNF} formula over $n$ boolean variables, and  $\mathcal{T}$ be the ensemble of decision trees outputted by Algorithm \ref{alg:SAT2b-SHAP} for the input $\Psi$. 
    We have that $f_{\mathcal{T}}(\x_{1}, \ldots, \x_{n}, \x_{n+1}) = 1$ if $\x_{n+1} = 1 \land \x \models \Psi
          $, and $f_{\mathcal{T}}(\x_{1}, \ldots, \x_{n}, \x_{n+1})=0$, otherwise.
\end{proposition}


Using the result of Proposition~\ref{prop:dnf2bshaprf}, the following lemma directly establishes the NP-hardness of the decision problem for \texttt{LOC-B-SHAP}($\texttt{ENS-DT}_{\texttt{C}}$):

\begin{lemma} \label{lemma:sat2bshap}
    Let $\Psi$ be an arbitrary \emph{CNF} formula of $n$ variables, and  $\langle\mathcal{T},n+1,\x,\x^{\text{reff}}\rangle$ be the output of Algorithm \ref{alg:SAT2b-SHAP} for the input $\Psi$. We have that 
    $\phi_{b}(f_{\mathcal{T}}, n+1, \x, \x^{\text{reff}}) > 0$ iff $\exists \x \in \{0,1\}^{n}: ~ \x \models \Psi$. 
\end{lemma}

\section{Generalized Complexity Relations of SHAP Variants} \label{sec:generalized}


While the previous sections presented specific results on the complexity of generating various SHAP variants for different models and distributions, this section aims to establish \emph{general} relationships concerning the complexity of different SHAP variants. We will then leverage these insights to derive corollaries for other SHAP contexts not explicitly covered in the paper.

\begin{proposition} \label{prop:hardnessrelation}
  \begin{inparaenum}[(i)]
    Let $\mathcal{M}$ be a class of models and $\mathcal{P}$ a class of probability distributions such that $\texttt{\emph{EMP}} \preceq_{P}  \mathcal{P}$. 
    Then, \texttt{\emph{LOC-B-SHAP}}($\mathcal{M}$)  $\preceq_{P}$ \texttt{\emph{GLO-B-SHAP}}($\mathcal{M}$,$\mathcal{P}$) and \texttt{\emph{LOC-B-SHAP}}($\mathcal{M}$)  $\preceq_{P}$   \texttt{\emph{LOC-I-SHAP}}($\mathcal{M},\mathcal{P}$).
    \end{inparaenum}
\end{proposition}

In other words, assume a class of probability distributions $\mathcal{P}$ is "harder" (under polynomial reductions) than the class of empirical distributions \texttt{EMP}, i.e., any $P \in \mathcal{P}$ can be reduced in poly-time to some $P' \in$ \texttt{EMP}. Then, global baseline SHAP is at least as hard to compute as local baseline SHAP and local interventional SHAP is at least as hard to compute as local baseline SHAP (both under polynomial reductions). Since $\texttt{EMP} \preceq_{P} \texttt{HMM}$ (proof in Appendix~\ref{app:reductiontree}), these corollaries follow from Theorem~\ref{thm:intractable} and proposition~\ref{prop:hardnessrelation}:



\begin{corollary} \label{cor:globshaphard}
    \begin{inparaenum}[(i)]
    Let there be some $\texttt{\emph{P}} \in \{ \texttt{\emph{EMP}}, \texttt{\emph{HMM}}\}$, and $\overrightarrow{\texttt{\emph{P}}} \in \{\texttt{\emph{EMP}}, \overrightarrow{\texttt{\emph{HMM}}} \}$. Then it holds that:
        \item
        Unless P=co-NP, the problems \texttt{\emph{GLO-B-SHAP}}\texttt{\emph{(SIGMOID}},$\overrightarrow{\texttt{\emph{P}}}$\texttt{\emph{)}} and \texttt{\emph{LOC-I-SHAP}}\texttt{\emph{(NN-SIGMOID}},$\overrightarrow{\texttt{\emph{P}}}$\texttt{\emph{)}} can not be computed exactly in polynomial time;
        \item 
        The decision version of the problems $\texttt{\emph{GLO-B-SHAP}}(\texttt{\emph{RNN-ReLu}}, \texttt{\emph{P}})$ and $\texttt{\emph{LOC-I-SHAP}}(\texttt{\emph{RNN-ReLu}}, \texttt{\emph{P}})$ are co-NP-Hard;
        \item 
        The decision version of the problems $\texttt{\emph{GLO-B-SHAP}\texttt{\emph{(M}},\texttt{\emph{P}}})$ and $\texttt{\emph{LOC-I-SHAP}}\texttt{\emph{(ENS-DT}}_{\texttt{\emph{C}}},\texttt{\emph{P)}}$ are NP-Hard. 
    \end{inparaenum} 
\end{corollary}

\section*{Conclusion}
This paper aims to enhance our understanding of the computational complexity of computing various Shapley value variants. We found that for various ML models --- including decision trees, regression tree ensembles, weighted automata, and linear regression --- both local and global interventional and baseline SHAP can be computed in polynomial time under HMM modeled distributions. This extends popular algorithms, such as TreeSHAP, beyond their empirical distributional scope. We also establish strict complexity gaps between the various SHAP variants (baseline, interventional, and conditional) and prove the intractability of computing SHAP for tree ensembles and neural networks in simplified scenarios. Overall, we present SHAP as a versatile framework whose complexity depends on four key factors: \begin{inparaenum}[(i)] \item model type, \item SHAP variant, \item distribution modeling approach, \item and local vs. global explanations\end{inparaenum}. We believe this perspective provides deeper insight into the computational complexity of SHAP, paving the way for future work.


Our work opens promising directions for future research. First, expanding our computational analysis to other SHAP-related metrics, such as asymmetric SHAP~\citep{frye20} and SAGE~\citep{covert2020understanding}, would be valuable. Additionally, we aim to explore more expressive distribution classes and relaxed assumptions beyond those in Section \ref{sec:tractable} while maintaining tractable SHAP computation. Finally, when exact computation is intractable (Section \ref{sec:intractable}), investigating the approximability of SHAP metrics through approximation and parameterized complexity theory~\citep{downey2012parameterized} is an important direction.

\section*{Acknowledgments}
This work was partially funded by the European Union (ERC,
VeriDeL, 101112713). Views and opinions expressed are however
those of the author(s) only and do not necessarily reflect those of the
European Union or the European Research Council Executive Agency.
Neither the European Union nor the granting authority can be held
responsible for them.


\bibliography{bibliography}

\section*{Checklist}



 \begin{enumerate}

 \item For all models and algorithms presented, check if you include:
 \begin{enumerate}
   \item A clear description of the mathematical setting, assumptions, algorithm, and/or model. [Yes] Mainly in Section 1 and the appendix.
   \item An analysis of the properties and complexity (time, space, sample size) of any algorithm. [Yes] See Section 2, Section 3, Section 4, and the appendix.
   \item (Optional) Anonymized source code, with specification of all dependencies, including external libraries. [Not Applicable]
 \end{enumerate}

 \item For any theoretical claim, check if you include:
 \begin{enumerate}
   \item Statements of the full set of assumptions of all theoretical results. [Yes] Mainly in Section 1 and the appendix.
   \item Complete proofs of all theoretical results. [Yes] Because of space constraints, we provide only a brief summary of the proofs for our claims in the paper, with the complete detailed proofs available in the appendix.
   \item Clear explanations of any assumptions. [Yes]     
 \end{enumerate}

 \item For all figures and tables that present empirical results, check if you include:
 \begin{enumerate}
   \item The code, data, and instructions needed to reproduce the main experimental results (either in the supplemental material or as a URL). [Not Applicable]
   \item All the training details (e.g., data splits, hyperparameters, how they were chosen). [Not Applicable]
         \item A clear definition of the specific measure or statistics and error bars (e.g., with respect to the random seed after running experiments multiple times). [Not Applicable]
         \item A description of the computing infrastructure used. (e.g., type of GPUs, internal cluster, or cloud provider). [Not Applicable]
 \end{enumerate}

 \item If you are using existing assets (e.g., code, data, models) or curating/releasing new assets, check if you include:
 \begin{enumerate}
   \item Citations of the creator If your work uses existing assets. [Not Applicable]
   \item The license information of the assets, if applicable. [Not Applicable]
   \item New assets either in the supplemental material or as a URL, if applicable. [Not Applicable]
   \item Information about consent from data providers/curators. [Not Applicable]
   \item Discussion of sensible content if applicable, e.g., personally identifiable information or offensive content. [Not Applicable]
 \end{enumerate}

 \item If you used crowdsourcing or conducted research with human subjects, check if you include:
 \begin{enumerate}
   \item The full text of instructions given to participants and screenshots. [Not Applicable]
   \item Descriptions of potential participant risks, with links to Institutional Review Board (IRB) approvals if applicable. [Not Applicable]
   \item The estimated hourly wage paid to participants and the total amount spent on participant compensation. [Not Applicable]
 \end{enumerate}

 \end{enumerate}

\setcounter{definition}{0}
\setcounter{section}{0}
\setcounter{proposition}{0}
\setcounter{theorem}{0}
\setcounter{lemma}{0}

\onecolumn
\aistatstitle{
Appendix}

\vspace{-2.5em}

The appendix provides formalizations, supplementary background, and gathers the proofs of several mathematical statements referenced either implicitly or explicitly throughout the main article. It is structured as follows:

\begin{itemize} 
\item Appendix~\ref{app:limitations} discusses the limitations of this work and outlines potential directions for future research.
\item Appendix \ref{app:sec:terminology} provides the technical background and preliminary results that will be referenced throughout the appendix.
\item Appendix \ref{app:shapwa} includes the proofs of intermediate mathematical statements that demonstrate the tractability of computing both local and global interventional SHAP and baseline SHAP variants (Theorem \ref{thm:shapwa}). 
\item 
Appendix~\ref{app:reductiontree} elaborates on the reduction strategy from WAs to decision trees, linear regression models, and tree ensembles employed for regression (Theorem~\ref{cor:reductions}). Moreover, it presents a polynomial-time reduction between the distribution families discussed throughout the main paper, along with corollaries that can be derived from these relationships.

\item Appendix~\ref{app:ISHAPRNN} provides proofs of the intermediary results that validate the reduction strategy from the closest string problem to $\texttt{LOC-I-SHAP}(\texttt{RNN-ReLu}, \texttt{IND})$ (Theorem \ref{thm:relushap}).
\item Appendix~\ref{app:BSHAPSIGMOID} is focused on providing the complete proofs of intermediate results related to the complexity of computing local Baseline SHAP for Sigmoidal Neural Networks, RNN-ReLUs, and tree ensemble classifiers (Theorem \ref{thm:intractable}).
\item Appendix~\ref{app:sec:generalized},  presents the proof for the main result discussed in the section on generalized relations of SHAP variants (Section~\ref{sec:generalized}) of the main article, specifically Proposition~\ref{prop:hardnessrelation}.
\end{itemize}

\section{Limitations and Future Work} \label{app:limitations}
First, while our work presents novel complexity results for the computation of various Shapley value variants, there are many other variants that we did not address, such as asymmetric Shapley values~\citep{frye20}, counterfactual Shapley values~\citep{albini22}, and others. Additionally, in Section~\ref{sec:generalized}, we provided new complexity relationships between SHAP variants, but many more connections remain unexplored. Extensions of our work could also consider different model types and distributional assumptions beyond those discussed here.

A second limitation is that, while we thoroughly examine theoretical strict complexity gaps between polynomial-time and non-polynomial-time (i.e., NP-Hard) computations of SHAP variants, we did not focus on techincal optimizations of the specific poly-time algorithms proposed in this paper. Some of the polynomial factors of these algorithms, though constant, may be improved (see Section~\ref{sec:generalized} for more details), and we believe improving them, as well as exploring potential parallelization techniques for solving them presents an exciting direction for future research.

Finally, we adopt the common convention used in all previous works on the computational complexity of computing Shapley values~\citep{arenas23, vander21, marzouk24a, huangupdates}, assuming a discrete input space to simplify the technical aspects of the proofs. However, we emphasize that many of our findings also extend to continuous domains. Specifically, for both decision trees and tree ensembles, the complexity results remain consistent across both discrete and continuous domains, as the tractability of the underlying models is unaffected. On the other hand, while this assumption does not generally apply to linear models and neural networks, we stress that the computational \emph{hardness} results we present for these models continue to hold in continuous settings. The same, however, cannot be said for membership proofs. Extending the membership proofs for linear models and neural networks, as presented in this work, to continuous domains and exploring other computational complexity frameworks for these models offers a promising direction for future research.

\section{Extended Related Work}

In this section, we include a more elaborate discussion of related work and key complexity results examined in prior research. 

\textbf{SHAP values.} Building on the initial SHAP framework introduced by~\citep{lundberg2017} for deriving explanations of ML models, various subsequent works extensively investigated the application of SHAP across various contexts within the XAI literature. Many efforts have concentrated on developing numerous other SHAP variants beyond Conditional SHAP \citep{sundararajan20b, janzing20a, heskes2020causal}, aligning SHAP with the distribution manifold \citep{frye2020shapley, taufiq2023manifold}, and enhancing the approximation of its calculation \citep{fumagalli2024shap, sundararajan2020shapley, burgess2021approximating, kwon2021efficient}. Additionally, the literature has explored various limitations of SHAP in different contexts \citep{fryer2021shapley, huang2024failings, kumar2020problems, marques2024explainability}.


\textbf{The Complexity of SHAP.} Notably, \citep{vander21} investigates Conditional SHAP, presenting a range of tractability and intractability results for different ML models, with a key insight being that computing Conditional SHAP under independent distributions is as complex as computing the conditional expectation. Later, \cite{arenas23} generalizes these findings, showing that the tractability results for Conditional SHAP align with the class of Decomposable Deterministic Boolean Circuits and establishing that both the Decomposability and Determinism properties are necessary for tractability. More recently, \cite{marzouk24a} moves beyond the independent distribution assumption, extending the analysis to \emph{Markovian} distributions, which are significantly less expressive than the HMM-modeled distributions considered in our work. Additionally, \cite{huangupdates} introduces distinctions between regression and classification in tree ensembles for Conditional SHAP and extends previous results to diverse input and output settings. Other relevant, but less direct extensions of these complexity results are extensions to the domain of databse tuples~\citep{deutch2022computing, livshits2021shapley, bertossi2023shapley, kara2024shapley, karmakar2024expected}, as well as obtaining the complexity of other interaraction values other than Shapley values~\citep{abramovich2024banzhaf, barcelo2025computation}. Our work provides a novel analysis of all three major SHAP variants (baseline, interventional, and conditional), broadening distributional assumptions and offering new insights into local and global SHAP values. We frame SHAP computation as a multifaceted process shaped by \begin{inparaenum}[(i)] \item model type; \item SHAP variant; \item distributional assumptions; and \item the local-global distinction.\end{inparaenum}

\textbf{Formal XAI.} More broadly, our work falls within the subdomain of interest known as \emph{formal XAI}~\citep{marques2023logic}, which aims to generate explanations for ML models with formal guarantees~\citep{ignatiev2020towards, bassan2023towards, darwiche2020reasons, darwiche2022computation, ignatiev2019abduction, audemard2022preferred}. These explanations are often derived using formal reasoning tools, such as SMT solvers~\citep{barrett2018satisfiability} (e.g., for explaining tree ensembles~\citep{audemard2022trading}) or neural network verifiers~\citep{katz2017reluplex, wu2024marabou, wang2021beta} (e.g., for explaining neural networks~\citep{izza2024distance, bassan2023formally}). A key focus within formal XAI is analyzing the computational complexity of obtaining such explanations~\citep{barcelo2020model, waldchen2021computational, cooper2023tractability, bassanlocal, blanc2021provably, amir2024hard, bass2025exp, adolfi2024computational, barcelo2025explaining, calautti2025complexity, ordyniak2023parameterized}.


\section{Complexity Classes, Models, and Distributions} \label{app:sec:terminology}

In this section, we introduce the general preliminary notations used throughout our paper, including those related to model types, distributional assumptions, and complexity classes.


\subsection{Computational Complexity Classes}

In this work, we assume that readers are familiar with standard complexity classes, including polynomial time (PTIME) and non-deterministic polynomial time (NP and coNP). We also discuss the complexity class NTIME, which refers to the set of decision problems solvable by a non-deterministic Turing machine within a specified time bound, such as polynomial time. Additionally, we cover the class $\#$P, which counts the number of accepting paths of a non-deterministic Turing machine and can be seen as the ``counting'' counterpart of NP. It is known that PTIME is contained within NP, coNP, NTIME, and $\#$P, but it is widely believed that these containments are strict, i.e., PTIME $\subsetneq$ NP, coNP, NTIME, $\#$P~\citep{arora2009computational}. We use the standard notation $L_1 \preceq_{P} L_2$ to indicate that a polynomial-time reduction exists from the computational problem $L_1$ to $L_2$.

\subsection{Models}

In this subsection, we describe the families of model types used throughout the paper, including decision trees (\texttt{DT}), tree ensembles for classification ($\texttt{ENS-DT}_{\texttt{C}}$) and regression ($\texttt{ENS-DT}_{\texttt{R}}$), linear regression models ($\texttt{LIN}_{\texttt{R}}$), as well as neural networks (\texttt{NN-SIGMOID} and \texttt{RNN-ReLU}).

\textbf{Decision trees (\texttt{DT}).} We define a \emph{decision tree} (DT) as a directed acyclic graph representing a graphical model for a discrete function $f: \mathcal{X} \to \mathbb{R}$ in regression tasks, or $f: \mathcal{X} \to [c]$ for classification tasks (where $c \in \mathbb{N}$ is the number of classes). We assume that each input $\x_i$ can take a bounded number of discrete values, limited by some $k$. This graph encodes the function as follows:

\begin{enumerate}
	\item Each internal node $v$ is associated with a unique binary input feature from the set $\{1,\ldots,n\}$;
	\item Every internal node $v$ has up to $k$ outgoing edges, corresponding to the values $[k]$ which are assigned to $v$;
	\item In the DT, each variable is encountered no more than once on any given path $\alpha$;
 	\item Each leaf node is labeled as one of the classes in $[c]$ (for classification tasks) or some $c\in\mathbb{R}$ for regression tasks.
\end{enumerate}

Thus, assigning a value to the inputs $\x\in\mathcal{X}$ uniquely determines a specific path $\alpha$ from the root to a leaf in the DT. The function $f(\x)$ is assigned either some $i\in [c]$ or some $i\in \mathbb{R}$ (depending on classification or regression tasks). The size of the DT, denoted as $\vert f \vert$, is measured by the total number of edges in its graph. To allow flexibility of the modeling of the DT, we permit different varying orderings of the input variables $\{1,\ldots,n \}$ across any two distinct paths, $\alpha$ and $\alpha'$. This ensures that no two nodes along any single path $\alpha$ share the same label. We note that in some of the proofs provided in this work, we simplify them by assuming that $k:=2$ or in other words that each feature is defined over binary feature assignments rather than discrete ones. However, this assumption is only for the sake of simplifying the proofs, and our proofs hold also for a general $k$ as well.

\textbf{Decision tree ensembles ($\texttt{ENS-DT}_{\texttt{R}}$, $\texttt{ENS-DT}_{\texttt{C}}$).} There are various well-known architectures for tree ensembles. While these models typically differ in their training processes, our work focuses on post-hoc interpretation, so we emphasize the distinctions in the inference phase rather than the training phase. Specifically, our analysis targets ensemble families that use weighted-voting methods during inference. This includes tree ensembles based on boosting, such as XGBoost. Additionally, in cases where all weights are equal, our formalization also covers majority-voting tree ensembles, such as those used in common bagging techniques like Random Forests. Our approach applies to both tree ensembles for regression tasks (denoted as $\texttt{ENS-DT}_{\texttt{R}}$) and classification tasks (denoted as $\texttt{ENS-DT}_{\texttt{C}}$). We make this distinction explicitly, as we will demonstrate that the complexity results differ for classification and regression ensembles.




Conceptually, an ensemble tree model $\mathcal{T} \in  \texttt{ENS-DT}_{\texttt{R}}$ is constructed as a linear combination of DTs. Formally, $\mathcal{T}$ is parametrized by the tuple $\{T_{i}\}_{i \in [m]}, \{ w_{i}\}_{i \in [m]}$ where  $\{T_{i}\}_{i \in [m]}$ is a collection of DTs (referred to as an ensemble), and $\{w_{i}\}_{i \in [m]}$ is a set of real numbers. The model $\mathcal{T}$ is used for regression tasks and the function that it computes is given as:

\begin{equation}
\label{tree_ensemble_formulation_appendix}
    f_{\mathcal{T}}(x_{1}, \ldots, x_{n}) := \sum\limits_{i=1}^{m} w_{i} \cdot f_{T_{i}}(x_{1}, \ldots, x_{n})
\end{equation}

Aside from regression tasks, ensemble trees are also commonly used for classification tasks. In this context, each decision tree $\mathcal{T}$ is a classification tree (as defined earlier for DTs) that assigns a class label $i \in [c]$. By using the formulation from Equation~\ref{tree_ensemble_formulation_appendix}, we can limit the following sum to only those decision trees relevant to class $i$, and obtain a corresponding weight $f^i_{\mathcal{T}}$. The class assigned by the ensemble will be the one with the highest value of $f^i_{\mathcal{T}}$. It’s worth noting that if all weights are taken to be equal, this mirrors majority voting classification, as seen in random forest classifiers. Furthermore, since our proofs in this context only address \emph{hardness}, for simplicity, we assume the number of classes $c = 2$, meaning the random forest classifier is binary. Clearly, the hardness results for this case also apply to the more general multiclass scenario. In this specific case, the given formalization can be restated as:


 \begin{equation}
\label{tree_ensemble_formulation_appendix}
    f_{\mathcal{T}}(x_{1}, \ldots, x_{n}) := \text{step}(\sum\limits_{i=1}^{m} w_{i} \cdot f_{T_{i}}(x_{1}, \ldots, x_{n}))
\end{equation}

where the step function is defined as: $\text{step}(\x) = 1 \iff \x \geq 0$. We refer to the class of classification tree ensembles as $\texttt{ENS-DT}_{\texttt{C}}$.


\paragraph{Neural Networks (\texttt{NN-SIGMOID}, \texttt{RNN-ReLU}).} 
Here, we define the general neural network architecture used throughout this work, followed by a specific definition of the (non-recurrent) sigmoidal neural network (\texttt{NN-SIGMOID}) referenced in the paper. We note that the definition of recurrent neural networks with ReLU activations (\texttt{RNN-ReLU}) mentioned in the paper was provided in the main text, and hence it is not explicitly redefined here. We denote a neural network $f$ with $t-1$ hidden layers ($f^{(j)}$, where $j$ ranges from $1$ to $t-1$), and a single output layer ($f^{(t)}$). We observe that we can assume a single output layer since $f^{(t)}$ will yield a value in $\mathbb{R}$ for regression tasks, or it will be applied for binary classification (we will later justify why assuming binary classification is sufficient, and the correctness results will extend to multi-class classification as well). The layers of $f$ are defined recursively — each layer $f^{(j)}$ is computed by applying the activation function $\sigma^{(j)}$ to the linear combination of the outputs from the previous layer $f^{(j-1)}$, the corresponding weight matrix $W^{(j)}$, and the bias vector $b^{(j)}$. This is represented as:

\begin{equation}
    f^{(j)} := \sigma^{(j)}(f^{(j-1)}W^{(j)} + b^{(j)})
\end{equation}

Where $f^{(j)}$ is computed for each $j$ in $\{1,\ldots,t\}$. The neural network includes $t$ weight matrices ($W^{(1)},\ldots,W^{(t)}$), $t$ bias vectors ($b^{(1)},\ldots,b^{(t)}$), and $t$ activation functions ($\sigma^{(1)},\ldots,\sigma^{(t)}$). The function $f$ is defined to output $f := f^{(t)}$. The initial input layer $f^{(0)}$ serves as the model's input. The dimensions of the biases and weight matrices are specified by the sequence of positive integers ${d_{0}, \ldots, d_{t}}$. We specifically focus on weights and biases that are rational numbers, represented as $W^{(j)}\in \mathbb{Q}^{d_{j-1}\times d_{j}}$ and $b^{(j)}\in \mathbb{Q}^{d_{j}}$, which are parameters optimized during training. Clearly, it holds that $d_0=n$. For regression tasks, the output of $f$ (i.e., $d_t$) will be in $\mathbb{R}$. For classification tasks, we assume multiple outputs, where typically a softmax function is applied to choose the class $i\in[c]$ with the highest value. However, since the proofs for both \texttt{NN-SIGMOID} and \texttt{RNN-ReLU} are \emph{hardness} proofs, we assume for simplicity that in classification cases, we use a binary classifier, i.e., $c=2$. Thus, hardness results will clearly apply to the multi-class setting as well. Therefore, for binary classification, it follows that $d_{0} = n$ and $d_{t} = 1$.

The activation functions $\sigma^{(i)}$ that we focus on in this work are either the ReLU activation function, defined as $\text{ReLU}(x) := \max(0, x)$ (used for the RNN-ReLU model), or the sigmoid activation function, defined as $\text{Sigmoid}:=\frac{1}{1+e^{-x}}$. In the case of our binary classification assumption, we assume that the last output layer can either use a sigmoid function or (without loss of generality) a step function for the final layer activation. Here, we denote $\text{step}(x) = 1 \iff x \geq 0$.

\paragraph{Linear Regression Models ($\texttt{\text{LIN}}_{\texttt{\text{R}}}$).} 
A linear regression model corresponds to a single-layer neural network in the regression context, as defined in the formalization above, where $t=1$. The linear model is described by the function $f(\x) := (\mathbf{w} \cdot \x) + b$, with $b \in \mathbb{Q}$ and $\mathbf{W} \in \mathbb{Q}^{n \times d_1}$. Furthermore, we introduce an alternative renormalization of linear regression models, which will be helpful for the technical developments in Section~\ref{app:reductiontree}. Specifically, a linear regression model over the finite set $\mathbb{D} := [m_{1}] \times \ldots [m_{n}]$, where $\{m_{i}\}_{i \in [n]}$ represents a set of integers, can be parametrized as $\langle \{w_{i,d}\}_{i \in [n], d \in \mathbb{D}_{i}}, b\rangle$, where $\{w_{i,d}\}_{i \in [n], d \in m_{i}}$ is a collection of rational numbers and $b \in \mathbb{Q}$. The function computed by $f$ is given as:


$$f(x_{1}, \ldots , x_{n}) = \sum\limits_{i=1}^{n} \sum\limits_{d \in \mathbb{D}_{i}} w_{i,d} \cdot I(x_{i}=d) + b$$


\subsection{Distributions}
In this subsection, we will formally define various distributions relevant to this work, which have been referenced throughout the main paper or in the appendix. We note that the definition of hidden Markov models (\texttt{HMM}), the most significant class of distributions examined in this study, is not included here as it was already provided in the main text.

\paragraph{Independent Distributions (\texttt{IND}).} The family \texttt{IND} is the most elementary family of distributions based on the assumption of probabilistic independence of all random variables (RVs) involved in the model. Formally, given some set of discrete values $[k]$, we can describe a probability function $p:[n]\times [k]\to [0,1]$. For example, $p(1,2)=\frac{1}{2}$, implies that the probability of feature $i=1$, to be set to the value $k=2$ is $\frac{1}{2}$. Then we can define $\mathcal{D}_p$ as an independent distribution over $\mathcal{X}$ iff:

\begin{equation}
    \label{eq:explanation}
    \mathcal{D}_p(\x):=\Big({\displaystyle \prod_{i\in [n], \x_i=j} p(i, j)}\Big)
\end{equation}

It is evident that the uniform distribution is a specific instance of $\mathcal{D}_p$, obtained by setting $p(i,j):=\frac{1}{|k|}$ for every $i \in [n]$ and $j \in [k]$.


\paragraph{Empirical Distributions (\texttt{EMP}).} The empirical distribution provides a practical way of estimating probabilities from a finite dataset. Given a set of $M$ samples, each represented as a vector in $\{0,1\}^{N}$, the empirical distribution assigns a probability to each possible vector $x$ in the space. This probability is simply the proportion of samples in the dataset that are equal to $x$. Formally, for a dataset $\mathcal{D} = \{x_{1}, \ldots, x_{M}\}$, the empirical distribution $P_{\mathcal{D}}(x)$ is defined as the frequency of occurrences of the vector $x$ in the dataset, normalized by the total number of samples, i.e.:
$$P_{\mathcal{D}}(x) = \frac{1}{|\mathcal{D}|} \sum\limits_{i=1}^{|\mathcal{D}|}  I(x_{i} = x)$$

\paragraph{Naive Bayes Model (\texttt{NB}).} A naive Bayes model is a latent probabilistic model that involves $n+1$ random variables (RVs), denoted as $(X_{1}, \ldots , X_{n}, Y)$, where $\{X_{i}\}_{i \in [n]}$ represent the $n$ observed RVs, and $Y$ is an unobserved (latent) RV. The key probabilistic assumption of naive Bayes models is that the observed RVs are conditionally independent given the value of the latent variable $Y$. More formally, a model $M \in \ar{\texttt{NB}}$ over $n$ RVs is specified by the parameters $\langle\pi, \{P_{i}\}_{i \in [n]}\rangle$, where:

\begin{itemize}
    \item $\pi$ is a probability distribution over the domain value of the latent variable $Y$ ($dom(Y)$),
    \item For $i \in [n]$, $P_{i} \in \mathbb{R}^{n \times \text{dom}(Y)}$ is a stochastic matrix. 
\end{itemize}
The marginal probability distribution computed by $M$ is given as:
$$P_{M}(x_{1}, \ldots x_{n}, y) = \pi(y) \prod\limits_{i=1}^{n} P_{i}[x_{i}, y]$$

\paragraph{Markovian Distributions (\texttt{MARKOV}).} A (stationary) Markovian distribution $M \in \texttt{MARKOV}$ over an alphabet $\Sigma$ is represented by the tuple $\alpha, T$ where $\pi$ is a probability distribution over $\Sigma$ and $T$ is a stochastic matrix in $\mathbb{R}^{|\Sigma| \times |\Sigma|}$ \footnote{A matrix $A \in \mathbb{R}^{n \times m}$ is said to be stochastic if each row vector corresponds to a probability distribution over $[m]$.}. A Markovian model $M$ computes a probability distribution over $\Sigma^{\infty}$. The probability of generating a given sequence $w \in \Sigma^{*}$ as a prefix by a Markovian model $M = \pi, T$ is given as:
$$P_{M}^{(|w|)}(w) = \pi[w_{1}] \cdot \prod\limits_{i=2}{|w|} T[w_{i-1}, w_{i}]$$
where for a given integer  $n$, $P_{M}^{(n)}$ designates the probability distribution over $\Sigma^{n}$ interpreted as the probability of generating a prefix of length $n$. 

Analogous to HMMs, one can define a family of models representing the non-sequential counterpart of Markovian models, which we'll refer to as $\ar{\texttt{\text{MARKOV}}}$. A model $\ar{\texttt{M}}  \in \ar{\texttt{\text{MARKOV}}}$ defines a probability distribution over $\Sigma^{n}$ for $n \geq 1$, and parameterized by the tuple $\pi, \alpha, \{T_{i}\}_{i \in [n]}$, where:
\begin{itemize}
    \item $\pi$ is a permutation from $[n]$ to $[n]$.
    \item $\alpha$ defines a probability distribution over $[n]$ (also called the initial state vector),
    \item For each $i \in [n]$, $T_{i}$ is a stochastic matrix over $\mathbb{R}^{|\Sigma| \times |\Sigma|}$
\end{itemize}

The procedure of generating the tuple $(x_{1}, \ldots, x_{n})$ (where $x_{i} \in \Sigma$) by $\ar{\texttt{M}}$ can be described recursively as follows:
 \begin{enumerate}
     \item \textbf{Generation of the first element of the sequence.} Generate $x_{\pi(1)}$ with probability $\alpha[x_{\pi(1)}]$.
     \item \textbf{Generation of the (i+1)-th element.} For $i \in [n-1]$, the probability of generating the element $x_{\pi(i+1)}$ given that $x_{\pi(i)}$ is generated is equal to $T[x_{\pi(i)}, x_{\pi(i+1)}]$
 \end{enumerate}

$$
      $$
\section{The Tractability of computing SHAP for the class of WAs: Proofs of Intermerdiary Results} \label{app:shapwa}

In this section of the appendix, we provide detailed proofs of intermediary mathematical statements to prove the tractability of computing different SHAP variants on the class of WAs (Theorem \ref{thm:shapwa}). Specifically, we shall provide proofs of three intermediary results: 
\begin{enumerate}
    \item Proposition \ref{prop:efficentoperations} that states the computational efficiency of implementing the projection operation. 
    \item Lemma~\ref{lemma:shapasoperations}, which demonstrates how the computation of both local and global Interventional and Basleline SHAP for the family of WAs under distributions modeled by HMMs can be reduced to performing operations over N-Alphabet WAs.
    \item Proposition \ref{prop:nletterwaconstruction}, which asserts that the construction of N-Alphabet WAs can be achieved in polynomial time, thus enabling the polynomial-time algorithmic construction for both $\texttt{LOC-I-SHAP}(\texttt{WA}, \texttt{HMM})$ and $\texttt{GLO-I-SHAP}(\texttt{WA}, \texttt{HMM})$.
\end{enumerate}

\subsection{Terminology and Technical Background}
The proof of Theorem \ref{thm:shapwa} will rely on certain technical tools that were not introduced in the main paper. This initial section is devoted to providing the technical background upon which the proofs of various results presented in the rest of this section rely.

\paragraph{The Kronecker product.} \label{app:sec:ter} 
The Kronecker product between $A \in \mathbb{R}^{n \times m}$ and $B \in \mathbb{R}^{k \times l}$, denoted $A \otimes B$, is a matrix in $\mathbb{R}^{(n \cdot k) \times (m \cdot l)}$  
constructed as follows 
 $$A \otimes B = \begin{bmatrix}
     a_{1,1} \cdot B & a_{1,2} \cdot B & \dots & a_{1,m} \cdot B] \\ 
     a_{2,1} \cdot B & a_{2,2} \cdot B & \dots &  a_{2,m} \cdot B] \\
     \vdots & \vdots & \vdots & \vdots \\ 
     a_{n,1} \cdot B & a_{n,2} \cdot B & \dots & a_{n,m} \cdot B
 \end{bmatrix}$$
 where, for $(i,j) \in [n] \times [m]$, $a_{i,j}$ corresponds to the element in the $i$-th row and the $j$-th column of $A$. A property of the Kronecker product of matrices that will be utilized in several proofs in the appendix is the \emph{mixed-product} property:

 \begin{property} \label{app:eq:mixedproduct}
  Let $A, B, C, D$ be four matrices with compatible dimensions. We have that:
  $$(A \cdot B) \otimes (C \cdot D) = (A \otimes C) \cdot (B \otimes D)$$
 \end{property}

\paragraph{N-Alphabet Deterministic Finite Automata (N-Alphabet DFAs).} In subsection \ref{prop:nletterwaconstruction}, we shall employ a sub-class of N-Alphabet WAs more adapted to model binary functions (i.e. functions whose output domain is $\{0,1\}$)in the proof of Proposition \ref{prop:nletterwaconstruction}. The class of N-Alphabet DFAs can be seen as a generalization of the classical family of Deterministic Finite Automata for the multi-alphabet case. N-Alphabet DFAs are formally defined as follows:
\begin{definition} \label{app:def:naldfa}
A N-Alphabet DFA $A$ is represented by a tuple $\langle Q, q_{init}, \delta, F \rangle$ where:
\begin{itemize}
    \item $Q$ is a finite set corresponding to the state space,
    \item $q_{init} \in Q$ is called the initial state.
    \item $\delta$, called the transition function, is a partial map from $Q \times \Sigma_{1} \times \ldots \times \Sigma_{N}$ to $Q$
    \item $F \subseteq Q$ is called the final state set
\end{itemize}
\end{definition}

Figure \ref{app:fig:nalphabetdfa} illustrates the graphical representation of some N-Alphabet DFAs. Analgous to N-Alphabet WAs, we shall use the terminology \textit{DFA}, instead of 1-Alphabet DFA for $N=1$.

To show how N-Alphabet DFAs compute (binary) functions, we need to introduce the notion of a \textit{path}. For a N-Alphabet DFA $A = \langle Q, q_{init}, \delta, F \rangle$ over $\Sigma_{1} \times \ldots \times \Sigma_{N}$, a valid path in $A$ is a sequence $P = (q_{1} ,\sigma_{1}^{(1)}, \ldots \sigma_{1}^{(N)}) \ldots (q_{L} ,\sigma_{1}^{(L)}, \ldots \sigma_{L}^{(N)}) q_{L+1}$ in $ (Q \times \Sigma_{1} \times \ldots \times  \Sigma_{N})^{*} \times Q$ such that for any $i \in [L]$, $\delta(q_{i} ,(\sigma_{i}^{(1)}, \ldots \sigma_{i}^{(N)})) = q_{i+1}$. Given this definition of a valid path, the N-Alphabet $A$ for a given tuple of sequences $(w^{(1)}, \ldots , w^{(N)}) \in \Sigma_{1}^{*} \times \ldots \times \Sigma_{N}^{*}$ such that $|w^{(1)}| = \ldots = |w^{(N)}|= L $ if and only if there exists a valid path  $(q_{1}, w_{1}^{(1)}, \ldots w_{1}^{(N)}) (q_{2}, w_{2}^{(1)}, \ldots , w_{2}^{(N)}) \ldots (q_{L}, w_{L}^{(1)}, \ldots, w_{L}^{(N)}) q_{L+1}$ such that $q_{1} = 1$ and $q_{L+1} \in F$. For instance, the sequence $abab$ of the 1-Alphabet DFA in Figure \ref{app:fig:nalphabetdfa} is labeled by $1$. Indeed, the valid path $(q_{init}, a) (q_{init},b)(q_{1},a)(q_{2},b)q_{1}$ satisfies these conditions.

\begin{figure}
  \begin{minipage}[t]{0.3\linewidth}
        \centering
        
        \begin{tikzpicture}[shorten >=1pt, node distance=2cm, on grid, auto]
            \node[state]   (q0)                {$q_{init}$};
            \node[state, accepting]            (q1) [right=of q0]  {$q_{1}$};
            \node[state] (q2) [right=of q1]  {$q_{2}$};
            
            \path[->]
                (q0) edge [loop above]   node {$a $}
                (q0)
                     edge []    node {$b $} (q1)
                 (q1) edge [loop above]   node {$c$}
                (q1) edge [bend left]    node {$a$} 
                (q2)
                     
                (q2) edge [loop above]   node {$a $} (q2)
                     edge [bend left]    node {$b$} (q1);
        \end{tikzpicture} \\
        \textbf{(a) A 1-Alphabet DFA: $\Sigma_{1} = \{a,b,c\}$}
    \end{minipage}%
    \hfill
    \begin{minipage}[t]{0.3\linewidth}
        \centering 
        \begin{tikzpicture}[shorten >=1pt, node distance=2cm, on grid, auto]
            \node[state]   (q0)                {$q_{init}$};
            \node[state,accepting]            (q1) [right=of q0]  {$q_{1}$};
            \node[state] (q2) [right=of q1]  {$q_{2}$};
            
            \path[->]
                (q0) edge [loop above]   node {$(a,0)$} (q0)
                     edge []    node {$(b,1)$} (q1)
                (q1) edge [bend left]    node {$(a,0)$} (q2)
                     edge [loop above]   node {$(b,1)$} (q1)
                (q2) edge [loop above]   node {$(c,1)$} (q2)
                     edge [bend left]    node {$(b,1)$} (q1);
        \end{tikzpicture} \\
         \textbf{(b) A 2-Alphabet DFA: $\Sigma_{1} = \{a,b,c\},~\Sigma_{2} = \{0,1\}$}
    \end{minipage}
    \hfill
    \begin{minipage}[t]{0.3\linewidth}
        \centering 
        \begin{tikzpicture}[shorten >=1pt, node distance=2cm, on grid, auto]
            \node[state]   (q0)                {$q_{init}$};
            \node[state, accepting]            (q1) [right=of q0]  {$q_{1}$};
            \node[state] (q2) [right=of q1]  {$q_{2}$};
            
            \path[->]
                (q0) edge [loop above]   node {$(a,0,x)$} (q0)
                     edge [bend left]    node {$(b,1,y)$} (q1)
                (q0)
                     edge [bend right, below]    node {$(b,1,x)$} (q1)
                (q1) edge [bend left]    node {$(a,0,x) $} (q2)
                     edge [loop above]   node {$(b,1,y)$} (q1)
                (q2) edge [loop above]   node {$(c,1,x)$} (q2)
                     edge [bend left]    node {$(b,1,x) $} (q1);
        \end{tikzpicture} \\
         \textbf{(b) A 3-Alphabet DFA : $\Sigma_{1} = \{a,b,c\},~\Sigma_{2} = \{0,1\},~\Sigma_{2} = \{x,y\}$}
    \end{minipage}
    \caption{A graphical representation of N-Alphabet DFAs. Nodes corresponding to final states are represented by double circles.}
  \label{app:fig:nalphabetdfa}
\end{figure}
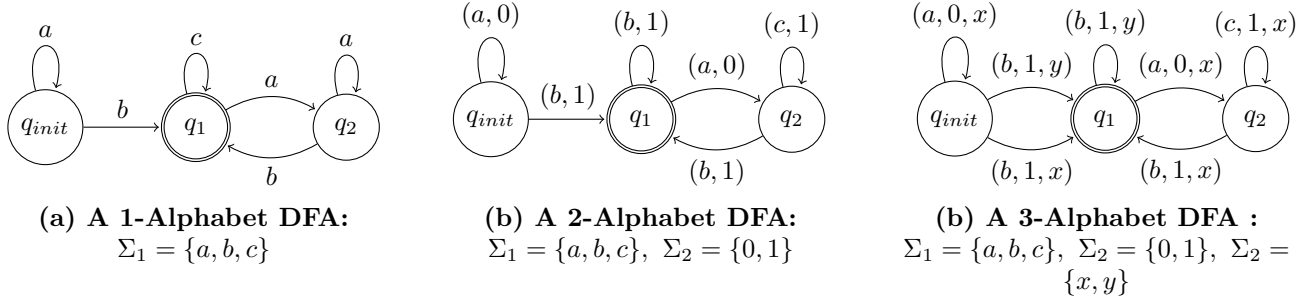
 
 \paragraph{Linear Algebra operations over N-Alphabet WAs.} 
 The development of polynomial-time algorithms for computing various SHAP variants for the class of WAs in section \ref{sec:tractable} of the main paper is based on two operations: the projection and the Kronecker product operations. In the following section of the appendix, we will provide a proof of the tractability of their construction.
 

In addition to these two operators, linear algebra operations over N-Alphabet WAs have also been implicitly utilized in the construction. The closure of 1-Letter WAs under linear algebra operations is a well-established result in the WA literature \citep{droste10}. For completeness, we offer a brief discussion below on how linear algebra operations can be extended to handle multiple alphabets, including their construction and the associated complexity results:

     \begin{itemize}
         \item \emph{The addition operation:} Given two N-Alpphabet WAs $T = \langle\alpha, \{A_{\sigma_{1},\ldots, \sigma_{N}}\}_{(\sigma_{1}, \ldots, \sigma_{N}) \in \Sigma_{1} \times \ldots \times \Sigma_{N}}, \beta\rangle$ and $T' = \langle\alpha', \{A'_{\sigma_{1},\ldots, \sigma_{N}}\}_{(\sigma_{1}, \ldots, \sigma_{N}) \in \Sigma_{1} \times \ldots \times \Sigma_{N}}, \beta'\rangle$  over $\Sigma_{1} \times \ldots \times \Sigma_{N}$, the N-Alphabet WA, denoted $T + T'$, that computes the function:
         $$f_{T + T'}(w^{(1)}, \ldots, w^{(N)}) = f_{T}(w^{(1)}, \ldots, w^{(N)}) + f_{T'}(w^{(1)}, \ldots, w^{(N)})$$
         is parametrized as follows:
          $$\langle\begin{pmatrix}
              \alpha \\ \alpha'
          \end{pmatrix} , \{\begin{pmatrix}
               A_{\sigma_{1}, \ldots , \sigma_{N}} & \mathbf{O} \\
               \mathbf{O} & A'_{\sigma_{1}, \ldots, \sigma_{N}}
          \end{pmatrix} \}_{(\sigma_{1}, \ldots, \sigma_{N}) \in \Sigma_{1} \times \ldots \times \Sigma_{N}}, \begin{pmatrix}
              \beta \\ \beta' 
          \end{pmatrix} \rangle$$
          The running time of the addition operation is $O(|\Sigma|^{N} \cdot (\texttt{size}(T) + \texttt{size}(T') ))$. The size of the resulting N-Alphabet WA is equal to $O(\texttt{size}(T) + \texttt{size}(T'))$.
          \item \emph{Multiplication by a scalar.} Let $T = \langle\alpha, \{A_{\sigma_{1},\ldots,\sigma_{N}}\}_{(\sigma_{1}, \ldots, \sigma_{N}) \in \Sigma_{1} \times \ldots \times \Sigma_{N}}, \beta\rangle$ be an N-Alphabet WA over $\Sigma_{1} \times \ldots \times \Sigma_{N}$, and a real number $C> 0$, the N-Alphabet WA, denoted $C \cdot T$ that computes the function $f_{C \cdot T}(w^{(1)}, \ldots, w^{(N)}) = C \cdot f_{T}(w^{(1)}, \ldots, w^{(N)})$ is parametrized as: $\langle C \cdot \alpha, \{ A_{\sigma_{1},  \ldots, \sigma_{N}}\}_{(\sigma_{1}, \ldots, \sigma_{N}) \in \Sigma_{1} \times \ldots \times \Sigma_{N}}, \beta \rangle$. It is easy to see that the construction of the N-Alphabet WA $C \cdot T$ runs in $O(1)$ time, and has size equal to the size of $T$. 
     \end{itemize}
  \begin{table}[ht]
  \footnotesize
  	\setlength{\tabcolsep}{0.8em}
    \centering
            \caption{Operations on N-Alphabet WAs, along with their time complexity and output size. The ``In 1'' and ``In 2'' (respectively ``Out'') columns indicate the number of alphabets in the input N-Alphabet WAs for each operation. The ``Time'' column specifies the time complexity of executing the operation, and the ``Output size'' column denotes the size of the resulting N-Alphabet WA after the operation is applied. By convention, a value of $0$ in the ``Output'' and ``Output size'' columns indicates a scalar result.
            }
    \begin{tabular}{|c|c|c|c|c|c|}
        \hline
         & \textbf{In 1} & \textbf{In 2} & \textbf{Out} & \textbf{Time} & \textbf{Output size} \\ \hline
        Addition ($+$) & $N$ & $N$ & $N$ & $O(\max\limits_{i \in [N]} |\Sigma_{i}|^{N} \cdot (\texttt{size}(\text{in}_{1}) + \texttt{size}(\text{in}_{2}))$ & $O(\texttt{size}(\text{in}_{1}) + \texttt{size}(\text{in}_{2}))$ \\ \hline
        Scalar Multiplication & $N$ & $0$ & $N$ & $O(1)$ & $O(\texttt{size}(\text{in}_{1})$  \\ \hline
        $\Pi_{0}$ & $1$ & - & $0$ & $O(|\Sigma_{1}| \cdot \texttt{size}(\text{in}_{1})^{2} \cdot n)$ & $0$ \\ \hline
        $\Pi_{1}$ & $1$ & $1$ & $0$ & $O(|\Sigma_{1}| \cdot (| \texttt{size}(\text{in}_{1}) \cdot \texttt{size}(\text{in}_{2}))^{2} \cdot n)$ \footnote{We assume that the operations $\Pi_{0}$ and $\Pi_{1}$ are applied to WAs whose support is equal to $\Sigma^{n}$. This assumption is sufficient for the purpose of our work. In this case, the running time complexity of $\Pi_{0}$ and $\Pi_{1}$ depends on the parameter $n$.} & $0$ \\ \hline
        $\Pi_{i}$ ($i \geq 2$) & $1$ & $N$ & $N-1$ &$O(\max\limits_{i \in [N]} |\Sigma_{i}|^{N} \cdot \texttt{size}(\text{in}_{1}) \cdot \texttt{size}(\text{in}_{2}))$ &  $O(\texttt{size}(\text{in}_{1}) \cdot \texttt{size}(\text{in}_{2}))$ \\ \hline
        $\otimes$ & $N$ & $N$ & $N$ &$O(\max\limits_{i \in [N]} |\Sigma_{i}|^{N} \cdot \texttt{size}(\text{in}_{1}) \cdot \texttt{size}(\text{in}_{2}))$ &  $O(\texttt{size}(\text{in}_{1}) \cdot \texttt{size}(\text{in}_{2}))$ \\ \hline
    \end{tabular}
    \label{app:fig:operationswas}
\end{table}

     A summary of the running time complexity and the size of outputted WAs by all operations over $N$-Alphabet encountered in this work can be found in Table \ref{app:fig:operationswas}. 

\subsection{Proof of proposition \ref{prop:efficentoperations}}
Recall the statement of Proposition \ref{prop:efficentoperations}:

\begin{unumberedproposition}
       Assume that $N = O(1)$. Then, the projection and the Kronecker product operations between $N$-Alphabet WAs can be computed in polynomial time.
\end{unumberedproposition}
 The following result provides an implicit construction of these two operators which implicitly induces the result of Proposition \ref{prop:efficentoperations}: 

\begin{proposition}
Let $N$ be an integer, and $\{\Sigma_{i}\}_{i \in [N]}$, a collection of finite alphabets. We have: 
\begin{enumerate}
 \item The projection operation: Fix an integer $i \in [N]$. Let $A = \langle \alpha, \{A_{\sigma}\}_{\sigma \in \Sigma}, \beta \rangle$ be a WA over $\Sigma_{i}$,  $T = (\alpha', \{A'_{\sigma_{1},\ldots, \sigma_{N}}\}_{(\sigma_{1}, \ldots, \sigma_{N}) \in \Sigma_{1} \times \ldots \times \Sigma_{N}}, \beta')$ be an N-Alphabet WA over $\Sigma_{1} \times \ldots \times \Sigma_{N}$, and $A$ be a WA over $\Sigma_{i}$. The projection of $A$ over $T$ at index $i$, denoted $\Pi_{i}(A,T)$, is parametrized as:
     \begin{align*}
     \Pi_{i}(A,T) :=&  \langle\Sigma_{1} \times \ldots \times \Sigma_{i-1} \times \Sigma_{i} \times \ldots \times \Sigma_{N}, \alpha \otimes \alpha' , \\
     & \{ \sum\limits_{\sigma_{i} \in \Sigma_{i}} A_{\sigma_{i}} \otimes A'_{\sigma_{1}, \ldots, \sigma_{i-1}, \sigma_{i+1}, \ldots, \sigma_{N}} \}_{(\sigma_{1}, \ldots, \sigma_{i-1},\sigma_{i}, \ldots, \sigma_{N}) \in \Sigma_{1} \times \ldots \times \Sigma_{i-1} \times \Sigma_{i+1} \times \ldots \times \Sigma_{N}} ,\beta \otimes \beta'\rangle
     \end{align*}

  \item The Kronecker product operation: Let $T = \langle\alpha, \{A_{\sigma_{1},\ldots, \sigma_{N}}\}_{(\sigma_{1}, \ldots, \sigma_{N}) \in \Sigma_{1} \times \ldots \times \Sigma_{N}}, \beta\rangle$ and $T' = \langle\alpha', \{A'_{\sigma_{1},\ldots, \sigma_{N}}\}_{(\sigma_{1}, \ldots, \sigma_{N}) \in \Sigma_{1} \times \ldots \times \Sigma_{N}}, \beta'\rangle$ be two N-Alphabet WAs over $\Sigma_{1} \times \ldots \times \Sigma_{N}$. The Kronecker product between $T$ and $T'$, $T \otimes T'$, is parametrized as:
  $$T \otimes T' = \langle\alpha \otimes \alpha', \sum\limits_{\sigma \in \Sigma} A_{\sigma_{1}, \ldots, \sigma_{N}} \otimes A'_{\sigma_{1}, \ldots, \sigma_{N}}, \beta \otimes \beta'\rangle$$
\end{enumerate}
\end{proposition}

\begin{proof}
    Let $N$ be an integer, and $\{\Sigma_{i}\}_{i \in [N]}$, a collection of finite alphabets.
    \begin{enumerate}
     \item For the projection operation: Fix $i \in [N]$. Let $T = \langle\alpha', \{A'_{\sigma_{1},\ldots, \sigma_{N}}\}_{(\sigma_{1}, \ldots, \sigma_{N}) \in \Sigma_{1} \times \ldots \times \Sigma_{N}}, \beta'\rangle$ be an N-Alphabet WA over $\Sigma_{1} \times \ldots \times \Sigma_{N}$, and let $A$ be a WA over $\Sigma_{i}$. Let there be some $(w^{(1)}, \ldots, w^{(N)}) \in \Sigma_{1}^{*} \times \ldots \Sigma_{N}^{*}$, such that $|w^{(1)}| = \ldots = |w^{(N)}| = L$. We have:

    \begin{align*}
        f_{\Pi_{i}(A,T)}(w^{(1)}, \ldots, w^{(i-1)}, w^{(i+1)}, w^{(N)}) &= \sum\limits_{w \in \Sigma_{i}^{L}} f_{A}(w) \cdot f_{T}(w^{(1)}, \ldots, w^{(i-1)}, w, w^{(i+1)}, w^{(N)}) \\
        &= \sum\limits_{w \in \Sigma_{i}^{L}} \left( \alpha^{T} \cdot \prod\limits_{j=1}^{L} A_{w_{j}} \cdot \beta \right) \cdot \left( \alpha'^{T} \cdot \prod\limits_{j=1}^{L} A'_{w_{j}^{(1)}, \ldots, w_{j}^{(i)}, \ldots w_{j}^{(N)}} \cdot \beta' \right)  \\
        &= \sum\limits_{w \in \Sigma_{i}^{L}} (\alpha \otimes \alpha')^{T} \cdot \left[ \prod\limits_{j=1}^{L} A_{w_{j}} \otimes A'_{w_{j}^{(1)}, \ldots, w_{j}^{(i)} \ldots w_{j}^{(N)}} \right] \cdot (\beta \otimes \beta') \\
        &= (\alpha \otimes \alpha')^{T} \cdot \prod\limits_{j=1}^{L} \left( \sum\limits_{\sigma \in \Sigma_{i}} A_{\sigma} \otimes A'_{w_{j}^{(1)}, \ldots, \sigma, w_{j}^{(N)}} \right) \cdot (\beta \otimes \beta')
    \end{align*}
    where the third equality is obtained using the mixed-product property of the Kronecker product between matrices.
    \item The Kronecker product operation: Let $T = \langle\alpha, \{A_{\sigma_{1},\ldots, \sigma_{N}}\}_{(\sigma_{1}, \ldots, \sigma_{N}) \in \Sigma_{1} \times \ldots \times \Sigma_{N}}, \beta\rangle$ and $T' = \langle\alpha', \{A'_{\sigma_{1},\ldots, \sigma_{N}}\}_{(\sigma_{1}, \ldots, \sigma_{N}) \in \Sigma_{1} \times \ldots \times \Sigma_{N}}, \beta'\rangle$ be two N-Alphabets WA over $\Sigma_{1} \times \ldots \times \Sigma_{N}$. Let $(w^{(1)}, \ldots, w^{(N)}) \in \Sigma_{1}^{*} \times \ldots \Sigma_{N}^{*}$ such that $|w^{(1)}| = \ldots = |w^{(N)}| = L$. We have:
    \begin{align*}
        f_{T \otimes T'}(w^{(1)}, \ldots, w^{(N)}) &= f_{T}(w^{(1)}, \ldots, w^{(N)}) \cdot f_{T'}(w^{(1)}, \ldots, w^{(N)}) \\
        &= (\alpha^{T} \cdot \prod\limits_{j=1}^{L} A_{w_{j}^{(1)}, \ldots w_{j}^{N}} \cdot \beta) \cdot (\alpha'^{T} \cdot \prod\limits_{j=1}^{L} A'_{w_{j}^{(1)}, \ldots w_{j}^{N}} \cdot \beta') \\
        &=  (\alpha \otimes \alpha')^{T} \cdot \prod\limits_{j=1}^{L} \left( A_{w_{j}^{(1)}, \ldots w_{j}^{N}} \otimes A'_{w_{j}^{(1)}, \ldots w_{j}^{N}} \right) \cdot (\beta \otimes \beta')
    \end{align*}
    where the last equality is obtained using the mixed-product property of the Kronecker product between matrices.
    \end{enumerate}
\end{proof}

\subsection{Proof of Lemma \ref{lemma:shapasoperations}}

In this segment, we provide the proof of the main lemma of section \ref{sec:tractable}:

\begin{unumberedlemma}
Fix a finite alphabet $\Sigma$. Let $f$ be a WA over $\Sigma$, and consider a sequence $(w, w^{\text{reff}}) \in \Sigma^{*} \times \Sigma^{}$ (representing an input and a basline $\x, \x^{\text{reff}} \in \mathcal{X}$) such that $|w| = |w^{\text{reff}}|$. Let $i \in [|w|]$ be an integer, and $\mathcal{D}_P$ be a distribution modeled by an HMM over $\Sigma$. Then:
         {\small 
            \begin{align*}
         \phi_i
         (f,w,i,\mathcal{D}_P) = \quad\quad\quad\quad\quad\quad\quad\quad\quad\quad\quad\quad\quad\quad\quad\quad \\ \Pi_{1} (A_{w,i}, \Pi_{2}(\mathcal{D}_P, \Pi_{3}(f,T_{w,i}) 
          - \Pi_{3}(f,T_{w}) ) ); \quad\quad \\
            \Phi_i(f,i,n,\mathcal{D}_P) = \quad\quad\quad\quad\quad\quad\quad\quad\quad\quad\quad\quad\quad\quad\quad\quad \\ 
            \Pi_{0} ( \Pi_{2}(\mathcal{D}_P, A_{i,n} \otimes \Pi_{2}(\mathcal{D}_P, 
             \Pi_{3}(f,T_{i}) - \Pi_{3}(f,T)))); \quad\\
            \phi_b(f,w,i,w^{\text{reff}}) = \quad\quad\quad\quad\quad\quad\quad\quad\quad\quad\quad\quad\quad\quad\quad\quad \\ \Pi_{1} (A_{w,i}, \Pi_{2}(f_{w^{\text{reff}}}, \Pi_{3}(f,T_{w,i}) 
          - \Pi_{3}(f,T_{w})));\quad \\
            \Phi_b(f,i,n,w^{\text{reff}},\mathcal{D}_P) = \quad\quad\quad\quad\quad\quad\quad\quad\quad\quad\quad\quad\quad\quad \\ \Pi_{0} ( \Pi_{2}(\mathcal{D}_P, A_{i,n} \otimes \Pi_{2}(f_{w^{\text{reff}}} ,
            \Pi_{3}(f,T_{i}) - \Pi_{3}(f,T)) ))
               \end{align*}
             } where:
    \begin{itemize}
        \item $A_{w,i}$ is a 1-Alphabet WA over $\Sigma_{\#}$ implementing the uniform distribution over coalitions excluding the feature $i$ (i.e., $f_{A_{w,i}} = \mathcal{P}_{i}^{w}$);
        \item $T_{w}$ is a 3-Alphabet \emph{WA} over $\Sigma_{\#} \times \Sigma \times \Sigma$ implementing the function: $ g_{w}(p,w',u) := I(\texttt{do}(p,w',w) = u)$.
        \item $T_{w,i}$ is a 3-Alphabet \emph{WA} over $\Sigma_{\#} \times \Sigma \times \Sigma$ implementing the function: $            g_{w,i}(p,w',u) := I(\texttt{do}(\texttt{swap}(p,w_{i},i),w',w) = u)$.
        \item $T$ is a 4-Alphabet \emph{WA} over $\Sigma_{\#} \times \Sigma \times \Sigma \times \Sigma$ given as: $g(p,w',u,w) := g_{w}(p,w',u)$.
        \item $T_{i}$ is a 4-Alphabet \emph{WA} over $\Sigma_{\#} \times \Sigma \times \Sigma \times \Sigma$ given as: $ g_{i}(p,w',u,w) := g_{w,i}(p,w',u)$.
        \item $A_{i,n}$ is a 2-Alphabet \emph{WA} over $\Sigma_{\#} \times \Sigma$ implementing the function:
         $g_{i,n}(p,w) := I(p \in \mathcal{L}_{i}^{w}) \cdot \mathcal{P}_{i}^{w}(p)$,
         where $|w| = |p| = n$.
       \item $f_{w^{\text{reff}}}$ is an \emph{HMM} such that the probability of generating $w^{\text{reff}}$ as a prefix is equal to $1$.
    \end{itemize}

\end{unumberedlemma}

\begin{proof}

We will prove the complexity results specifically for the cases involving either local or global \emph{Interventional} SHAP. The corresponding proof for local and global Baseline SHAP can be derived by following the same approach as in the interventional case, with the sole modification of replacing $\mathcal{D}_{p}$ with the HMM that models the empirical distribution induced by the reference instance $w^{\text{reff}}$.



Fix a finite alphabet $\Sigma$. Let $f$ be a WA over $\Sigma$, $w \in \Sigma^{*}$ a sequence, $i \in [|w|]$ an integer, and $\mathcal{D}_{P}$ an HMM. Define $A_{w,i}$ as a WA and $T_{w},~T_{w,i}$ as two 3-Alphabet WAs as stated in the lemma. We divide the following proof into two parts. The first part addresses the local interventional SHAP version, while the second part covers the global version.

    \begin{enumerate}
        \item For \emph{local Interventional SHAP} we have that:
        \begin{align}
            \phi_i(f,w,i,\mathcal{D}_P) &= \mathbb{E}_{p \sim \mathcal{P}_{i}^{w}} \left[ V_{I}(w,\texttt{swap}(p,w_{i},i), \mathcal{D}_P) - 
 V_{I}(w,p,\mathcal{D}_P) \right] \nonumber \\ 
 &= \sum\limits_{p \in \Sigma_{\#}^{|w|}} f_{A_{w,i}}(p) \left[ \sum\limits_{w' \in \Sigma^{w}} \mathcal{D}_P(w') \cdot \left[ f(\texttt{do}(\texttt{swap}(p,w'_{i},i),w',w)) - f(\texttt{do}(p,w',w)\right] \right] \label{eq:locishap}
        \end{align}
        
 Note that for any $p \in \Sigma_{\#}^{|w|}$ and $(w',u) \in \Sigma^{|w|} \times \Sigma^{|w|}$, we have:
 \begin{equation} \label{eq:obswi}
     f(\texttt{do}(\texttt{swap}(p,w'_{i},i),w',w) = \sum\limits_{u \in \Sigma^{|w|}} f(u) \cdot g_{w,i}(p,w',u) = f_{\Pi_{3}(f,T_{w,i})}(p,w')  
 \end{equation}
 and,
 \begin{equation} \label{eq:obsw}
     f(\texttt{do}(p,w',w)) = \sum\limits_{u \in \Sigma^{|w|}} f(u) \cdot g_{w}(p,w',w) = f_{\Pi_{3}(f,T_{w})}(p,w')
 \end{equation}
 where $g_{w,i}$ and $g_{w}$ are defined implicitly in the body of the lemma statement.

 By plugging equations \eqref{eq:obswi} and \eqref{eq:obsw} in Equation \eqref{eq:locishap}, we obtain:
 \begin{align*}
     \phi_i(f,w,i, \mathcal{D}_P) &=  \sum\limits_{p \in \Sigma_{\#}^{|w|}} f_{A_{w,i}}(p) \left[ \sum\limits_{w' \in \Sigma^{|w|}} \mathcal{D}_P(w') \cdot [ f_{\Pi_{3}(M,T_{w,i})}(p,w') - f_{\Pi_{3}}(M, T_{w}) (p,w') ] \right]
 \end{align*}
 To ease exposition, we employ the symbol $\Tilde{T}$ to refer to the intermediary 2-Alphabet WA over $\Sigma_{\#} \times \Sigma$ defined as:
 $$\Tilde{T} \myeq \Pi_{3}(M,T_{w,i}) - \Pi_{3}(M,T_{w})$$
 Then, we have:
\begin{align*}
    \phi_i(f,w,i, \mathcal{D}_P) &= \sum\limits_{p \in \Sigma_{\#}^{|w|}} f_{A_{w,i}}(p) \left[ \mathcal{D}_P(w') \cdot f_{\Tilde{T}}(p,w') \right] \\
    &= \sum\limits_{p \in \Sigma_{\#}^{|w|}} f_{A_{w,i}}(p) \cdot f_{\Pi_{2}(\mathcal{D}_P,\Tilde{T})}(p) \\
    &= \Pi_{1}(A_{w,i}, \Pi_{2}(\mathcal{D}_P,\Tilde{T}))
\end{align*}
  \item For \emph{Global Interventional SHAP}, the proof follows the same structure as that of the Local version. We hence have that:
 \begin{align}
     \phi_i(f,i,n,P) &= \sum\limits_{w \in \Sigma^{n}} \mathcal{D}_P(w) \cdot \phi_i(f,w,i, \mathcal{D}_P) \nonumber \\
     &= \sum\limits_{w \in \Sigma^{n}}  \mathcal{D}_P(w) \sum\limits_{p \in \Sigma_{\#}}^{|w|} f_{A_{w,i}}(p) \left[ \sum\limits_{w' \in \Sigma^{w}} \mathcal{D}_P(w') \cdot \left[ f(\texttt{do}(\texttt{swap}(p,w'_{i},i),w',w)) - f(\texttt{do}(p,w',w)\right] \right] \label{eq:gloishapwa}
 \end{align}

   Note that for any $p \in \Sigma_{\#}^{n}$ and $(w',u,w) \in \Sigma^{n} \times \Sigma^{n} \times \Sigma^{n}$, we have:
 \begin{equation} \label{eq:obsi}
     f(\texttt{do}(\texttt{swap}(p,w'_{i},i),w',w) = \sum\limits_{u \in \Sigma^{n}} f(u) \cdot g_{i}(p,w',u,w) = f_{\Pi_{3}(M,T_{i})}(p,w',w)  
 \end{equation}
 and,
 \begin{equation} \label{eq:obs}
     f(\texttt{do}(p,w',w)) = \sum\limits_{u \in \Sigma^{|w|}} f(u) \cdot g(p,w',u,w) = f_{\Pi_{3}(M,T)}(p,w',w)
 \end{equation}
 where $g_{i}$ and $g$ are functions defined implicitly in the lemma statement.

 By, again, plugging equations \eqref{eq:obsi} and \eqref{eq:obs} into the equation \ref{eq:gloishapwa}, we obtain:
 \begin{align*}
     \phi_i(f,i,n,P) &= \sum\limits_{w \in \Sigma^{n}} \mathcal{D}_P(w) \cdot \sum\limits_{p \in \Sigma_{\#}^{n}}  f_{A_{i,n}}(p,w) \left[ \sum\limits_{w' \in \Sigma^{n}} \mathcal{D}_P(w') \cdot f_{\Tilde{T}}(p,w',w) \right] \\
     &= \sum\limits_{w \in \Sigma^{n}}  \mathcal{D}_P(w) \cdot \sum\limits_{p \in \Sigma_{\#}^{n}} f_{A_{i,n}}(p,w) \cdot f_{\Pi_{2}(\mathcal{D}_P,\Tilde{T})}(p,w) \\
     &= \sum\limits_{w \in \Sigma^{n}}  \mathcal{D}_P(w) \sum\limits_{p \in \Sigma_{\#}^{n}} f_{A_{i,n} \otimes \Pi_{2}(\mathcal{D}_P,\Tilde{T})}
     (p,w) \\
     &= \sum\limits_{p \in \Sigma_{\#}^{n}} f_{\Pi_{2}(\mathcal{D}_P, A_{i,n} \otimes \Pi_{2}(\mathcal{D}_P,\Tilde{T}))}(p) \\
     &= \Pi_{0} \left( \Pi_{2}(\mathcal{D}_P, A_{i,n} \otimes \Pi_{2}(\mathcal{D}_P,\Tilde{T}))\right)     
  \end{align*} 
    \end{enumerate}
\end{proof}

\subsection{Proof of proposition \ref{prop:nletterwaconstruction}}
In this segment, we shall provide a constructive proof of all machines defined implicitly in Lemma \ref{lemma:shapasoperations}. Formally, we shall prove the following:

\begin{unumberedproposition}
 The N-Alphabet WAs $A_{w,i}$, $T_{w}$, $T_{w,i}$, $T$, $T_{i}$, $A_{i,n}$ and the HMM $f_{w^{\text{reff}}}$ can be constructed in polynomial time with respect to $|w|$ and $|\Sigma|$.
\end{unumberedproposition}
\begin{table}[ht]
    \centering
        \caption{The summary of complexity results (in terms of the length of the sequence to explain $w$ and the size of the alphabet $|\Sigma|$) for the construction of models in proposition \ref{prop:nletterwaconstruction}}
    \begin{tabular}{|c|c|c|c|}
        \hline
         & \textbf{Running Time complexity} & \textbf{Output Size} & \textbf{Output's Alphabet size ($N$)}  \\ \hline
        $A_{w,i}$ & $O(|w|^{3})$ & $O(|w|^{3})$ & $1$  \\ \hline
        $T_{w,i}$ &  $O(|\Sigma|^{3} \cdot |w|)$ & $|w|$ & $3$  \\ \hline
        $T_{w}$ &$O(|\Sigma|^{3} \cdot |w|)$ & $|w|$ & $3$  \\ \hline
        $T_{i}$ & $O(|w|)$ & $O(|w|)$ & $4$  \\ \hline
        $T$ & $O(1)$ & $O(1)$ & $4$ \\ \hline
        $A_{i,n}$ &  $O(|\Sigma|^{2} \cdot |w|^{4})$ & $O(|w|^{4})$ & $2$  \\ \hline
        $f_{w^{reff}}$ & $O(|w|)$ & $O(|w|)$ & $1$  \\ \hline
    \end{tabular}
    \label{app:fig:prop6}
\end{table}
We split the proof of this proposition into four sub-sections. The first sub-section is dedicated to the construction of $A_{w,i}$ and $A_{i,n}$. The second sub-section is dedicated to the construction of $T_{w,i}$ and $T_{i}$. The third sub-section is dedicated to the construction of $T_{i}$ and $T$. The final subsection treats the construction of $f_{w^{reff}}$. The running time of all these constructions as well as the size of their respective outputs are summarized in Table \ref{app:fig:prop6}.  

\subsubsection{The construction of $A_{w,i}$ and $A_{i,n}$} 
Recall $A_{w,i}$ is a WA over $\Sigma_{\#}$ that implements the probability distribution: 
\begin{equation} \label{app:eq:piw}
\mathcal{P}_{i}^{(w)}(p) \myeq \frac{1}{|w|} \sum\limits_{k=1}^{|w|} \mathcal{P}_{i,k}^{(w)}(p)
\end{equation}
where $\mathcal{P}_{i,k}^{(w)}$ is the uniform distribution over patterns belonging to the following set:
$$\mathcal{L}_{i,k}^{(w)} \myeq \{ p \in \Sigma_{\#}^{|w|}: ~ w \in L_{p} \land |p|_{\#} = k \land p_{i} = \# \}$$
The 2-Alphabet WA $A_{i,n}$ can be seen as a global version of $A_{w,i}$ where the sequence $w$ becomes a part of the input of the automaton. Next, we shall provide the construction for $A_{w,i}$. The construction of $A_{i,n}$ is somewhat similar to that of $A_{w,i}$ and will be discussed later on in this section.

\paragraph{The construction of $A_{w,i}$.} Algorithm \ref{app:alg:Awi} provides the pseudo-code for constructing $A_{w,i}$. By noting that the target probability distribution $\mathcal{P}_{i}^{(w)}$ is a linear combination of the set of functions $ \mathcal{F} = \{\mathcal{P}_{i,k}^{(w)}\}$ (Equation \ref{app:eq:piw}), the strategy of our construction  consists at iteratively constructing a sequence of WAs $\{A_{i,k}^{(w)}\}_{k \in [|w|]}$, each of which implements a function $F \in \mathcal{F}$. Thanks to the closure of WAs under linear algebra operations (addition, and multplication with a scalar) and the tractability of implementing them, one can recover the target WA $A_{w,i}$ using linear combinations of the collection of WAs $ \mathcal{A} = \{A_{i,k}^{(w)}\}_{k \in [|w|]}$. 

\begin{algorithm}
\caption{Construction of $A_{w,i}$}
\label{app:alg:Awi}
\begin{algorithmic}[1]
\REQUIRE A sequence $w \in \Sigma^{*}$, An integer $i \in [|w|]$
\ENSURE A WA $A_{w,i}$  
\STATE Initialize $A_{w,i}$ as the empty WA
\FOR{$k = 1 \ldots |w|$}{
  \STATE Construct a DFA $\bar{A}_{i,k}^{|w|}$ that accepts the language $\mathcal{L}_{i,k}^{(w)}$
   \STATE $A_{w,i} \leftarrow A_{w,i} + \frac{1}{|w| \cdot |\mathcal{L}_{i,k}^{|w|}} \cdot \bar{A}_{i,k}^{|w|}$
}
\ENDFOR
\RETURN $A_{w,i}$
\end{algorithmic}
\end{algorithm}

The missing link to complete the construction is to show how each element in the set $\mathcal{A}$ is constructed. For a given $(i,k) \in [|w|]^{2}$, the strategy consists at constructing a DFA that accepts the language $\mathcal{L}_{i,k}^{(w)}$ (denoted $\bar{A}_{i,k}^{|w|}$ in Line 3 of Algorithm~\ref{app:alg:Awi})). Since the function implemented by $A_{i,k}^{(w)}$ represents the uniform probability distribution over patterns in $\mathcal{L}_{i,k}^{(w)}$, then $A_{i,k}^{(w)}$ can be recovered by simply normalizing the DFA $\bar{A}_{i,k}^{(w)}$ by the quantity $\frac{1}{|\mathcal{L}_{i,k}^{(w)}|}$.  

Given a sequence $w \in \Sigma^{*}$ and $(i,k) \in [|w|]^{2}$, the construction of the DFA, $\bar{A}_{i,k}^{(w)}$, is given as follows:
\begin{itemize}
    \item \textbf{The state space:} $Q = [|w|+1] \times \{0, \ldots, |w|\}$ (For a given state $(l,e) \in [|w|] \times [|w|]$, the element $l$ tracks the position of the running pattern, and $e$ computes the number of $\{\#\}$ symbols of the running pattern.
    \item \textbf{Initial State:} $(1,0)$
    \item \textbf{The transition function:} For a given state $(l,e) \in [|w|] \times \{0, \ldots, |w|\}$ and a symbol $\sigma \in \Sigma_{\#}$, we have:
    $$\delta((l,e),\sigma) = \begin{cases}
     (l+1, e+1) & \text{if} ~~\sigma = \# \\
     (l+1,e) & \text{if} ~~\sigma \in \Sigma \land w_{l} = \sigma \land l \neq i
    \end{cases}
    $$
    \item \textbf{The final state:} $F = (|w|+1 , k)$
\end{itemize}

\paragraph{Complexity.} The complexity of constructing $\bar{A}_{i,k}^{(w)}$ for a given $k \in [|w|]$ is equal to $O(|w|^{2})$. Consequently, taking into account the iterative procedure in Algorithm \ref{app:alg:Awi}, this latter algorithm runs in $O(|w|^{3})$. In addition, the size of the resulting $A_{w,i}$ is also $O(|w|^{3})$. 

\paragraph{The construction of $A_{i,n}$.} As previously noted, the 2-Alphabet WA can be interpreted as the global equivalent of $A_{w,i}$. Formally, for an integer $n > 0$ and $i \in [n]$, the 2-Alphabet WA $A_{i,n}$ realizes the function $g_{i,n}$ over $\Sigma_{\#} \times \Sigma$ as defined below:
\begin{equation} \label{app:eq:gin}
g_{i,n}(p,w) = I(p \in \mathcal{L}_{i}^{w}) \cdot \mathcal{P}_{i}^{w}(p)
\end{equation}

In light of Equation \eqref{app:eq:gin}, the construction of $A_{i,n}$ aligns to the following three step procedure:
\begin{enumerate}
    \item Construct a 2-Alphabet DFA $A'_{w,i}$ over $\Sigma_{\#} \times \Sigma$ that accepts the language $I(p \in \mathcal{L}_{i}^{w})$
    \item Construct a 2-Alphabet WA $\Tilde{A}_{w,i}$ over $\Sigma_{\#} \times \Sigma$ such that: $f_{\Tilde{A}_{w,i}}(p,w) = f_{A_{w,i}}(p)$ (Note that $f_{\Tilde{A}_{w,i}}$ is independent of its second argument)
    \item Return $A'_{w,i} \otimes A_{w,i}$ 
\end{enumerate}

The derivation of the 2-Alphabet WA $\Tilde{A}_{w,i}$ from $A_{w,i}$, whose construction is detailed in Algorithm \ref{app:alg:Awi}, is straightforward. It remains to demonstrate how to construct the 2-Alphabet DFA $A'_{w,i}$ (Step 1).

\paragraph{The construction of $A'_{w,i}$.} Recall that $\mathcal{L}_{i}^{(w)} \myeq \bigcup\limits_{k=1}^{|w|} \mathcal{L}_{i,k}^{(w)}$. Constructing a 2-Alphabet DFA that accepts the language $I(p \in \mathcal{L}_{i}^{w})$ is relatively simple: It involves verifying the conditions $w \in L_{p}$ and $p_{i} = \#$ for the running sequence $(p,w) \in \Sigma_{\#}^{*} \times \Sigma^{*}$. The construction is provided as follows:

\begin{itemize}
    \item \textbf{The state space:} $Q = [|w|]$
    \item \textbf{The initial state:} The state $1$
    \item \textbf{The transition function:} For a state $q \in Q$ and $(\sigma,\sigma') \in \Sigma_{\#} \times \Sigma$, then we have that $\delta(q, (\sigma, \sigma')) = q+1$ holds if and only if the predicate: 
     \begin{equation} \label{app:eq:predicate}
      \left(q \neq i \land (\sigma = \#) \lor (\sigma = \sigma') \right) \lor (q = i \land \sigma = \# )
     \end{equation}
     is true.

In essence, the predicate defined in Equation~\eqref{app:eq:predicate} captures the idea that, for a given position $q$ in the current pair of sequences $(p,w) \in \Sigma_{\#}^{*} \times \Sigma^{*}$, either $w_{q} \in L_{p_{q}}$ when $q \neq i$ (ensuring the condition $w \in L_p$), or $p_q = \#$ when $q = i$ (ensuring the condition $p_i = \#$).
     
    \item \textbf{The final state:} The state $|w|$
\end{itemize}

\paragraph{Complexity.} The construction of $A'_{w,i}$ requires $O(|w|)$ running time, and its size is equal to $O(|w|)$. The running time complexity for the construction of $A_{i,n}$ is dictated by the application of the Kronecker product between the 2-Alphabet WAs $A'_{w,i} \otimes A_{w,i}$ (Step 3 of the procedure outlined above). Given the complexity of computing the Kronecker product between N-Alphabet WAs (see Table \ref{app:fig:operationswas}) and the sizes of $A'_{w,i}$ and $A_{w,i}$, the overall time complexity is equal to $O(|\Sigma|^{2} \cdot |w|^{4})$. The size of $A_{i,n}$ is equal to $O(|w|^{4})$.  

\subsubsection{The construction of $T_{w,i}$ and $T_{w}$} 

The constructions of $T_{w,i}$ and $T_{w}$ are highly similar. Consequently, this section will mainly concentrate on the complete construction of $T_{w,i}$, as it introduces an additional challenge with the inclusion of the \texttt{swap} operation in the function implemented by this 3-Alphabet WA. A brief discussion on the construction of $T_{i}$ will follow at the end of this section, based on the approach used for $T_{w,i}$.


Recall that $T_{w,i}$ is a 3-Alphabet DFA over $\Sigma_{\#} \times \Sigma \times \Sigma$ that implements the function: 

$$g_{w,i}(p,w',u) = I(\texttt{do}(\texttt{swap}(p,w_{i},i),w',w)= u)$$
for a triplet $(p,w',u) \in \Sigma_{\#}^{*} \times \Sigma^{*} \times \Sigma^{*}$, for which $|p| = |w'| = |u|$.

To ease exposition, we introduce the following predicate:

\begin{equation} \label{app:eq:predicateTwi}
 \Phi(\sigma_{1}, \sigma_{2}, \sigma_{3}, \sigma_{4}) \myeq  \left( \sigma_{1} = \# \land \sigma_{3} = \sigma_{2} \right) \lor \left(\sigma_{1} \neq \# \land \sigma_{3} = \sigma_{4} \right) 
\end{equation}

where $(\sigma_{1}, \sigma_{2}, \sigma_{3}, \sigma_{4}) \in \Sigma_{\#} \times \Sigma \times \Sigma \times \Sigma$. 

The construction of $T_{w,i}$ is given as follows:
\begin{itemize}
    \item \textbf{The state space:} $Q = [|w| + 1]$
    \item \textbf{The initial state:} $q_{init} = 1$
    \item \textbf{The transition function:} For a state $q \in Q$, and a tuple of symbols $(\sigma_{1}, \sigma_{2}, \sigma_{3}) \in \Sigma_{\#} \times \Sigma \times \Sigma$, we have $\delta(q, (\sigma_{1}, \sigma_{2}, \sigma_{3})) = q+1$ if and only if the predicate:
    \begin{equation} \label{app:eq:Twi}
    \left[q \neq i \land \Phi(\sigma_{1},\sigma_{2}, \sigma_{3}, w_{q}) \right] \lor \left[q = i \land \sigma_{3} = w_{q} \right]
    \end{equation}
    is true.
    \item \textbf{The final state:} $|w| + 1$.
\end{itemize}

\textbf{Note.} As previously mentioned, the construction of the 3-Alphabet WA, $T_{w}$, is quite similar to that of $T_{w,i}$. To obtain a comparable construction of $T_{w}$, one can easily adjust the predicate defined in equation~\ref{app:eq:Twi} to $\Phi(\sigma, w_{q}, \sigma_{3}, \sigma_{4})$.

\subsubsection{The construction of $T_{i}$ and $T$.} 
The 4-Alphabet WAs $T_{i}$ and $T$ can be seen as the global counterparts of $T_{w,i}$ and $T_{w}$, respectively. In addition, their constructions are somewhat simpler than the latter.

\paragraph{The construction of $T_{i}$.} For a given $i \in \mathbb{N}$, the 4-Alphabet WA $T_{i}$ is of size $i$, and constructed as follows: 
\begin{itemize}
     \item \textbf{The state space:} $Q = [i+1]$
    \item \textbf{The initial state:} $q_{init} = 1$
    \item \textbf{The transition function:} For a state $q \in Q$, and $(\sigma_{1}, \sigma_{2}, \sigma_{3}, \sigma_{4}) \in \Sigma_{\#} \times \Sigma \times \Sigma \times \Sigma$, we have:
   $$\delta(q, (\sigma_{1}, \sigma_{2}, \sigma_{3}, \sigma_{4})) = 
   \begin{cases}
       q+1 & \text{if} ~~ [q < i \land \Phi(\sigma_{1}, \sigma_{2}, \sigma_{3}, \sigma_{4})] \lor [q=i \land \sigma_{3} = \sigma_{4}] \\
        i+1 &  \text{if} ~~ q = i+1 \land \Phi(\sigma_{1}, \sigma_{2}, \sigma_{3}, \sigma_{4})
   \end{cases}
   $$
   \item \textbf{The final state:} $i+1$
\end{itemize}

\paragraph{The construction of $T$.} The 4-Alphabet WA $T$ is an automaton with a single state, formally defined as follows:
\begin{itemize}
    \item \textbf{The state space:} $Q = \{1\}$,
    \item \textbf{The initial state:} $q_{init} = 1$.
    \item \textbf{The transition function:} For $(\sigma_{1}, \sigma_{2}, \sigma_{3}, \sigma_{4}) \in \Sigma_{\#} \times \Sigma \times \Sigma \times \Sigma$, we have that:
    $\delta(1, (\sigma_{1}, \sigma_{2}, \sigma_{3}, \sigma_{4})) = 1$ holds, if and only if the predicate $\Phi(\sigma_{1}, \sigma_{2}, \sigma_{3}, \sigma_{4})$ is true. 
\end{itemize}

\subsubsection{The construction of $f_{w^{reff}}$}
Given a sequence $w^{ref} \in \Sigma^{*}$, the machine $f_{w^{reff}}$ is defined as an HMM such that the probability of generating the sequence $w^{ref}$ as a prefix is equal to 1. The set of HMMs that meet this condition is infinite. In our case, we opt for an easy construction of an HMM, denoted $f_{w^{reff}}$, belonging to this set.

The state space of $f_{w^{reff}}$ is given as $Q = [|w_{ref}| + 1]$ states. Similarly to the constructions above, each state is associated with a position within the emitted sequence of the HMM. Its functioning mechanism can be described using the following recursive procedure: 
\begin{itemize}
    \item The HMM starts from the state $q = 1$ with probability $1$. The probability of emitting the symbol $w^{ref}_{1}$ and transitioning to the state $q= 2$ is equal to $1$.
    \item For $q \in [|w^{ref}|]$, the probability of generating the symbol $w^{ref}_{q}$ and transitioning to the state $q+1$ is equal to $1$.
\end{itemize}
When the HMM reaches the state $|w^{ref}| + 1$, it generates a random symbol in $\Sigma$ and remains at state $|w^{ref}| + 1$ with probability 1. One can readily verify (by a straightforward induction argument) that this HMM generates infinite sequences prefixed by the sequence $w^{ref}$ with probability 1. 

\paragraph{Complexity.} The running time of this construction is equal to $O(|w^{ref}|)$, and the size of the obtained HMM $f_{w^{reff}}$ is equal to $O(|w^{ref}|)$.
 \section{From Sequential Models to Non-Sequential Models: Reductions and Inter-inclusions} \label{app:reductiontree}

In Section~\ref{subsec:tree2WA} of the main paper, we demonstrated how the tractability result for computing various SHAP variants for WAs extends to several non-sequential models, including Ensemble Trees for Regression and Linear Regression Models (Theorem~\ref{cor:reductions}). In this section, we will present a complete and rigorous proof of this connection, along with additional theoretical insights into the relationships between these models and distributions. Specifically:


\begin{enumerate}
\item 
In the first subsection, we establish the proof of Theorem~\ref{cor:reductions}. This proof is based on several reductions, including those from linear regression models and ensemble trees to weighted automata, as well as from empirical distributions to the family $\overrightarrow{\texttt{HMM}}$.

\item In the second subsection, we present additional reductions that, while not explicitly used to prove the complexity results stated in the article, showcase the expressive power of the HMM class in modeling various families of distributions relevant to SHAP computations. These include distributions represented by Naive Bayes models and Markovian distributions. Moreover, we will highlight certain complexity results, listed in Table~\ref{fig:summaryresults}, that were not directly mentioned in the main text but are illuminated by these reduction findings.


\end{enumerate} 

\subsection{Proof of Theorem \ref{cor:reductions}}

Recall the statement of Theorem \ref{cor:reductions}:

\begin{unumberedtheorem}
        Let $\mathbb{S}:= \{ \emph{\texttt{LOC}}, \emph{\texttt{GLO}} \}$, $\mathbb{V} := \{\texttt{\emph{B}}, \texttt{\emph{I}}\}$, $\mathbb{P}:= \{ \texttt{\emph{EMP}}, \overrightarrow{\emph{\texttt{HMM}}} \} $, and $\mathbb{F}:= \{\emph{\texttt{DT}}, \emph{\texttt{ENS-DT}}_{\emph{\texttt{R}}}, \emph{\texttt{Lin}}_{\emph{\texttt{R}}}\}$. Then, for any \emph{\texttt{S}} $\in \mathbb{S}$, $\texttt{\emph{V}} \in \mathbb{V}$, $\texttt{\emph{P}} \in \mathbb{P}$, and $\emph{\texttt{F}}\in \mathbb{F}$ the problem  $\emph{\texttt{S-V-SHAP}}(\emph{\texttt{F}}, \texttt{\emph{P}})$ can be solved in polynomial time.
\end{unumberedtheorem}

The proof of this theorem will rely on four inter-model polynomial reductions:

\begin{unumberedclaim}
The following statements hold true:
\begin{enumerate}
    \item $\overrightarrow{\texttt{\emph{HMM}}} \preceq_{\emph{P}} 
\texttt{\emph{HMM}}$
   \item $\texttt{\emph{EMP}} \preceq_{\emph{P}} \overrightarrow{\texttt{\emph{HMM}}}$ 
   \item $\texttt{\emph{ENS-DT}}_{\texttt{\emph{R}}}\preceq_{\emph{P}} \texttt{\emph{WA}}$
   \item $\texttt{\emph{LIN}}_{\texttt{\emph{R}}} \preceq_{\emph{P}} \texttt{\emph{WA}}$
\end{enumerate}
\end{unumberedclaim}

Theorem \ref{cor:reductions} is a direct consequence of the following four claims. This section is structured into four parts, each dedicated to proving one of these claims. Without loss of generality and to simplify the notation, we will assume throughout that all models operate on binary inputs. However, it is important to note that all of our results can be extended to the setting where each input is defined over $k$ discrete values.


\subsubsection{Claim 1: $\overrightarrow{\texttt{\text{HMM}}}$ is polynomially reducible to $\texttt{\text{HMM}}$}

Consider a model $\ar{\tt{\text{M}}} = \langle \pi, \alpha, \{T_{i}\}_{i \in [n]}, \{O_{i}\}_{i \in [n]}\rangle$, where the size of $\ar{\text{M}}$ is denoted as $m$. This model $\ar{\text{M}}$ encodes a probability distribution over the hypercube $\{0,1\}^{n}$. The goal is to prove that a model $M \in \ttt{\text{HMM}}$, which satisfies the following condition:

\begin{equation} \label{app:eq
} P_{\ar{\text{M}}}(x_{1}, \ldots, x_{n}) = P_{\text{M}}^{(n)}(x_{1}, \ldots, x_{n}), \end{equation}

can be constructed in polynomial time with respect to the size of $\ar{\text{M}}$. Intuitively, this condition asserts that the probability of generating a sequence $x = x_{1}, \ldots, x_{n}$ of length $n$ using $M$ is the same as the probability of generating $(x_{1}, \ldots, x_{n})$ using $\ar{\text{M}}$.







The construction aims to build an HMM $M$ that simulates $\ar{\text{M}}$ up to position $n$. Afterward, $M$ transitions into a dummy hidden state where it stays stuck permanently, emitting symbols uniformly at random. To simulate $\ar{\text{M}}$ within the support $\{0,1\}^{M}$, $M$ must keep track of both the state reached by $\ar{\text{M}}$ after emitting a prefix $x_{\pi{1}}, \ldots x_{\pi{j}}$ for a given $j \in [n]$, and the position reached in the sequence. Once the $n$-th symbol has been emitted, the HMM $\text{M}$ transitions into a dummy state, which emits symbols uniformly at random and remains in that state indefinitely. The detailed construction is provided as follows:


\begin{itemize}
    \item \textbf{The state space:} $Q = [M] \times [n+1]$
    \item \textbf{State initialization:} The HMM $M$ begins generating from the state $(i,1)$ for $i \in [M]$ with a probability of $\alpha(i)$.
    
    \item \textbf{The transition dynamics:} For $(i,k,j) \in [M]^{2} \times [N]$, the HMM $M$ transitions from state $(i,j)$ to state $(k,j+1)$ with a probability of $T_{j}[i,k]$. Additionally, for any $i \in [M]$, we have $T[(i,n+1), (i,n+1)] = 1$ (meaning the HMM remains in the dummy state $n+1$ indefinitely).
    \item \textbf{Symbol emission:} At a state $(i,j) \in [M] \times [n]$, the probability of emitting the symbol $\sigma$ is equal to $O_{\pi(i)}[i,\sigma]$. On the other, for any state $(i,n+1)$ where $i \in [M]$, the HMM $M$ generates the symbol $\sigma$ uniformly at random. 
\end{itemize}

\subsubsection{Claim 2: \texttt{EMP} is polynomially reducible to $\overrightarrow{\texttt{HMM}}$}

The purpose of this subsection is to demonstrate that the class of Empirical distributions can be polynomially reduced to the \texttt{HMM} family. This claim has been referenced multiple times throughout the main article, where it played a key role in deriving complexity results for computing SHAP variants across various ML models, including decision trees, linear regression models, and tree ensemble regression models. Specifically, it shows that computing these variants under distributions modeled by HMMs is at least as difficult as computing them under empirical distributions. 




The strategy for reducing empirical distributions to those modeled by the $\overrightarrow{\text{HMM}}$ family involves two steps:

\begin{enumerate}
    \item The sequentialization step, which aims to transform the vectors in $\mathcal{D}$ into a dataset of binary sequences, referred to as $\texttt{SEQ}(\mathcal{D})$.
    \item The construction of a  model in $\overrightarrow{\text{HMM}}$ that encodes the empirical distribution induced by $\texttt{SEQ}(\mathcal{D})$.  
\end{enumerate}

We will now provide a detailed explanation of each step:

\emph{Step 1 (The sequentialization step):} The sequentialization step is fairly straightforward and has been utilized in \cite{marzouk24a} to transform decision trees into equivalent WAs. The sequentialization operation, denoted as $\texttt{SEQ}(.,.)$, is defined as follows:
$$\texttt{SEQ}(\overrightarrow{x}, \pi) = x_{\pi(1)} \ldots x_{\pi(N)}$$

where $\pi$ is a permutation. Essentially, the $\texttt{SEQ}(.,.)$ operation transforms a given vector in $\{0,1\}^{N}$ into a binary string, with the order of the feature variables determined by the permutation $\pi$. From here on, without loss of generality, we assume $\pi$ to be the identity permutation.


\emph{Step 2 (The construction of the HMM):} 
Applying the sequentialization step to the dataset $\mathcal{D}$ results in a \emph{sequentialized} dataset consisting of $M$ binary strings, denoted as $\texttt{SEQ}(\mathcal{D})$. The goal of the second step in the reduction is to construct a model in $\overrightarrow{HMM}$ that models the empirical distribution induced by $\texttt{SEQ}(\mathcal{D})$. While this construction is well-known in the literature of Grammatical Inference (refer to
La Higuera book on Grammatical Inference), we will provide the details of this construction below. 


For a given sequence $w$, we denote the number of occurences of $w$ as a prefix in the dataset $P(\mathcal{D})$ as:
$$N(w) \myeq \#|\{x \in \texttt{SEQ}(\mathcal{D}): ~ \exists (s) \in \{0,1\}^{*}: x = ws\}|$$

We also define the set of prefixes appearing in $\texttt{SEQ}(\mathcal{D})$ as the set of all sequences $w \in \{0,1\}^{*}$ such that $N(w) \neq 0$. This set shall be denoted as $\text{P}(\mathcal{D})$. Note that the set $\text{P}$ is prefix-closed \footnote{A set of sequences $S$ is prefix-closed if the set of prefixes of any sequence in $S$ is also in $S$}.

A (stationary) model in $\overrightarrow{\text{HMM}}$ that encodes the empirical distribution induced by the dataset $\texttt{SEQ}(\mathcal{D})$ is a probabilistic finite state automaton parametrized as follows: 
\begin{enumerate}
    \item \textbf{The permutation:} The identity permutation.
    \item \textbf{The state space:} $Q = \text{P}(\mathcal{D})$.
    \item \textbf{The initial state vector:} Given $\sigma \in \{0,1\}$, the probability of starting from the state $\sigma \in \text{P}(\mathcal{D})$ is equal to $\frac{N(\sigma)}{|\mathcal{D}|}$.
    \item \textbf{The transition dynamics:} For a state $w \sigma \in Q$, where $w \in Q$ \footnote{Thanks to the prefix-closedness property of the set $\text{P}(\mathcal{D})$, if $w\sigma$ is in $Q$, then $w$ is also in $Q$.} and $\sigma \in \Sigma$,  the probability of transitioning from the state $w$ to the state $w\sigma$ is equal to $\frac{N(w \sigma)}{N(w)}$. 
    \item \textbf{The symbol emission:} For any state $w \sigma \in Q$, where $w \in Q$ and $\sigma \in \{0,1\}$, the probability of generating the symbol $\sigma$ from the state $w\sigma$ is equal to $1$.
 \end{enumerate}

One can readily prove that the resulting model computes the empirical distribution induced by $\text{P}(\mathcal{D})$. Indeed, by construction, the probability of generating a binary sequence $w$ by the model is equal to the probability of generating the sequence of states $w_{1}, w_{1:2}, \ldots, w_{1:n-1} w$. The probability of generating this state sequence is equal to:
$$\frac{N(w_{1})}{|\mathcal{D}|} \prod\limits_{i=1}^{|w|} \frac{N(w_{1:i+1})}{N(w_{1:i})} = \frac{N(w)}{|\mathcal{D}|}$$

\subsubsection{Claim 3: \texttt{DT} and $\texttt{\text{ENS-DT}}_{\texttt{\text{R}}}$ are polynomially reducible to \texttt{WA}}

The construction of an equivalent WA for either a decision tree or a tree ensemble used in regression tasks (i.e., a model in \texttt{DT} or a model in $\texttt{\text{ENS-DT}}_{\text{R}}$) is built upon the following result provided in \cite{marzouk24a} for the \texttt{DT} family:


\begin{proposition} \label{app:prop:DT2WA}
    There exists a polynomial-time algorithm that takes as input a decision Tree $T \in \texttt{\emph{DT}}$ over the binary feature set $X = \{X_{1}, \ldots, X_{n} \}$ and outputs a WA $A$ such that:
     $$f_{A}(x_{1}\ldots x_{n}) = f_{T}(x_{1}, \ldots, x_{n} )$$
\end{proposition}

Since an ensemble of decision trees used for regression tasks is a linear combination of decision trees, and the family of WAs is closed under linear combination operations, with these operations being implementable in polynomial time as demonstrated in section~\ref{app:sec:ter}, Proposition~\ref{app:prop:DT2WA} consequently implies the existence of a polynomial-time algorithm that produces a WA equivalent to a given tree ensemble model used for regression.


\paragraph{Complexity.} The algorithm for constructing a WA equivalent to a given DT $T$ runs in $O(|T|)$, where $|T|$ denotes the number of edges in the DT $T$. The size of the resulting WA is $O(|T|)$ (See \cite{marzouk24a} for more details). Consequently, the overall algorithm for constructing an equivalent regression tree ensemble consisting of $m$ DTs $\{T_{i}\}_{i \in [m]}$ operates in $O(m \cdot \max\limits_{i \in [m]} |T_{i}|)$. The size of the ensemble is also $O(m \cdot \max\limits_{i \in [m]} |T_{i}|)$.


\subsubsection{Claim 4: $\texttt{\text{Lin}}_{\texttt{\text{R}}}$ is polynomially reducible to \texttt{WA}.}

Let $M := \langle \{w_{i,d}\}_{i \in [n], d \in \mathbb{D}_{i}}, b\rangle$ be a linear regression model over the finite set $\mathbb{D} = [m_{1}] \times \ldots [m_{n}]$ where $\{m_{i}\}_{i \in [n]}$, computing the function (as defined earlier):


$$f(x_{1}, \ldots , x_{n}) = \sum\limits_{i=1}^{n} \sum\limits_{d \in \mathbb{D}_{i}} w_{i,d} \cdot I(x_{i}=d) + b$$

For the purposes of our reduction to WAs, we will assume that: $m_{1} = \ldots = m_{n} = m$.
The conversion of a linear regression model $M$ over the input space $\mathbb{D} = [m_{1}] \times \ldots [m_{n}]$ to an equivalent one over $[m]^{n}$ can be done as follows: We set $m$ to be $\max\limits_{i\in [n]} m_{i}$, and the parameters $\{\tilde{w}_{i,d}\}_{i \in [n], d \in [m]}$ for the new model are set as:


$$\tilde{w}_{i,d} = \begin{cases}
    w_{i,d} & \text{if} ~~ d \leq m_{i} \\
    0 & \text{otherwise}
\end{cases}$$

Observe that this procedure operates in $O(\max\limits_{i \in [n]} m_{i})$ time.

\paragraph{Reduction strategy.} For a given linear regression model $M := \langle\{w_{i,d}\}_{i \in [n], d \in [m]},b\rangle$ over $[m]^{n}$, we construct a WA over the alphabet $\Sigma = [m]$ with $n+1$ states. The initial state vector has value $1$ for state $1$, and $0$ for all other states. For $i \in [n]$, the transition from state $i$ to state $i+1$ is assigned a weight of $w_{i,\sigma}$ when processing the symbol $\sigma$ from this state. Figure~\ref{app:fig:lin2wa} illustrates an example of the graphical representation of the resulting WA from a given linear regression tree.


The additive intercept parameter $b$ can be incorporated into this constructed model by simply adding the resulting WA from the procedure described above to a trivial single-state WA that outputs the value $b$ for all binary strings (see Section~\ref{app:sec:ter} for more details on the addition operation between two WAs).

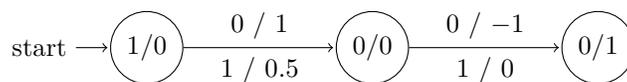
\begin{figure}[h]
    \centering
    \begin{tikzpicture}[shorten >=1pt, node distance=3cm, on grid, auto]

      \node[state, initial] (q_1) {$1/0$};
      \node[state] (q_2) [right=of q_1] {$0/0$};
      \node[state] (q_3) [right=of q_2] {$0/1$};

      \path[->] (q_1) edge node[above] {0 / $1$} 
                         node[below] {1 / $0.5$} (q_2);
      \path[->] (q_2) edge node[above] {0 / $-1$} 
                         node[below] {1 / $0$} (q_3);

    \end{tikzpicture}
    \caption{A construction of a WA equivalent to a linear regression over $2$ binary features with the following weights: $w_{1,0} = 1,~w_{1,1}=0.5,~w_{2,0}=-1,~w_{2,1}=0$. The notation $x/y$ within the state nodes represents the initial weight of the state ($x$) and the final weight of the state ($y$).
    }
    \label{app:fig:lin2wa}
\end{figure}

\subsection{Additional reduction results and implications on SHAP computation}
Some of the complexity results presented in Table~\ref{fig:summaryresults} were not explicitly discussed in the main article. Instead, additional polynomial-time reductions between Hidden Markov Models and other families of distributions are needed to derive them. For completeness, we include in the first part of this subsection proofs of these reduction relationships. Following that, we outline the implications of these relationships on the complexity of certain SHAP configurations shown in Table~\ref{fig:summaryresults}.


\subsubsection{Additional reduction results}

An interesting aspect of Hidden Markov Models is that they provide a unifying probabilistic modeling framework encompassing several families of distributions discussed in the literature on SHAP computation. Specifically, the class \texttt{MARKOV} (and, \textit{de facto}, the family \texttt{IND}, which is trivially a subclass of \texttt{MARKOV}) as well as \texttt{NB} (more formally, $\texttt{NB} \preceq_{P} \ar{\texttt{\text{HMM}}}$) can be polynomially reduced to the class of Hidden Markov Models.


\paragraph{Claim 1: \texttt{NB} is  polynomially reducible to $\overrightarrow{\texttt{HMM}}$.}

Let $M \in \texttt{NB}$ be a model over $n$ binary features with a hidden variable $Y$ that takes values from the discrete set $[m]$. The goal is to construct, in polynomial time, a model $M' \in \ar{\texttt{\text{HMM}}}$ that is equivalent to $M$, i.e.:


$$\forall (x_{1}, \ldots, x_{n}) \in \{0,1\}^{n}:~~P_{M}(x_{1}, \ldots, x_{n}) = P_{M'}(x_{1}, \ldots , x_{n})$$

A key observation underlying the reduction is that a naive Bayes Model can be viewed as a specific instance of a model in the family $\ar{\texttt{\text{HMM}}}$, where the state space coincides with the domain of the hidden variable in the naive Bayes model. The distinct feature of this model is that it remains in the same state throughout the entire generation process, with this state being determined at the initialization phase according to the probability distribution.


Formally, let $M = \langle\pi, \{P_{i}\}_{i \in [n]}\rangle$ be a naive bayes model over $n$ observed RVs such that the domain value of its hidden state variable $Y$ is equal to $[m]$.  The construction of a model $M' \in \ar{\texttt{\text{HMM}}} $ is given as follows:
\begin{itemize}
    \item \textbf{The permutation:} The identity permutation,
    \item \textbf{The state space:} $[m]$.
    \item \textbf{The state initialization:} $M'$ starts generating from the state $i \in [m]$ with probability equal to $\pi[m]$.
    \item \textbf{The transition dynamics:} $M'$ transitions from a given state $j$ to the same state $j$ with probability equal to $1$.
    \item \textbf{The symbol emission:} At position $i$, the model $M'$ emits the symbol $\sigma \in \{0,1\}$ from the state $j$ with probability equal to $P_{i}$. 
\end{itemize}

\paragraph{Claim 2: \texttt{MARKOV} (resp. $\ar{\texttt{\text{MARKOV}}}$) is polynomially reducible to \texttt{HMM} (resp. $\ar{\texttt{\text{HMM}}}$).} The class \texttt{HMM} (resp. $\ar{\texttt{\text{HMM}}}$) can be viewed as a generalization of the class \texttt{MARKOV} (resp. $\ar{\texttt{\text{MARKOV}}}$) for handling partially observable stochastic processes: Markovian distributions represent fully observable processes, whereas HMMs represent partially observable ones. Converting a Markovian distribution into a distribution modeled by an HMM involves encoding the Markovian dynamics into the hidden state dynamics of the HMM. The HMM then trivially emits the corresponding symbol of the reached hidden state at each position.


For completeness, the following provides a formal description of how \texttt{MARKOV} is polynomially reducible to \texttt{HMM}. The derivation of a polynomial-time reduction from $\ar{\texttt{\text{MARKOV}}}$ to $\ar{\texttt{\text{HMM}}}$ can be achieved using a similar construction. Let $M := \langle \pi, T\rangle $ be a model in \texttt{MARKOV} over the alphabet $\Sigma$. The description of a model $M' \in \texttt{\text{HMM}}$, equivalent to $M$, is obtained as follows:


\begin{itemize}
    \item \textbf{The state space:} The alphabet $\Sigma$
    \item \textbf{The state initialization:} $M'$ starts by generating a state $\sigma \in \Sigma$ with probability equal to $\pi(\sigma)$.
    \item \textbf{The transition dynamics:} From a hidden state $\sigma \in \Sigma$, the probability of transitioning to the state $\sigma' \in \Sigma$ is equal to $T[\sigma, \sigma']$. 
    \item \textbf{The symbol emission:} For any symbol $\sigma \in \Sigma$, $M'$ emits a symbol $\sigma \in \Sigma$ from the hidden state $\sigma$ with probability equal to $1$.
\end{itemize}

\subsubsection{Corollaries on SHAP computational problems} 

The reduction relationships among independent distributions, Markovian distributions, and those modeled by HMMs (i.e., \texttt{IND} $\preceq_{P}$ \texttt{MARKOV} $\preceq_{P}$ \texttt{HMM} and $\ar{\texttt{IND}}$ $\preceq_{P}$ $\ar{\texttt{MARKOV}}$ $\preceq_{P}$ $\ar{\texttt{HMM}}$), as well as the relationship between empirical and HMM-modeled distributions (i.e., \texttt{EMP} $\preceq_{P}$ $\ar{\texttt{HMM}}$ $\preceq_{P}$ \texttt{HMM}), naturally leads to a corollary regarding the relationships between different SHAP variants:

\begin{corollary}
    For any family of sequential (resp. non-sequential) models $\mathbb{M}$ (resp. $\ar{\mathbb{M}}$), $\mathbb{S} = \{ \texttt{\emph{LOC}}, \texttt{\emph{GLO}} \}$, we have:
    \begin{equation}\label{app:eq:markov2hmm1}
        \mathbb{S}-\mathbb{V}-\texttt{\emph{SHAP}}(\mathbb{M}, \texttt{\emph{IND}}) \preceq_{P} \mathbb{S}-\mathbb{V}-\texttt{\emph{SHAP}}(\mathbb{M}, \texttt{\emph{MARKOV}}) \preceq_{P} \mathbb{S}-\mathbb{V}-\texttt{\emph{SHAP}}(\mathbb{M}, \texttt{\emph{HMM}})
    \end{equation}
    \begin{equation}\label{app:eq:markov2hmm2}
        \mathbb{S}-\mathbb{V}-\texttt{\emph{SHAP}}(\ar{\mathbb{M}},
    \ar{\texttt{\emph{IND}}}) \preceq_{P} \mathbb{S}-\mathbb{V}-\texttt{\emph{SHAP}}(\ar{\mathbb{M}}, \ar{\texttt{\emph{MARKOV}}}) \preceq_{P} \mathbb{S}-\mathbb{V}-\texttt{\emph{SHAP}}(\ar{\mathbb{M}}, \ar{\texttt{\emph{HMM}}})
    \end{equation}
    \begin{equation}\label{app:eq:markov2hmm1}
        \mathbb{S}-\mathbb{V}-\texttt{\emph{SHAP}}(\mathbb{M}, \texttt{\emph{EMP}}) \preceq_{P} \mathbb{S}-\mathbb{V}-\texttt{\emph{SHAP}}(\mathbb{M}, \texttt{\emph{HMM}}) \ \ ; \      \    \mathbb{S}-\mathbb{V}-\texttt{\emph{SHAP}}(\ar{\mathbb{M}}, \texttt{\emph{EMP}}) \preceq_{P} \mathbb{S}-\mathbb{V}-\texttt{\emph{SHAP}}(\ar{\mathbb{M}}, \ar{\texttt{\emph{HMM}}})
    \end{equation}
\end{corollary}

This corollary, along with previous complexity results for computing conditional SHAP values~\citep{vander21, huangupdates, arenas23}, and the fact that conditional and interventional SHAP variants coincide under independent distributions~\citep{sundararajan20b}, presents a wide range of complexity results on SHAP computational problems that can be derived as corollaries:

\begin{corollary} All of the following complexity results hold:
\begin{enumerate} 
    \item For the class of models $\{\texttt{\emph{DT}}, \texttt{\emph{ENS-DT}}_{\texttt{\emph{R}}}, \texttt{\emph{LIN}}_{\texttt{\emph{R}}}\}$: The computational problems $\mathbb{S}-\mathbb{V}-\texttt{\emph{SHAP}}(\ar{\mathbb{M}},
    \ar{\texttt{\emph{IND}}})$ for  $\mathbb{V}\in\texttt{\emph{\{B,V\}}}$ can be solved in polynomial time. The computational problem $\texttt{\emph{GLO}}-\texttt{\emph{C}}-\texttt{\emph{SHAP}}(\ar{\mathbb{M}}, 
    \ar{\texttt{\emph{IND}}})$ can be solved in polynomial time.
    The computational problem $\texttt{\emph{GLO}}-\texttt{\emph{C}}-\texttt{\emph{SHAP}}(\ar{\mathbb{M}}, 
    \ar{\texttt{\emph{HMM}}})$ is $\#$P-Hard.
    
    \item For \texttt{\emph{WA}}:  The computational problems         $\mathbb{S}-\mathbb{V}-$\texttt{\emph{SHAP(WA,}}$\mathbb{P}$\texttt{\emph{)}} for $\mathbb{S}\in\texttt{\emph{\{LOC,GLO\}}}$, $\mathbb{P}\in\texttt{\emph{\{IND,EMP\}}}$, and $\mathbb{V}\in\texttt{\emph{\{B,V\}}}$ can be solved in polynomial time. The computational problem $\texttt{\emph{GLO}}-\texttt{\emph{C}}-\texttt{\emph{SHAP}}(\texttt{\emph{WA}}, 
    \texttt{\emph{IND}})$ can be solved in polynomial time. The computational problem $\texttt{\emph{LOC}}-\texttt{\emph{C}}-\texttt{\emph{SHAP}}(\texttt{\emph{WA}}, 
    \texttt{\emph{EMP}})$ is NP-Hard. The computational problem $\texttt{\emph{LOC}}-\texttt{\emph{C}}-\texttt{\emph{SHAP}}(\texttt{\emph{WA}}, 
    \texttt{\emph{HMM}})$ is $\#$P-Hard.
    \item For $\texttt{\emph{ENS-DT}}_{\texttt{\emph{C}}}$: The computational problem $\texttt{\emph{LOC}}-\texttt{\emph{C}}-\texttt{\emph{SHAP}}\texttt{\emph{(ENS-DT}}_{\texttt{\emph{C}}}, \texttt{\emph{EMP}}$\texttt{\emph{)}} is NP-Hard. The computational problem $\texttt{\emph{LOC}}-\texttt{\emph{C}}-\texttt{\emph{SHAP}}\texttt{\emph{(ENS-DT}}_{\texttt{\emph{C}}},\ar{\texttt{\emph{HMM}}}$\texttt{\emph{)}} is $\#$P-Hard.

    \item For \texttt{\emph{NN-SIGMOID}}: The computational problem $\texttt{\emph{LOC}}-\texttt{\emph{C}}-\texttt{\emph{SHAP}}\texttt{\emph{(NN-SIGMOID,}}$$\ar{\texttt{\emph{HMM}}}$\texttt{\emph{)}} is NP-Hard.

    \item For \texttt{\emph{RNN-ReLU}}: The computational problem $\texttt{\emph{LOC}}-\texttt{\emph{C}}-\texttt{\emph{SHAP}}\texttt{\emph{(RNN-ReLU,}}$$\ar{\texttt{\emph{IND}}}$\texttt{\emph{)}} is NP-Hard.

\end{enumerate}
\end{corollary}

\begin{proof}

We will explain the results of each part of the corollaries separately: \begin{enumerate}
 \item For the class of models $\{\texttt{DT}, \texttt{ENS-DT}_{\texttt{R}}, \texttt{LIN}_{\texttt{R}}\}$: The complexity results for these model families under independent distributions, for both baseline and interventional SHAP, are derived from our findings on the tractability of these models under HMM-modeled distributions and our proof that \texttt{EMP} $\preceq_{P}$ $\ar{\texttt{HMM}}$. The complexity results for the global and conditional forms under independent distributions stem from our tractability results for interventional SHAP in this setup, as well as the fact that interventional and conditional SHAP coincide under independent distributions~\citep{sundararajan20b}. Lastly, the $\#$P-Hardness of the conditional variant under HMM-modeled distributions follows from the hardness of generating this form of explanation under Naive Bayes modeled distributions~\citep{vander21}, and our proof that $\ar{\texttt{NB}}$ $\preceq_{P}$ $\ar{\texttt{HMM}}$.
\item For \texttt{WA}, the tractability results for local and global baseline, as well as interventional SHAP under independent and empirical distributions, follow from our primary complexity findings for this family of models, which demonstrate tractability over HMM-modeled distributions, and our proof that both \texttt{EMP} $\preceq_{P}$ \texttt{HMM} and \texttt{IND} $\preceq_{P}$ \texttt{HMM}. The tractability of global interventional SHAP under independent distributions also follows from these results, along with the fact that interventional and conditional SHAP coincide under independent distributions~\citep{sundararajan20b}. The NP-hardness of conditional SHAP under empirical distributions is derived from the complexity of this setting for decision tree classifiers~\citep{vander21} and our proof that \texttt{DT} $\preceq_{P}$ \texttt{WA}. Finally, the \#P-hardness of conditional SHAP under HMM-modeled distributions results from the hardness of this setting for decision trees, as discussed earlier, and our proof that \texttt{DT} $\preceq_{P}$ \texttt{WA}.

\item For $\texttt{ENS-DT}_{\texttt{C}}$: The complexity results for local conditional SHAP under empirical distributions are derived from the hardness results provided for decision trees under Naive Bayes-modeled distributions~\citep{vander21}, along with the facts that \texttt{DT} $\preceq_{P}$ $\texttt{ENS-DT}_{\texttt{C}}$ and $\ar{\texttt{NB}}$ $\preceq_{P}$ $\texttt{EMP}$. The result on the complexity of conditional SHAP under hidden Markov distributions follows as a corollary of the results in~\cite{huangupdates} regarding the hardness of computing conditional SHAP under independent distributions, combined with our result showing that $\texttt{IND}$ $\preceq_{P}$ $\texttt{EMP}$.

\item For \texttt{NN-SIGMOID}: The complexity results for conditional SHAP under HMM distributions follow as a corollary of the hardness results presented in~\citep{vander21} for computing conditional SHAP values for sigmoidal neural networks under independent distributions, along with our proof that $\texttt{IND}$ $\preceq_{P}$ $\ar{\texttt{HMM}}$.
\item For \texttt{RNN-ReLU}: The complexity results for conditional SHAP under independent distributions stem from our main complexity result for this family of models, which established the hardness for interventional SHAP. Additionally, interventional and conditional SHAP coincide under independent distributions~\citep{sundararajan20b}.

\end{enumerate}

\end{proof}

\section{On the NP-Hardness of computing local Interventional SHAP for RNN-ReLus under independent distributions} 
\label{app:ISHAPRNN}
The objective of this segment is to prove Lemma \ref{lemma:emptyrnnrelu} from section \ref{subsec:rnnreluhard} of the main paper. This lemma played a crucial role in proving the primary result of this section, which demonstrates the intractability of the problem $\texttt{LOC-I-SHAP}(\texttt{RNN-ReLu}, \texttt{IND})$.  The lemma is stated as follows:

\begin{unumberedlemma}
    The problem \texttt{\emph{EMPTY}}\texttt{\emph{(RNN-ReLu)}} is NP-Hard. 
\end{unumberedlemma}

The proof is performed by reduction from the closest string problem (\texttt{CSP}). Formally, the \texttt{CSP} problem is given as follows:

$\bullet$ \textbf{Problem:} \texttt{CSP} \\
    \textbf{Instance:} A collection of strings $S = \{w_{i}\}_{i \in [m]}$, whose length is equal to $n$, and an integer $k > 0$. \\ 
      \textbf{Output:} Does there exist a string $w' \in \Sigma^{n}$, such that for any $w_{i} \in S$, we have $d_{H}(w_{i}, w') \leq k$?  \\ 
        where $d_{H}(.,.)$ is the Hamming distance given as: $d_{H}(w,w') := \sum\limits_{i=1}^{ [|w|]} \mathrm{1}_{w_{j}}(w'_{j})$. \\
     
The \texttt{CSP} problem is known to be NP-Hard \citep{li2000closest}. Our reduction approach involves constructing, in polynomial time, an RNN that accepts only closest string solutions for a given input instance $(S,k)$ of the \texttt{CSP} problem. Consequently, the \texttt{CSP} produces a \textit{Yes} answer for the input instance if and only if the constructed RNN-ReLU is empty on the support $\Sigma^{n}$, which directly leads to the result of Lemma~\ref{lemma:emptyrnnrelu}. We divide the proof of this reduction strategy into two parts:

     \begin{itemize}
         \item Given an arbitrary string $w$ and an integer $k > 0$, we present a construction of an RNN-ReLu that simulates the computation of the Hamming distance for $w$. The outcome, indicating whether the Hamming distance exceeds the threshold $k$, is encoded in the activation of a specific neuron within the constructed RNN-ReLu. This procedure will be referred to as $\texttt{CONSTRUCT}(w,k)$.
         \item Concatenate the RNN-ReLUs generated by the procedure $\texttt{CONSTRUCT}(.,.)$ into a single unified RNN cell. Then, assign an appropriate output matrix that accomplishes the desired objective of the construction.
     \end{itemize}
     
     \paragraph{The procedure \texttt{CONSTRUCT}($w,k$).}  Let $w \in \{0,1\}^{n}$ be a fixed reference string and $k > 0$ be an integer. The procedure $\texttt{CONSTRUCT}(w,k)$ returns an RNN cell of dimension $|w| + 1$ that satisfies the properties discussed earlier. The parametrization of $\texttt{CONSTRUCT}(w,k)$ is defined as follows:

\begin{itemize}
\item \textbf{The initial state vector.} $h_{init} = [1, 0 \ldots, 0 ]^{T} \in \mathbb{R}^{n+1}$.
 \item \textbf{The transition matrix.} he transition matrix is dependent on $k$ and is expressed as:
\begin{equation}\label{transition}
W_{k} = \begin{pmatrix}
   0 & 0 & .& . & . & 0 & 0\\
   1 & 0 & . & . & . & 0  & 0\\ 
   0 & 1 & . & . & . & 0  & 0\\ 
    . & . & . & . & . & .  & .\\
   . & . & . & . & . & .  & . \\
   0 & 0 & .& . & . & 0 & -k \\ 
   0 & 0 & .. & . & . & 0 &  1 
\end{pmatrix} \in \mathbb{R}^{(n+1) \times (n+1)}
\end{equation}
\item \textbf{The embedding vectors.} Embedding vectors are dependent on the reference string $w$. To prevent any confusion, we will use the superscript $w$ for indexing. The embedding vectors are constructed as follows: \\ For a symbol $\sigma \in \{0,1\}$  and $l \in [n]$, we define $v_{\sigma}^{w_{i}}[l] = 1 $ if $w_{l} \neq \sigma$. All other elements of $v_{\sigma}^{w}$ are set to $0$.
\end{itemize}
The following proposition formally demonstrates that this construction possesses the desired property:
\begin{proposition}\label{app:prop:reluproperties}
Let $w \in \{0,1\}^{n}$ be an arbitrary string, and $k > 0$ be an arbitrary integer. The procedure $\emph{\texttt{CONSTRUCT}}(w,k)$ outputs an RNN cell $\langle h_{init}, W_{k}, \{v_{\sigma}^{(w)}\}_{\sigma \in \{0,1\}}\rangle$, which satisfies the following properties: 
\begin{enumerate}
    \item For any string $w' \in \{0,1\}^{s}$, where $ s < n$, we have that $h_{w}[s] = d_{H}(w,w')$,
    \item For a string $w' \in \{0,1\}^{n}$, it holds that: 
\begin{equation}\label{eq:levsimulate}        
h_{w'}[n] = \emph{\text{ReLu}}(d_{H}(w, w') - k)
    \end{equation}
\end{enumerate}
\end{proposition}

\begin{proof} We individually prove the two defined requirements as follows:
\begin{enumerate}
    \item The proof proceeds by induction on $s$. For the base case, when $s = 1$, we observe that for any symbol $\sigma \in \Sigma$, it holds that $h_{\sigma} = v_{\sigma}^{(w)}$. By construction, $v_{\sigma}^{(w)}[1] = \mathrm{1}_{w_{1}}(\sigma) = d_{H}(\sigma,w_{1})$.
    Assume the proposition holds for $s < n-2$. We now prove it for $s+1$. Let $w' \in \Sigma^{s}$ and $\sigma \in \Sigma$. Then, we have
 
    \begin{align*}
    h_{w' \sigma}[s+1] &= \text{ReLu}(W[:,s+1]^{T} \cdot h_{w'} + v_{\sigma}[s+1]) \\
    &= \text{ReLu}(h_{w'}[s] + v_{\sigma}[s+1]) \\
    &= d_{H}(w_{1:s}, w') + \mathrm{1}_{w_{s+1}}( \sigma) \\
    &= d_{H}(w'\sigma, w_{1:s+1})
    \end{align*}
 
    \item Let $w" = w'\sigma$ be a string of length $n$. From the first part of the proposition, we have 
    $h_{w'}[n-1] = d(w', w_{1:(n-1)})$. Consequently,
    \begin{align*} 
    h_{w"}[n] &= \text{ReLu}(W[:,n]^{T} \cdot h_{w'} + v_{\sigma}[n]) \\
    &= \text{ReLu}(h_{w'}[n-1]  - k + \mathrm{1}_{w_{n} } (\sigma)) \\ 
    &= \text{ReLu}(d_{H}(w, w^{i}) - k)
    \end{align*}
    \end{enumerate}
\end{proof}

The key property emphasized by Proposition~\ref{app:prop:reluproperties} is stated in Equation~\ref{eq:levsimulate}. It demonstrates that the $n$-th neuron of the constructed RNN cell encodes the Hamming distance between the reference string $w$ and the input string $w'$. Notably, the activation value of this neuron is $0$ if and only if $d_{H}(w,w') \leq k$; otherwise, it is at least $1$.


\paragraph{Concatenation and Output Matrix instantiation.} 

Let $(S, k)$ be an instance of the closest string problem. The final RNN cell will be formed by concatenating the RNN cells $\langle h_{init}, W_{k}, \{v_{\sigma}^{w^{(i)}}\}_{i \in [|S|]}\rangle$, where the set $\{w^{(i)}\}_{i \in |S|}$ consists of elements from $S$. The concatenation of two RNN cells, such as $\langle h_{1}, W_{1}, v_{\sigma}^{(1)}\rangle$ and $\langle h_{1}, W_{2}, v_{\sigma}^{(2)}\rangle$, produces a new cell defined as
$\Big\langle \begin{pmatrix}
    h_{1} \\h_{2}
\end{pmatrix}, \begin{pmatrix}
    W_{1} & 0 \\ 
    0 & W_{2}
\end{pmatrix}$, $\{ \begin{pmatrix} v_{\sigma}^{1} \\ 
v_{\sigma}^{2}
\end{pmatrix}
\}_{\sigma \in \Sigma}\Big\rangle$. \footnote{Although the concatenation operation is non-commutative, the order according to which we perform the concatenation operator does not affect the result of the reduction.} \\

We concatenate all RNN-ReLUs produced by the \texttt{CONSTRUCT}(.,.) procedure on the instances $\{(w,k)\}_{w \in S}$. The output matrix of the resulting RNN-ReLU, denoted as $O \in \mathbb{R}^{(n+1) \cdot |S|}$, is selected such that for any set of vectors $\{h_{i}\}_{i \in [|S|]}$ in $\mathbb{R}^{n+1}$, the following holds:

$$O^{T} \cdot \begin{bmatrix}
    h_{1} & \ldots & h_{s}
\end{bmatrix} = \sum\limits_{i \in [|S|]} - h_{i}[n] + \frac{1}{2} \cdot h_{1}[n+1]$$

Under this setting, it is important to note that, given the properties of the RNN-ReLu cell produced by the procedure $\texttt{CONSTRUCT}(.,.)$ (Proposition \ref{app:prop:reluproperties}), and the fact that the activation value of the $(n+1)$-neuron is always equal to 1 by design, the output of the constructed RNN-ReLu on an input sequence $w'$ is 1 if and only if for all $w \in S$, we have $d_{H}(w,w') \leq k$. As a result, the constructed RNN-ReLu is empty on the support $n$ if and only if there is no string $w'$ such that $d_{H}(w,w') \leq k$ for all $w$.

\section{The problem \texttt{LOC-B-SHAP}(\texttt{NN-SIGMOID}), \texttt{LOC-B-SHAP}(\texttt{RNN-ReLu}) and \texttt{LOC-B-SHAP}($\texttt{ENS-DT}_{\texttt{C}}$) are Hard: Proofs of Intermediary Results} \label{app:BSHAPSIGMOID}

This section of the appendix is devoted to presenting the proofs of the mathematical results discussed in subsection~\ref{subsec:baseline} regarding the local Baseline SHAP problem for different model families. The structure is as follows:

\begin{enumerate}
    \item 
The first subsection (Subsection~\ref{app:subsec:locbshap}) presents the proof of the intractability of the \texttt{LOC-B-SHAP}(\texttt{NN-SIGMOID}) problem, established through a reduction from the dummy player problem in WMGs.
    \item The second subsection (Subsection~\ref{app:subsec:bshaprf}) presents the proof of the NP-hardness of computing Local B-SHAP for the tree ensemble classifiers, achieved via a reduction from the 3SAT problem.
\end{enumerate}

\subsection{Reducing the dummy player problem of WMGs to \texttt{LOC-B-SHAP}(\texttt{NN-SIGMOID}) and \texttt{LOC-B-SHAP}(\texttt{RNN-ReLu})} \label{app:subsec:locbshap}

The purpose of this segment is to establish Proposition~\ref{prop:reductionsigmoid} from the main paper. Let us first restate the proposition:

\begin{unumberedproposition} There exist two polynomial-time algorithms, which are defined as follows:
 \begin{enumerate}
        \item For \texttt{\emph{NN-SIGMOID}}, there exists a polynomial-time algorithm that takes as input a WMG $G$ and a player $i$ in $G$ and returns a sigmoidal neural network $f_{G}$ over $\{0,1\}^{N}$, $(x,x^{ref}) \in \{0,1\}^{N}$ and $\epsilon \in \mathbb{R}$ such that:
        $$\text{The player i is not dummy} \iff \phi_{b}(f_{G},i,x,x^{ref}) > \epsilon$$
        \item For \texttt{\emph{RNN-ReLu}}, there exists a polynomial-time algorithm that takes as input a WMG $G$ and a player $i$ in $G$ and returns a sigmoidal neural network $f_{G}$ over $\{0,1\}^{N}$, $(x,x^{ref}) \in \{0,1\}^{N}$ and $\epsilon \in \mathbb{R}$ such that:
         $$\text{The player i is not dummy} \iff \phi_{b}(f_{G},i,x,x^{ref}) > 0$$
    \end{enumerate}
\end{unumberedproposition}

We divide the remainder of this segment into two parts. The first part focuses on the family of sigmoidal neural networks (\texttt{NN-SIGMOID}), and the second part addresses the class of RNN-ReLus (\texttt{RNN-ReLu}).

\subsubsection{The case of \texttt{LOC-B-SHAP}(\texttt{NN-SIGMOID})}
The intractability of the problem \texttt{LOC-B-SHAP}(\texttt{NN-SIGMOID}) is obtained by reduction from the dummy player problem of WMGs. The remainder of this segment will be dedicated to prove the following:
\begin{proposition} \label{app:prop:bshapsigmoid}
    There exists a polynomial-time algorithm that takes as input a WMG $G = \langle N, \{n_{j}\}_{j \in [N]}, q \rangle$, and a player $i$ in $G$, and returns a sigmoidal neural network $f_{G}$ over $\mathbb{R}^{N}$, two vectors $(x, x^{ref}) \in \mathbb{R}^{N} \times \mathbb{R}^{N}$, and a scalar $\epsilon > 0$ such that:
    \begin{equation*}
     \text{The player i is dummy} \iff \texttt{\emph{LOC-B-SHAP}}(f_{G},x,i,x^{ref}) \leq \epsilon   
    \end{equation*}
     
\end{proposition}

\begin{algorithm}
\caption{Reduction of the \texttt{DUMMY} problem to \texttt{LOC-B-SHAP}(\texttt{NN-SIGMOID})}
\label{app:alg:WMG2b-SHAP}
\begin{algorithmic}[1]
\REQUIRE A WMG $G = \langle N, \{n_{j}\}_{j \in [N]},q\rangle$, $i \in [N]$
\ENSURE A Sigmoidal Neural Network $\sigma$ over $\mathbb{R}^{N}$, an integer $i \in [N]$, $(x,x^{ref}) \in \mathbb{R}^{2}$, and $\epsilon \geq 0$ 
\STATE $x \leftarrow [1 , \ldots , 1]$
\STATE $x^{ref} \leftarrow [0, \ldots , 0]$
\STATE $C_{N} \leftarrow N \cdot \binom{N-1}{\lfloor \frac{N-1}{2} \rfloor}!$
\STATE $\epsilon \leftarrow \frac{1}{1 + C_{N}}$
\STATE Construct the Sigmoidal Neural Network $f$, parameterized by:
$$f_{G}(x) = \sigma(2 \cdot \log(\frac{1-\epsilon}{\epsilon}) \cdot (x - q + \frac{1}{2}))$$
\RETURN $\langle f_{G},~i,~x,~x^{(ref)},~\epsilon\rangle$
\end{algorithmic}
\end{algorithm}

The pseudo-code of the (polynomial-time) algorithm whose existence is implicitly stated in proposition \ref{app:prop:bshapsigmoid} is provided in Algorithm \ref{app:alg:WMG2b-SHAP}. In the following, we provide the proof of the correctness of this reduction. In the remainder of this segment, we fix an input instance $I = \langle G,i\rangle $ of the \texttt{DUMMY} problem where $G = \langle N, \{n_{j}\}_{j \in [N]}, q\rangle$ is a WMG, and $i \in [N]$ is a player in $G$. And, we use the notation $\langle f_{G}, i, x,x^{ref}, \epsilon\rangle$ to represent the output of Algorithm~\ref{app:alg:WMG2b-SHAP}. To ease exposition, we also introduce the following parameter which depends on the number of players $N$ (appearing in Line 3 of Algorithm \ref{app:alg:WMG2b-SHAP}): 
$$C_{N} \myeq N \cdot \binom{\lfloor \frac{N-1}{2} \rfloor}{N-1}!$$

Note that by setting the instance to explain $x$ to $\begin{pmatrix} 1 \\ \ldots \\ 1\end{pmatrix}$, and the reference instance $x^{ref}$ to $\begin{pmatrix}
    0 \\ \ldots \\ 0
\end{pmatrix}$, the constructed function $f_{G}$ (Line 5 of Algorithm \ref{app:alg:WMG2b-SHAP}) computes the following quantity for any coalition of players $S \subseteq [N] \setminus \{i\}$:

$$f_{G}(x_{S};x_{\bar{S}}^{ref}) = \sigma \left(2 \log(\frac{1 - \epsilon}{\epsilon})\cdot \log(N) \cdot (\sum\limits_{j \in S} n_{j} - q + \frac{1}{2})\right)$$

where $\mathbf{x}_{S} = (x_{S}; x_{\bar{S}}^{ref}) \in \{0,1\}^{n}$ is such that $x_{i} = 1$ if $i \in S$, $0$ otherwise.

The correctness of the reduction proposed in  Algorithm~\ref{app:alg:WMG2b-SHAP} is a result of the following two claims:

\begin{claim} \label{app:claim:one} 
If the player $i$ is dummy then for any $S \subseteq [N] \setminus \{i\}$ it holds that: 
    \begin{equation} \label{app:eq:claim1}
    f_{G}(x_{S \cup \{i\}};x_{\bar{S \cup \{i\}}}^{ref}) - f_{G}(x_{S}; x_{\bar{S}^{ref}}) \leq \epsilon
    \end{equation}
    \end{claim}

  \begin{claim} \label{app:claim:two}
  If the player $i$ is not dummy then there must exist some $S_{d} \subseteq [N] \setminus \{i\}$, such that: 
  \begin{equation} \label{app:eq:claim2}
    f_{G}(x_{S_{d} \cup \{i\}};x_{\bar{S_{d} \cup \{i\}}}^{ref}) - f_{G}(x_{S_{d}}; x_{\bar{S_{d}}}) > 1 - \epsilon
    \end{equation}
 \end{claim}
The result of proposition \ref{app:prop:bshapsigmoid} is a corollary of claims \ref{app:claim:one} and \ref{app:claim:two}. Before proving these two claims, we will first prove that, provided claims \ref{app:claim:one} and  \ref{app:claim:two} are true, then the proposition \ref{app:prop:bshapsigmoid} holds.  To prove this, we will consider two separate cases:
\begin{itemize}
    \item \textbf{Case 1 (The player $i$ is dummy):} In this case, we show that if the player $i$ is dummy and claim \ref{app:claim:one} holds, then: 
      $$\phi_{b}(f_{G}, i , x, x^{ref}) \leq \epsilon$$. 
    
Assume claim \ref{app:claim:one} holds, we have:
    \begin{align*} 
    \phi_b(f_{G},i,x,x^{ref}) &= \sum_{S \subseteq [n] \setminus \{i\}} \frac{|S|!\cdot (N - |S| - 1)!}{N!} \cdot \left[ f_{G}(x_{S \cup \{i\}};x_{\bar{S \cup \{i\}}}^{ref}) - f_{G}(x_{S}; x_{\bar{S}^{ref}})  \right] \\
    & \leq \sum_{S \subseteq [n] \setminus \{i\}} \frac{|S|!\cdot (N - |S| - 1)!}{N!} \cdot \epsilon \\
    & = \epsilon
    \end{align*}

\item \textbf{Case 2 (The player $i$ is not dummy):} In this case, we show that if the player $i$ is not dummy and Claim \ref{app:claim:two} holds, then: 
$$\phi_{b}(f_{G},i,x,x^{ref}) > \epsilon$$

Assume claim \ref{app:claim:two} holds, we have:
\begin{align*}
    \phi_b(f_{G}, i , x, x^{ref}) &= \sum_{S \subseteq [n] \setminus \{i\}} \frac{|S|!\cdot (n - |S| - 1)!}{n!} \cdot \left[ f_{G}(x_{S \cup \{i\}};x_{\bar{S \cup \{i\}}}^{ref}) - f_{G}(x_{S}; x_{\bar{S}^{ref}})  \right]  \\
    &> \frac{|S_{d}|! \cdot (N - |S_{d}|-1)!}{N!} \cdot (1 - \epsilon) \\ 
    & = \frac{1}{N \cdot  \binom{|S_{d}|}{N-1}!}  \cdot (1 - \epsilon) \\
    &> \frac{1}{N \cdot \binom{\lfloor \frac{N-1}{2} \rfloor}{N-1}!} \cdot (1 - \epsilon) \ \\ 
    &> \frac{1}{C_{N}}  \cdot (1 - \epsilon) = \frac{1}{C_{N}} \cdot (1 - \frac{1}{1 + C_{N}}) = \frac{1}{1+C_{N}} = \epsilon
\end{align*}
where the third inequality follows from the fact that for any $N \in \mathbb{N}$, $k \in [N]$, we have $\binom{N}{k} ! \leq \binom{N}{\lfloor \frac{N}{2}! \rfloor}$.
\end{itemize}


The above argument indicates that proving claims \ref{app:claim:one} and \ref{app:claim:two} is sufficient to establish the desired result in this section (Proposition~\ref{app:prop:bshapsigmoid}). What remains is to show that claims \ref{app:claim:one} and \ref{app:claim:two} indeed hold. 

The following simple technical lemma leverages some properties of the sigmoidal function to prove that the constructed sigmoidal function $f_{G}$ in ALgorithm \ref{app:alg:WMG2b-SHAP} verifies the desired properties of both these claims: 

\begin{lemma} \label{app:lemma:technicallemma}
 The constructed function $f_{G}$ in Algorithm \ref{app:alg:WMG2b-SHAP}
  satisfies the following conditions:
  \begin{enumerate}
      \item If $x \leq q-1$, we have: $f_G(x; N,q) \leq \epsilon$
      \item If $x > q$, we have: $f_G(x;N,q) \geq 1 - \epsilon$
  \end{enumerate}
\end{lemma}
\begin{proof}
   By the property of monotonicity of the sigmoidal function and its symmertry around $0$ (thus, assuming $q = 0$ w.l.o.g), it's sufficient to prove that for $x = -1$, we have that: 
   $f_{G}(-1;N,0) \leq \epsilon$. By simple calculation of $f_{G}(-1,N,0)$ using the parametrization of the function $f_{G}$, one can obtain the result.   
\end{proof}

Now we are ready to prove claims \ref{app:claim:one} and \ref{app:claim:two}.

\paragraph{Proof of Claim \ref{app:claim:one}.}
Assume that player $i$ is a dummy. Fix an arbitrary coalition $S \subseteq [N] \setminus \{i\}$. We now need to demonstrate that condition~\eqref{app:eq:claim1} holds for $S$. Since player $i$ is a dummy, there are two possible cases: either both $S$ and $S \cup \{i\}$ are winning, or neither of them are.


$\bullet$ \textbf{Case 1 (The winning case: $v_{G}(S) = 1$ and $v_{G}(S \cup \{i\}) = 1$):} In this case, we have both $\sum\limits_{j \in S} n_{j} \geq q$ and $\sum\limits_{j \in S} n_{j} + n_{i} \geq q$. Consequently, by Lemma \ref{app:lemma:technicallemma}, both
$f_G(\sum\limits_{j \in S \cup \{i\}} n_{j}; N,q)$ and $f_G(\sum\limits_{j \in S} n_{j}; N,q)$ lie 
in the interval $(1 - \epsilon,1)$. Thus, we have 
 that:
$$f_{G}(x_{S \cup \{i\}};x_{\bar{S \cup \{i\}}}^{ref}) - f_{G}(x_{S}; x_{\bar{S}}^{ref}) \leq 1 - (1 - \epsilon) = \epsilon$$

$\bullet$ \textbf{Case 2 (The non-winning case: $v_{G}(S) = 0$ and $v_{G}(S \cup \{i\}) = 0$):} The proof for this case mimicks the one of the former case. In this case, we have both $\sum\limits_{j \in S} n_{j} \leq q-1$ and $\sum\limits_{j \in S} n_{j} + n_{i} \leq q-1$. Consequently, by Lemma \ref{app:lemma:technicallemma}, both
$g(\sum\limits_{j \in S \cup \{i\}} n_{j}; N,q)$ and $g(\sum\limits_{j \in S} n_{j}; N,q)$ lies 
in the interval $[0, \epsilon)$. Thus, we have:
$$f_{G}(x_{S \cup \{i\}};x_{\bar{S \cup \{i\}}}^{ref}) - f_{G}(x_{S}; x_{\bar{S}^{ref}}) \leq \epsilon$$

\paragraph{Proof of Claim \ref{app:claim:two}.} Assume that the player $i$ is not dummy. Then, there must exist a coalition $S_{d} \subset [N] \setminus \{i\}$ such that $\sum\limits_{j \in S \cup \{i\}} n_{j} \geq q$ (a winning coalition), and $\sum\limits_{j \in S } n_{j} \leq q-1$ (A losing coalition). By proposition \ref{app:lemma:technicallemma}, we have then: $f_{G}(x_{S \cup \{i\}}; x_{\bar{S_{d} \cup \{i\}}}) \in (1 - \epsilon,1)$ and $f_{G}(x_{S_{d}}; x_{\bar{S_{d}}}) \in (0,\epsilon)$. Consequently, we have:
   $$f_{G}(x_{S \cup \{i\}}; x_{\bar{S_{d} \cup \{i\}}}) - f_{G}(x_{S_{d}}; x_{\bar{S_{d}}}) \geq 1 - \epsilon$$
\subsubsection{The case \texttt{LOC-B-SHAP}(\texttt{RNN-ReLu})} \label{app;subsec:bshaprnnrelu}
In this subsection, we shall provide the details of the construction of a polynomial-time algorithm that takes as input an instance $\langle G,i\rangle$ of the Dummy problem of WMGs and outputs an input instance of $\texttt{LOC-B-SHAP}(\texttt{RNN-ReLu})$ $\langle f_{G},i,x,x^{ref}\rangle$ such that:
$$\text{The player i is dummy} \iff \phi_{b}(f_{G},i,x,x^{ref}) > 0$$
where $f_{G}$ is a RNN-ReLu. 

Similar to the case of Sigmoidal neural networks, the main idea is to construct a RNN-ReLu that simulates a given WMG. However, contrary to  sigmoidal neural networks, the constructed RNN-ReLu perfectly simulates a WMG:

\begin{proposition} \label{app:prop:WMG2relu}
There exists a polynomial time algorithm that takes as input a WMG $G = \langle N, \{n_{j}\}_{j \in [N]}, q\rangle$ and outputs an RNN-ReLu such that: 
$$\forall x \in \{0,1\}^{N}: f_{G}(x) = v_{G}(S_{x})$$
where: $S_{x} \myeq \{j \in [N]: x_{j} = 1\}$ 
\end{proposition}

\begin{proof}
   Fix a WMG $G = <N,\{n_{j}\}_{j \in [N]}, q>$. The idea of constructing a RNN-ReLu that satisfies the property implicitly stated in Proposition \ref{app:prop:WMG2relu} consists at maintaining the sum of votes of players participating in the coalition in the hidden state vector of the RNN during the forward run of the RNN.
      
    The dynamics of a RNN-ReLu of size $N+2$ that performs this operation is given as follows:
    \begin{enumerate}
        \item \textbf{Initial state:} $h_{init} = \begin{pmatrix} 0 \\ \vdots \\ 0 \\ 1 \end{pmatrix}$
        \item \textbf{Transition function:} for $j \in \{1, \ldots, N-1\}$
        $$h[j+1] = \begin{cases}
        h[j] + n_{j} & \text{if}~~ x_{j} = 1 ~~ \text{(Add the vote of the player $j$ and store it in neuron j+1 (i.e. $h[j+1]$))}\\
        h[j] & \text{if}~~x_{j} = 0 ~~\text{(Ignore the vote of player j as he/she is not part of the coalition)}
        \end{cases}
        $$ 
        $$h[N+2] = h[N+2]$$
                \item \textbf{The output layer:} For a given hidden state vector $\mathbb{R}^{N+1}$ 
        $$y = I(h[N+1] - q \cdot h[N+2] \geq 0)$$
        In other words, the output vector $O = \begin{pmatrix}
            0 \\ \vdots \\ 0 \\ 1 \\ -q
        \end{pmatrix}$
    \end{enumerate}
    The dynamics of this RNN-ReLu is designed in such a way that for any $j \in [N]$ the value of the element $h[j+1]$ of the hidden state vector stores the cumulative votes of participants represented by the input sequence $x_{1:j}$ (The first $j$ symbols of the input sequence corresponding tp the players $\{1, \ldots, j\}$ in the game). Consequently, when $j=N+1$, the voting power of all players in the coalition is stored in $h[N+1]$. The output layer then outputs $1$ if the $h[N+1] \geq q$, otherwise it's equal to $0$.
    \end{proof}

\subsection{The problem \texttt{LOC-B-SHAP}($\texttt{ENS-DT}_{\texttt{C}}$) is NP-Hard} \label{app:subsec:bshaprf}
This segment is dedicated to proving the remaining point of Theorem \ref{thm:intractable}, which states the NP-Hardness of \texttt{LOC-B-SHAP}($\texttt{ENS-DT}_{\texttt{C}}$). As noted in the definition of $\texttt{ENS-DT}{\texttt{C}}$ in Appendix~\ref{app:sec:terminology}, since we are focusing on a \emph{hardness} proof, we can, for simplicity, assume that the weights associated with each tree are equal, giving us a classic majority voting setting. Additionally, we will assume that the number of classes $c:=2$, meaning we have a binary classifier. Clearly, proving hardness for this setting will establish hardness for the more general setting as well. Essentially, these assumptions provide us with a simple random forest classifier for boolean classification. For simplicity, we will henceforth denote this problem as \texttt{LOC-B-SHAP}($\texttt{ENS-DT}_{c}$) rather than the more general \texttt{LOC-B-SHAP}($\texttt{ENS-DT}_{\texttt{C}}$). Proving NP-Hardness for the former will also establish it for the latter. As mentioned earlier, this problem is reduced from the classical 3-SAT problem, a widely known NP-Hard problem.


\paragraph{Reduction strategy.} The reduction strategy is illustrated in Algorithm \ref{alg:SAT2b-SHAP}. For a given input CNF formula $\Psi$ over $n$ boolean variables $X = \{X_{1}, \ldots, X_{n}\}$ and $m$ clauses, the constructed random forest is a model whose set of input features contains the set $X$, with an additional feature denoted $X_{n+1}$ added for the sake of the reduction. The resulting random forest comprises a collection of $2m - 1$ decision trees, which can be categorized into two distinct groups:
\begin{itemize}
    \item $\mathcal{T}_{\Psi}$: A set of $m$ decision trees, each corresponding to a distinct clause in the input CNF formula. For a given clause $C$ in $\Psi$, the associated decision tree is constructed to assign a label 1 to all variable assignments that satisfy the clause $C$  while also ensuring that $x_{n+1} = 1$. It is simple to verify that such a decision tree can be constructed in polynomial time relative to the size of the input CNF formula
    \item \( \mathcal{T}_{null} \): This set consists of \( m - 1 \) copies of a trivial null decision tree. A null decision tree assigns a label of 0 to all input instances. 
\end{itemize}
\begin{algorithm}
\caption{Reduction of the \texttt{SAT} problem to \texttt{LOC-B-SHAP}($\texttt{ENS-DT}_{c})$}
\label{alg:SAT2b-SHAP}
\begin{algorithmic}[1]
\REQUIRE A CNF Formula $\Phi$ of $m$ clauses over $X = \{X_{1}, \ldots, X_{n} \}$
\ENSURE An input instance of \texttt{LOC-B-SHAP}($\texttt{ENS-DT}_{c}$): $\langle\mathcal{T}$, $i$, $x$, $x^{ref}\rangle$
\STATE $x \leftarrow [1 , \ldots , 1]$
\STATE $x^{ref} \leftarrow [0, \ldots , 0]$
\STATE $i \leftarrow n+1$
\STATE $\mathcal{T} \leftarrow \emptyset$
\FOR{$j \in [1,m]$}
 \STATE Construct a Decision Tree $T_{j}$ that assigns a label $1$ to variable assignments satisfying the formula: $C_{j} \land x_{n+1}$
 \STATE $\mathcal{T} \leftarrow \mathcal{T} \cup \{T\}_{j}$
\ENDFOR
 \STATE Construct a null decision tree $T_{null}$ that assigns a label $0$ to all variable assignments
 \STATE Add $m-1$ copies of $T_{null}$ to $\mathcal{T}$
\RETURN $\langle\mathcal{T},~i,~,x,~x^{ref}\rangle$
\end{algorithmic}
\end{algorithm}

The next proposition provides a property of the Random Forest classifier resulting from the construction. This property shall be leveraged in Lemma \ref{lemma:sat2bshap} to yield the main result of this section:

\begin{proposition} \label{app:prop:dnf2bshaprf}
    Let $\Psi$ be an arbitrary \emph{CNF} formula over $n$ boolean variables, and let $\mathcal{T}$ be the ensemble of decision trees outputted by Algorithm \ref{alg:SAT2b-SHAP} for the input $\Psi$. 
    We have: 
    $$f_{\mathcal{T}}(x_{1}, \ldots, x_{n}, x_{n+1}) = \begin{cases}
          1 & \text{if} ~~ x_{n+1} = 1 \land (x_1 , \ldots, x_n) \models \Psi \\
          0 & \text{otherwise}
    \end{cases}$$   
\end{proposition}
\begin{proof}
  Fix an arbitrary CNF formula over $n$ boolean variables, and let $(x_{1}, \ldots, x_{n})$ be an arbitrary variable assignment. If $x_{n+1} = 0$, the decision trees in the set $\mathcal{T}$ assign the label $0$, by construction. Consequently, $f_{\mathcal{T}}(x) = 0$.

    For now, we assume that $x_{n+1} = 1$, if $x \models \Psi$, then $x$ is assigned the label 1 for all decision trees in $\mathcal{T}$. Consequently, $f_{\mathcal{T}}(x) = 1$. On the other hand, if $x$ is not satisfied by $\Psi$, then there exists at least one decision tree in $\mathcal{T}$, say $T_{j}$, that assigns a label 0 to $x$. Consequently, $x$ is assigned a label $0$ by at least $m$ decision trees, i.e., for all decision trees in $\mathcal{T}_{null}\cup\{T_{j}\}$.
\end{proof}

Leveraging the result of Proposition \ref{app:prop:dnf2bshaprf}, the following lemma yields immediately the result of NP-Hardness of the decision problem associated to \texttt{LOC-B-SHAP}($\texttt{ENS-DT}_{c}$):
\begin{lemma} \label{app:lemma:sat2bshap}
    Let $\Psi$ be an arbitrary \emph{CNF} formula of $n$ variables, and  $\langle\mathcal{T},~n+1,~x,x^{ref}\rangle$ be the output of Algorithm \ref{alg:SAT2b-SHAP} for the input $\Psi$. We have:
    $$\phi_{b}(f_{\mathcal{T}}, n+1, x, x^{ref}) > 0 \iff \exists x \in \{0,1\}^{n}: ~ x \models \Psi$$  
\end{lemma}

\begin{proof}
            Let $\Phi$ be an arbitrary CNF formula of $n$ variables, and $\langle \mathcal{T}, n+1, x, x^{ref} \rangle$ be the output of Algorithm \ref{alg:SAT2b-SHAP}. 

    By proposition \ref{app:prop:dnf2bshaprf}, we note that: 
    $$\forall  \mathbf{x}_{n} = (x_1, \ldots, x_n) \in \{0,1\}^{n}: ~ \left[ f_{\mathcal{T}}(\mathbf{x}_{n}, 1) - f_{\mathcal{T}}(\mathbf{x}_{n}, 0) = 0 \iff \mathbf{x}_{n} \text{ doesn't satisfy } \Phi \right] $$
    Combining this fact with the following two facts yiels the result of the lemma:
    \begin{enumerate}
     \item The Baseline SHAP score $\phi_{b}(f_{\mathcal{T}}, n+1, x, x^{ref})$ is expressed as a linear combination of positive weights of the terms $\{f_{\mathcal{T}}(\mathbf{x}_{n},1) - f_{\mathcal{T}}(\mathbf{x}_{n},1) \}_{\mathbf{x}_{n} \in \{0,1\}^{n} }$
     \item According to Proposition\ref{app:prop:dnf2bshaprf}, we obtain:
     $$\forall \mathbf{x}_{n} \in \{0,1 \}^{n} :~f_{\mathcal{T}}(\mathbf{x}_{n}, 1) - f_{\mathcal{T}}(\mathbf{x}_{n}, 0) = f_{\mathcal{T}}(\mathbf{x}_{n}, 1) \geq 0$$
     \end{enumerate}
     
\end{proof}

\section{Generalized Complexity Relations of SHAP Variants (Proof of Proposition~\ref{prop:hardnessrelation})} \label{app:sec:generalized}

In this section, we present the proof for Proposition~\ref{prop:hardnessrelation} from the main paper. Let us first restate the proposition:


\begin{unumberedproposition}
    Let $\mathcal{M}$ be a class of models and $\mathcal{P}$ a class of probability distributions such that $\texttt{\emph{EMP}} \preceq_{P}  \mathcal{P}$. 
    Then, \texttt{\emph{LOC-B-SHAP}}($\mathcal{M}$)  $\preceq_{P}$ \texttt{\emph{GLO-B-SHAP}}($\mathcal{M}$, $\mathcal{P}$) and \texttt{\emph{LOC-B-SHAP}}($\mathcal{M}$)  $\preceq_{P}$   \texttt{\emph{LOC-I-SHAP}}( $\mathcal{M}, \mathcal{P}$) .
\end{unumberedproposition}
\begin{proof} 
    (1) \texttt{LOC-B-SHAP}($\mathcal{M}$) $\preceq_{P}$ \texttt{GLO-B-SHAP}($\mathcal{M}$, $\mathcal{P}$). 
    
This result is derived directly by observing that:
      $$\mathbf{\Phi}_{b}(f,i,x^{ref},P_{x}) = \Phi_{b}(f,i,x,x^{ref})$$
      where $P_{x}$ is the empirical distribution induced by the input instance $x$.

      (2) \texttt{LOC-B-SHAP}($\mathcal{M}$) $\preceq_{P}$ \texttt{LOC-I-SHAP}($\mathcal{M}$, $\mathcal{P}$):
         The result is obtained straightforwardly by noting that:
         $$\phi_{i}(f,i,x,P_{x^{ref}}) = \phi_{b}(f,i, x, x^{ref})$$
         where $P_{x^{ref}}$ is the empirical distribution induced by the reference instance $x^{ref}$.
\end{proof}
\vfill

\end{document}